\documentclass[twoside]{article}
\usepackage{amsmath,amsthm,amssymb,calc,xcolor}
\usepackage{amsfonts}
\usepackage{latexsym}
\usepackage{mathtools}
\usepackage{stmaryrd}
\usepackage{multicol}
\usepackage{pdfsync}
\usepackage{dsfont}
\usepackage{comment}
\usepackage{algorithmic}
\usepackage{algorithm}
\usepackage[margin=1in,footskip=0.25in]{geometry}
\usepackage{graphicx}
\usepackage{subcaption}
\usepackage{hyperref}

\usepackage{fancyhdr}
\renewcommand{\P}{\mathbb{P}}
\newcommand{\R}{\mathbb{R}}

\newcommand{\CX}{\mathcal{X}}
\newcommand{\CY}{\mathcal{Y}}
\newcommand{\CM}{\mathcal{M}}

\newcommand{\CF}{\mathcal{F}}
\newcommand{\abs}[1]{\left\lvert#1\right\rvert}
\newcommand{\brac}[1]{\left\{#1\right\}}
\newcommand{\ip}[1]{\left\langle#1\right\rangle}
\newcommand{\norm}[1]{\left\lVert#1\right\rVert}
\newcommand{\paren}[1]{\left(#1\right)}
\newcommand{\floor}[1]{\left\lfloor#1\right\rfloor}

\DeclareMathOperator{\E}{\mathbb{E}}
\DeclarePairedDelimiter\ceil{\lceil}{\rceil}

\newtheorem{theorem}{Theorem}[section]
\newtheorem{lemma}[theorem]{Lemma}
\newtheorem{proposition}[theorem]{Proposition}

\newtheorem{remark}{Remark}[section]
\newtheorem{assumption}{Assumption}[section]

\newcommand{\la}{\left \langle}
\newcommand{\ra}{\right\rangle}

\newcounter{listctr}
\newenvironment{custlist}[1][Case]%
	{\vspace{-5pt}\begin{list}{}{\usecounter{listctr}%
		\renewcommand\makelabel[1]{\hfil\itshape #1\ \arabic{listctr}.}%
		\settowidth\labelwidth{\makelabel{\quad}}%
		\setlength\leftmargin{0pt}}%
		\setlength\itemindent{\labelwidth+\labelsep}}
	{\end{list}}

\title{Normalization effects on shallow neural networks and related asymptotic expansions}

\author{Jiahui Yu\footnote{Department of Mathematics and Statistics, Boston University, Boston, E-mail: jyu32@bu.edu} \phantom{.}  and Konstantinos Spiliopoulos\footnote{Department of Mathematics and Statistics, Boston University, Boston, E-mail: kspiliop@math.bu.edu}
\thanks{K.S. was partially supported by the National Science Foundation (DMS 1550918) and Simons Foundation Award  672441. Code used for the numerical section of this paper is available at \url{https://github.com/kspiliopoulos/NormalizationEffectsNeuralNetworks} }\\
}

\begin{document}


\date{\today}
\maketitle

\begin{abstract}
We consider shallow (single hidden layer) neural networks and characterize their performance when trained with stochastic gradient descent as the number of hidden units $N$ and gradient descent steps grow to infinity. In particular, we investigate the effect of different scaling schemes, which lead to different normalizations of the neural network, on the network's statistical output, closing the gap between the $1/\sqrt{N}$ and the mean-field $1/N$ normalization. We develop an asymptotic expansion for the neural network's statistical output pointwise with respect to the scaling parameter as the number of hidden units grows to infinity. Based on this expansion, we demonstrate mathematically that to leading order in $N$, there is no bias-variance trade off, in that both bias and variance (both explicitly characterized) decrease as the number of hidden units increases and time grows. In addition, we show that to leading order in $N$, the variance of the neural network's statistical output decays as the implied normalization by the scaling parameter approaches the mean field normalization.
Numerical studies on the MNIST and CIFAR10 datasets show that test and train accuracy monotonically improve as the neural network's normalization gets closer to the mean field normalization.
\end{abstract}


\section{Introduction}\label{S:Introduction}

Neural networks are nonlinear parametric models whose coefficients can be estimated from data using stochastic gradient descent (SGD) methods. Neural networks have received a lot of attention in recent years due to their success in many applications. Neural networks have revolutionized fields such as speech, text and image recognition, see for example  \cite{LeCun,Goodfellow,DriverlessCar,FacialRecognition,DeepVoice,GoogleDuplex,SpeechRecognition3} to name a few. In addition, there is a growing interest in using neural network models in scientific fields such as medicine, robotics, finance and engineering, see for example \cite{Ling1,Ling2,Robotics2,Robotics3,NatureMedicine1,NatureMedicine2,Finance1,Finance2,Finance3}. Due to their success in applications, a deep understanding of their mathematical properties becomes more and more important.

The noticeable impact on applications has led to an increased interest in studying their mathematical properties. Despite the progress in recent years \cite{Telgarsky1,Mallat,Telgarsky2} and more recently
\cite{Chizat2018,NTK,JasonLee,Du1, Montanari,RVE,SirignanoSpiliopoulosNN1,SirignanoSpiliopoulosNN2,SirignanoSpiliopoulosNN3,SirignanoSpiliopoulosNN4} in regards to optimization properties and limiting theorems as well as earlier classical results regarding the approximation power of neural networks \cite{Barron,Hornik1,Hornik2}, many questions are still open.

In this paper we characterize the performance of neural networks trained with SGD as number of hidden units and gradient descent steps grow to infinity. Our contribution is four-fold. First, we provide a stochastic Taylor kind of approximation of the neural network's statistical output, see (\ref{Eq:FormalExpansion}) below,  as the number of hidden units $N$ grow to infinity. Second, using the stochastic Taylor expansion, we identify the main contributions to the neural network's bias and variance and we find that to leading order in $N$, there is no bias-variance trade-off in the sense that both decay as the number of hidden units increases and time grows. Third, we identify how the learning rate should be scaled with $N$, in order to have meaningful statistical behavior of the neural network's output. Fourth, we explore different normalization (equivalently initialization) schemes and demonstrate numerically that train and test accuracy monotonically improve as the neural network's normalization approaches the mean field normalization.

Our starting point is the single-layer neural network with varying scale initialization factor $\gamma \in \left[\frac{1}{2},1\right]$:
\begin{equation}
g^N(x;\theta) = \frac{1}{N^{\gamma}} \sum_{i=1}^{N} C^i\sigma(W^i x),\label{Eq:NN1}
\end{equation}
where $C^i \in \R$, $W^i \in \R^d$, $x \in \R^d$, and $\sigma(\cdot):\R \to \R$. For convenience, we write $W^ix = \ip{W^i,x}_{l^2}$ as the standard $l^2$ inner product for the vectors. The neural network model has parameters $\theta = (C^1,\ldots, C^N, W^1, \ldots, W^N) \in \R^{(1+d)N}$, which are to be estimated from data $(X,Y) \sim \pi(dx,dy)$. The objective function is
\begin{equation}
\mathcal{L}(\theta) = \frac{1}{2} \E_{X,Y} \left[(Y-g^N(X;\theta))^2\right],\label{Eq:ObjFunction}
\end{equation}
and the model parameters $\theta$ are trained by the stochastic gradient descent:
\begin{equation}\label{SGD_iteration}
\begin{aligned}
C^i_{k+1} &= C^i_{k} + \frac{\alpha^N_k}{N^{\gamma}}\paren{y_k-g^N_k(x_k)}\sigma(W^i_k  x_k),\\
W^i_{k+1} &= W^i_{k} + \frac{\alpha^N_k}{N^{\gamma}}\paren{y_k-g^N_k(x_k)}C^i_k \sigma'(W^i_k  x_k)x_k,
\end{aligned}
\end{equation}
for $k \in \{ 0, 1, 2, \cdots\}$, $\alpha^N_k$ is the learning rate that in principle may depend on both the number of hidden units  $N$ and on the iteration $k$ and $(x_k,y_k)$ are i.i.d data samples from the distribution $\pi(dx,dy)$. In this paper we take $\alpha^N_k=\alpha^N$ to only depend on the number of hidden units $N$. Let us also define
\begin{align*}
g^N_k(x) &= g^N(x;\theta_k) = \frac{1}{N^{\gamma}} \sum_{i=1}^{N} C^i_k\sigma(W^i_k x),
\end{align*}
and the empirical measure
\begin{equation*}
\nu_k^N = \frac{1}{N}\sum_{i=1}^N \delta_{C^i_k,W^i_k}.
\end{equation*}

The neural network output can be re-written as
$g^N_k(x) = \ip{c\sigma(w x), N^{1-\gamma} \nu_k^N},$
where $\ip{f,h}$ denotes the inner product of $f$ and $h$, i.e. $\ip{f, \nu_k^N} = \int f \nu_k^N(dc,dw)$. It turns out that in order to correctly pass from discrete time to continuous time we need to re-scale the training time $k$ in the discrete stochastic gradient descent algorithm as $k=\lfloor Nt \rfloor$. Here, $t$ is interpreted as training time in the related continuous time limiting algorithm as $N\rightarrow\infty$. Then, we define 
\begin{align*}
\mu_t^N &= \nu^N_{\lfloor Nt \rfloor}, \quad h_t^N= g^N_{\lfloor Nt \rfloor},
\end{align*}
where $g_k^N$ is an M-dimensional vector whose $i$-th entry is $g_k^N(x^{(i)})$ for a fixed data set $(x^{(i)},y^{(i)})_{i=1}^{M}$. Hence, $h_t^N$ is the M-dimensional vector whose $i$-th entry is $h_t^N(x^{(i)}) = g_{\floor{Nt}}^N(x^{(i)})$.  If $t\leq T$ then one can interpret $h_t^N$ as the output of the network after $({t}/{T})100\%$ of the training has been completed. Notice that the neural network's output $h_t^N$ depends on $\gamma\in[1/2,1]$ and we will write $h_t^{N,\gamma}$ when we want to emphasize this.

The behavior of the neural network has been studied in recent years in the cases $\gamma=1/2$ and $\gamma=1$. When  $\gamma=1/2$, one has the  Xavier initialization which is widely used in applications, see \cite{Xavier}. The subsequent works \cite{NTK,HuangYau2020, JasonLee} analyze the evolution of the network by studying the behavior of the so called neural tangent kernel, and \cite{Du1,SirignanoSpiliopoulosNN4} analyze the limiting behavior of $h_t^{N,1/2}$ and of related reinforcement learning algorithms as $N\rightarrow\infty$. On the other side of the aisle, when $\gamma=1$, one has the so called mean field normalization, which has also received considerable attention in recent years, see for example
\cite{Chizat2018,Montanari,RVE,SirignanoSpiliopoulosNN1,SirignanoSpiliopoulosNN2,SirignanoSpiliopoulosNN3} for a representative but not exhaustive list. In both cases, i.e., when $\gamma=1/2$ and when $\gamma=1$, the aforementioned papers study the behavior of the neural network $h_t^{N,\gamma}$ and of the related empirical measure of trained parameters $\mu_t^N$ and under various sets of assumptions prove convergence to the global minimum.

In this paper, we pose three questions:
\begin{enumerate}
\item{What is the behavior of the neural network's output $h_t^{N,\gamma}$ for $\gamma\in(1/2,1)$ and for fixed but large number of  hidden units $N$ and training time $t$?}
\item{What about variance-bias trade off in terms of the number of hidden units $N$ and training time $t$?}
\item{How can one compare the different scalings $\gamma\in[1/2,1]$ in terms of performance and test accuracy?}
\end{enumerate}

To answer these questions, we look at $h_t^{N,\gamma}$ and for given $\gamma$ we investigate a stochastic Taylor kind of expansion of $h_t^{N,\gamma}$ around its limiting behavior as $N\rightarrow\infty$ and its interaction with time $t$ as $t\rightarrow\infty$. For each fixed $\gamma\in(1/2,1)$ we demonstrate heuristically that as $N\rightarrow\infty$, and  when $\gamma \in \left(\frac{2\nu-1}{2\nu}, \frac{2\nu+1}{2\nu+2}\right)$ for fixed $\nu\in\{1,2,3,\cdots\}$:
\begin{align}
h_t^{N,\gamma}&\approx h_{t}+\sum_{j=1}^{\nu-1} N^{-j(1-\gamma)}Q^j_t + N^{-(\gamma-1/2)}e^{-At} \mathcal{G}+ \textrm{ lower order terms in }N,\label{Eq:FormalExpansion}
\end{align}
where $h_{t}$ is the limit of  $h_t^{N,\gamma}$ as $N\rightarrow\infty$, $Q^j_t$ are deterministic quantities, $A$ is a positive definite matrix and $\mathcal{G}$ is a Gaussian vector of mean zero and known variance-covariance structure. All of $h_{t}$, $Q^j_t$, $A$ and $\mathcal{G}$ are independent of $N<\infty$ and $\gamma>0$.  For all $\gamma \in (1/2,1)$, the limit of the network output recovers the global minimum as $t \to \infty$, i.e.  $h_t \to \hat{Y}$, where $\hat{Y} = \paren{y^{(1)}, \ldots, y^{(M)}}$ (note that this is also true for $\gamma=1/2$, see \cite{SirignanoSpiliopoulosNN4}).
For fixed $j\in\mathbb{N}$, one can also show that $Q^{j}_{t}\rightarrow 0$ exponentially fast as $t\rightarrow\infty$. The Gaussian vector $\mathcal{G}$ is related to the variance of the network at initialization which then propagates forward, see (\ref{limit_gaussian}). The behavior at the boundary points $\gamma \in \left\{\frac{2\nu-1}{2\nu}, \frac{2\nu+1}{2\nu+2}\right\}$ is more subtle and is presented in detail in Section \ref{sec::main_results}.

An expansion such as (\ref{Eq:FormalExpansion}) suggests that for fixed $N<\infty$ and $t<\infty$, to leading order in $N$, the variance of the neural network's statistical output is smaller when $\gamma\rightarrow 1$. Indeed, the variance of the leading random correction is of the order $N^{-2(\gamma-1/2)}$ which goes to zero faster when $\gamma\rightarrow 1$. On the other hand, when $\nu\geq 2$, i.e., when $\gamma>3/4$, there is, to leading order in $N$, a bias effect captured by the term $\sum_{j=1}^{\nu-1} N^{-j(1-\gamma)}Q^j_t$, and  it asymptotically vanishes as $N,t\rightarrow\infty$. The formal expansion equation (\ref{Eq:FormalExpansion}) shows that, to leading order in $N$, there is no bias-variance trade-off, in the sense that as the number of hidden units $N$ and the time $t$ increase, both the variance and the bias decrease for all $\gamma\in[1/2,1)$.

To further investigate the performance and generalization properties of the different scaled networks, various numerical studies were conducted and the results are shown in Section \ref{S:Numerics}. We see that  test and train accuracy of the trained neural networks behave better as $\gamma\rightarrow 1$. In particular, as $\gamma$ increases from $1/2$ to $1$,  test and train accuracy increase monotonically for both the MNIST \cite{MNIST} and the CIFAR10  \cite{Krizhevsky} data sets. 

Lastly, we mention that our conclusions are also related to recent (mostly empirical) observations in the literature which supports the idea that generalization error can decrease with over parametrization. In \cite{Geiger,Neal}, ideas based on bias-variance characterization of the neural network output for the $\gamma=1/2$ case are being utilized to make the case that both bias and variance decay as training moves forward, which indeed is also seen to be the case from (\ref{Eq:FormalExpansion}). These papers use scaling arguments to study out-of-sample performance of the neural network's output based on its leading order behavior as the number of hidden layers increases. We also point out that the results presented in our paper are also related to function approximation. For exmaple, this can be seen by replacing the data $y$ in (\ref{Eq:ObjFunction}) and (\ref{SGD_iteration}) by an unknown target function, say  $\phi(x)$, that we seek to learn.

The rest of the paper is organized as follows. In Section \ref{sec::main_results}, we state our assumptions, notation and main results that lead to the expansion (\ref{Eq:FormalExpansion}). Section \ref{S:Review} reviews the relevant known results for the so called Xavier normalization ($\gamma=1/2$) and for the mean field normalization ($\gamma=1$) and connects them to the results of this paper. In Section \ref{S:Numerics}, we discuss our theoretical results and present our numerical studies comparing the different scaling schemes with respect to $\gamma\in[1/2,1]$. In Section \ref{S:Conclusions}, we present our conclusions. The proofs of this paper are in the appendix, which is presented as follows. Section \ref{sec::LLN} discusses the law of large numbers, i.e the convergence of $h^{N,\gamma}_{t}\rightarrow h_{t}$ as $N\rightarrow\infty$ for $\gamma\in(1/2,1)$. Sections \ref{sec::CLT} and \ref{sec::Psi} discuss the fluctuation behaviors of $h^{N,\gamma}_{t}$ around $h_{t}$. Section \ref{sec::higher order} presents the derivation of the asymptotic expansion of $h^{N,\gamma}_t$ for general $\gamma\in(1/2,1)$, which is based on an induction argument. 

\section{Assumptions, notation and main results}\label{sec::main_results}
The purpose of this section is to present our main assumption for the optimization problem (\ref{Eq:NN1})-(\ref{SGD_iteration}), establish notation and present the main theoretical results that make the expansion (\ref{Eq:FormalExpansion}) rigorous.
\begin{assumption}
\begin{enumerate}
\item The activation function $\sigma \in C^{\infty}_b(\R)$, i.e. $\sigma$ is infinitely differentiable and bounded.
\item There is a fixed dataset $\CX \times \CY = (x^{(i)}, y^{(i)})_{i=1}^M$, and set $\pi(dx,dy) = \frac{1}{M} \sum_{i=1}^M \delta_{(x^{(i)},y^{(i)})}(dx,dy)$.
\item The initialized parameters $(C_0^i, W_0^i)$ are i.i.d.,generated from a mean-zero random variable with a distribution $\mu_0(dc,dw)$. Furthermore, the random variables $C_0$ and $W_0$ are independent.
\item The random variables $C^i_0$ have compact support, $\E\paren{C^i_0} = 0$, and $\E \paren{\norm{W^i_0}} < \infty$.
\end{enumerate}
\label{assumption}
\end{assumption}

For some of our results we would need to further assume the following.
\begin{assumption}\label{assumption1}
\begin{enumerate}
\item The activation function $\sigma$ is smooth, non-polynomial and slowly increasing\footnote{A function $\sigma(x)$ is called slowly increasing if $\lim_{x\rightarrow\infty}\frac{\sigma(x)}{x^{a}}=0$ for every $a>0$.}.
\item The fixed dataset $(x^{(i)}, y^{(i)})_{i=1}^M$ from part (ii) of Assumption \ref{assumption} has  data points that are in distinct directions (per definition on page $192$ of \cite{yIto}).
\end{enumerate}
\end{assumption}
Examples of activation functions that are non-polynomials and slowly increasing are tanh and sigmoid activation function.
Note that by Assumption \ref{assumption}, as $N \to \infty$, we have $\nu^N_0 \xrightarrow{p} \mu_0$, and for $x \in \mathcal{X}$,
by the Central Limit Theorem (CLT),
\begin{equation}\label{limit_gaussian}
\begin{aligned}
&N^{\gamma-\frac{1}{2}}h^N_0(x) = \ip{c\sigma(w x), \sqrt{N} \nu_0^N} \xrightarrow{d} \mathcal{G}(x),
\end{aligned}
\end{equation}
where $\mathcal{G}(x)$ is a Gaussian random variable with mean zero and variance $\tau^{2}(x)=\ip{|c\sigma(wx)|^{2},\mu_{0}}$. We will always denote $\mathcal{G}(x)\sim N(0,\tau^{2}(x))$. For notational convenience, we define the following three quantities
\begin{align}
A_{x,x'}&=\left<B_{x,x'}(c,w),\mu_{0}\right>\nonumber,\\
B_{x,x'}(c,w)&=\sigma(wx')\sigma(wx)+c^2\sigma'(wx')\sigma'(wx)xx'\nonumber,\\
C_{x'}^{f}(c,w) &= \nabla f(c,w)\cdot\nabla (c\sigma(wx'))= \partial_cf(c,w)\sigma(w x') + c\sigma'(wx')\nabla_w f(c,w)x', \text{ for }f \in C^2_b(\R^{1+d}). \label{Eq:KernelDef}
\end{align}

We will study the limiting behavior of the network output as well as its first and second order fluctuation processes as the width $N$ and the stochastic gradient descent steps $\floor{TN}$ approaches infinity simultaneously. To be precise, we study the behavior of $\{(\mu_t^N,h_t^N), t\in[0,T]\}_{N\in\mathbb{N}}$ as $N \to \infty$ in the Skorokhod space $D_E([0,T])$\footnote{$D_E([0,T])$ is the set of maps from $[0,T]$ into $E$ which are right-continuous and which have left-hand limits.}, where $E = \CM(\R^{1+d}) \times \R^M$ and $\CM(\R^{1+d})$ is the space of probability measures on $\R^{1+d}$. The first order convergence, i.e. the typical behavior as $N\rightarrow\infty$, is given in Theorem \ref{LLN:theorem} below.
\begin{theorem}\label{LLN:theorem}
Under Assumption \ref{assumption}, for fixed $\gamma \in (1/2,1)$ and learning rate $\alpha^{N} = \alpha/N^{2(1-\gamma)}$ with constant $0<\alpha<\infty$, as $N \to \infty$, the process $(\mu_t^N,h_t^N)$ converges in probability in the space $D_E([0,T])$ to $(\mu_t,h_t)$, which satisfies the evolution equation
\begin{equation}\label{LLN:limit_evolution}
\begin{aligned}
h_t(x) &= h_0(x) + \alpha \int^t_0 \int_{\CX \times \CY} \paren{y-h_s(x')} \ip{B_{x,x'}(c,w),\mu_0} \pi(dx',dy) ds,\\
\end{aligned}
\end{equation}
where $h_0(x) = 0$. Furthermore, for any $f \in C_b^2(\R^{1+d})$, $\ip{f,\mu_t} = \ip{f,\mu_0}$.
\end{theorem}

The proof of Theorem \ref{LLN:theorem} is in Section \ref{sec::LLN} of the appendix. We now discuss some of the implications of this result. Notice that (\ref{LLN:limit_evolution}) can simply be written as
\begin{align}\label{LLN:limit_evolution2}
h_t(x) &=  \alpha \int^t_0 \int_{\CX \times \CY} \paren{y-h_s(x')} A_{x,x'} \pi(dx',dy) ds.
\end{align}
In \cite{SirignanoSpiliopoulosNN4}, the case $\gamma=1/2$ was studied and it was proven there that $h_{t}^{N,1/2}\rightarrow h_{t}(x)$ with the only difference from (\ref{LLN:limit_evolution}) being that  if $\gamma=1/2$ then $h_0(x)=\mathcal{G}(x)$ whereas $h_0(x)=0$ when $\gamma\in(1/2,1)$. This means that even though in the $\gamma=1/2$ case the limit is random (due to the randomness at initialization), when $\gamma>1/2$ the first order limit is deterministic.

In analogy with \cite{SirignanoSpiliopoulosNN4} where the $\gamma=1/2$ case was studied, if we further assume Assumption \ref{assumption1}, then the matrix $A \in \mathbb{R}^{M \times M}$, whose elements are $\alpha A_{x,x'}$ with $x, x' \in \mathcal{X}$, is positive definite (see Lemma 3.3 of \cite{SirignanoSpiliopoulosNN4} and Proposition 2 in \cite{NTK}). The latter immediately implies that we have convergence to the global minimum, see Theorem 4.3 in \cite{SirignanoSpiliopoulosNN4}: 
 \begin{align}
h_t \rightarrow \hat Y \phantom{....} \textrm{as} \phantom{....} t \rightarrow\infty.\label{Eq:ConverenceGM}
\end{align}
 where $h_t = (h_t(x^{(1)}), \ldots, h_t(x^{(M)}))$ and $\hat{Y} = ( y^{(1)}, \ldots,  y^{(M)} )$.

The fact that for $\gamma\in(1/2,1)$ the first order convergence, as given by Theorem \ref{LLN:theorem}, is deterministic, motivates us to look at the second order convergence, i.e., at the fluctuations around the typical behavior: 
\begin{align}
K^N_t &= N^{\varphi} \left(h^N_t - h_t\right)\nonumber
\end{align}
and study its behavior as $N\rightarrow\infty$ for the appropriate value of $\varphi$. In Theorem \ref{CLT:theorem}, we prove a central limit theorem for the network, which shows how the finite neural network fluctuates around its limit for large $N$. It also quantifies the speed of convergence of the finite neural network to its limit. To study the limiting behavior of $K^N_t$, we need to study the convergence of $l_t^N(f) = \ip{f,\eta^N_t}$ for a fixed $f\in C^2_b(\R^{1+d})$, where $\eta^N_t = N^{\varphi}\paren{\mu^N_t -\mu_0}$. One can also show that the limit of $K^N_t$ goes to 0 exponentially fast as $t\to \infty$. The relevant results are given in Proposition \ref{prop::l_t}, Theorem \ref{CLT:theorem} and Theorem \ref{R:BiasToZero1} below and proven in Section \ref{sec::CLT} of the appendix.

\begin{proposition}\label{prop::l_t}
Under Assumption \ref{assumption}, for fixed $\gamma \in (1/2,1)$, learning rate $\alpha^{N} = \alpha/N^{2(1-\gamma)}$ with constant $0<\alpha<\infty$ and fixed $f \in C^2_b(\R^{1+d})$, if $\varphi \le 1-\gamma$, the processes $\{l_t^N(f) = \ip{f,\eta^N_t}, t\in[0,T]\}_{N\in\mathbb{N}}$ converges in probability in the space $D_{\R}([0,T])$ as $N \to \infty$, and
\begin{custlist}[Case]
\item If $\varphi < 1-\gamma$, $\ip{f,\eta^N_t} \rightarrow 0$.
\item If $\varphi = 1-\gamma$, $l^N_t(f) = \ip{f,\eta^N_t} \rightarrow l_t(f)$, where $l_t(f)$ is given by
\begin{equation}\label{l_t limit}
\begin{aligned}
l_{t}(f)
&= \int_0^t \int_{\CX \times \CY} \alpha \paren{y-h_s(x')}  \ip{  C^f_{x'}(c,w),\mu_0} \pi(dx',dy) ds.\\
\end{aligned}
\end{equation}
\end{custlist}
 \end{proposition}

\begin{theorem}\label{CLT:theorem}
Let $\mathcal{G}(x)$ be the Gaussian random variable defined in (\ref{limit_gaussian}). Under Assumption \ref{assumption}, for fixed $\gamma \in (1/2,1)$ and learning rate $\alpha^{N} = \alpha/N^{2(1-\gamma)}$ with constant $0<\alpha<\infty$,
as $N \to \infty$, the sequence of processes $\{K^N_t, t\in[0,T]\}_{N\in\mathbb{N}}$ converges in distribution in the space $D_{\R^M}([0,T])$ to $K_t$, which satisfies one of the following evolution equations, depending on the values of $\gamma$ and $\phi$:
\begin{custlist}[Case]
\item When $\gamma \in \paren{\frac{1}{2}, \frac{3}{4}}$ and $\varphi \le \gamma - \frac{1}{2}$, or when $\gamma \in \left[\frac{3}{4}, 1\right)$ and $\varphi < 1-\gamma \le \gamma - \frac{1}{2}$,

\begin{equation}\label{CLT:evolution}
\begin{aligned}
K_t(x) &= K_0(x)-\alpha \int^t_0 \int_{\CX \times \CY} K_s(x') A_{x,x'} \pi(dx',dy) ds
\end{aligned}
\end{equation}
where $K_0(x) = 0$ if $\varphi < \gamma - \frac{1}{2}$, and $K_0(x)=\mathcal{G}(x)$  if $\varphi = \gamma-\frac{1}{2}$.
\item When $\gamma \in \left[\frac{3}{4}, 1\right)$ and $\varphi = 1-\gamma$,
\begin{equation}\label{K_t for 1-gamma}
\begin{aligned}
K_t(x)
&= K_0(x) + \int^t_0 \int_{\CX \times \CY} \alpha\paren{y-h_s(x')} l_s(B_{x,x'}(c,w)) \pi(dx',dy) ds-  \alpha \int^t_0 \int_{\CX \times \CY}  K_s(x') A_{x,x'} \pi(dx',dy) ds,\\
\end{aligned}
\end{equation}
where $K_0(x) = 0$ if $\gamma \in \paren{\frac{3}{4},1}$,  $K_0(x) = \mathcal{G}(x)$ if $\gamma = \frac{3}{4}$, and $l_t(f)$ is given by equation \eqref{l_t limit} for any $f\in C^2_b(\R^{1+d})$.
\end{custlist}
\end{theorem}

\begin{theorem}\label{R:BiasToZero1}
Suppose parts (iii)-(iv) of Assumption \ref{assumption} and Assumption \ref{assumption1} hold. In both cases of Theorem \ref{CLT:theorem},  $ l_s(B_{x,x'}(c,w))$ is uniformly bounded, and  $K_t(x)\rightarrow 0$ exponentially fast for all $x\in\mathcal{X}$ as $t \rightarrow\infty$. In particular, there is a constant $C<\infty$, potentially depending on the normal random vector $\mathcal{G}$ and on the magnitude of the data $\hat{Y}$, such that $\|K_{t}\|\leq C(1+t)e^{-\lambda_{0}t}$, where $\lambda_{0}>0$ is the smallest eigenvalue of the positive definite matrix $A$.
\end{theorem}

Theorem \ref{CLT:theorem} shows that when $\gamma > {3}/{4}$, the limit of $K^N_t$ follows a deterministic evolution equation, which implies that the  convergence is in fact in probability and further it motivates us to investigate the convergence of the second order fluctuations $\Psi^N_t = N^{\zeta - \varphi} (K^N_t - K_t)$ for $\gamma \in \paren{{3}/{4},1}$. The results of this section is given by Proposition \ref{prop::L_t}, Theorem \ref{thm::Psi} and Theorem \ref{R:BiasToZero2} below and their proofs are in Appendix \ref{sec::Psi}.

\begin{proposition}\label{prop::L_t}
Under Assumption \ref{assumption}, for fixed $\gamma \in (3/4,1)$, $\varphi = 1-\gamma$, learning rate $\alpha^{N} = \alpha/N^{2(1-\gamma)}$ with constant $0<\alpha<\infty$ and fixed $f \in C^3_b(\R^{1+d})$,
if $\zeta \le 2\varphi$, the processes $\{L_t^N(f) = N^{\zeta-\varphi} [l^N_t(f) - l_t(f)], t\in[0,T]\}_{N\in\mathbb{N}}$ converges in probability in the space $D_{\R}([0,T])$ as $N \to \infty$, and
\begin{custlist}[Case]
\item If $\zeta < 2\varphi = 2-2\gamma$, $L^N_t(f) \rightarrow 0$.
\item If $\zeta = 2\varphi = 2-2\gamma$, $L^N_t(f)  \rightarrow L_t(f)$, where $L_t(f)$ is given by
\begin{equation}\label{limit_Lt}
\begin{aligned}
L_t(f) &=   \int_0^t \int_{\CX \times \CY}  \alpha\paren{y-h_s(x')} l_s( C_{x'}^{f}(c,w)) \pi(dx',dy) ds -  \int_0^t \int_{\CX \times \CY} \alpha K_s(x')\ip{C_{x'}^{f}(c,w),\mu_0} \pi(dx',dy) ds.
\end{aligned}
\end{equation}
\end{custlist}
\end{proposition}

\begin{theorem}\label{thm::Psi}
Let $\mathcal{G}(x)$ be the Gaussian random variable defined in (\ref{limit_gaussian}). Under Assumption \ref{assumption}, for fixed $\gamma \in (3/4,1)$, $\varphi = 1-\gamma$ and learning rate $\alpha^{N} = \alpha/N^{2(1-\gamma)}$ with constant $0<\alpha<\infty$ as $N \to \infty$, the sequence of processes $\{\Psi^N_t, t\in[0,T]\}_{N\in\mathbb{N}}$ converges in distribution in the space $D_{\R^M}([0,T])$ to $\Psi_t$, which satisfies  of the following evolution equations, depending on the values of $\gamma$ and $\zeta$:
\begin{custlist}[Case]
\item When $\gamma \in \paren{ \frac{3}{4}, \frac{5}{6}}$ and $\zeta \le \gamma - \frac{1}{2}$, or when $\gamma \in \left[\frac{5}{6}, 1\right)$ and $\zeta < 2-2\gamma \le \gamma - \frac{1}{2}$,
\begin{equation}\label{limit_Psi}
\begin{aligned}
\Psi_t(x) &= \Psi_0(x)-\alpha \int^t_0 \int_{\CX \times \CY} \Psi_s(x') A_{x,x'} \pi(dx',dy) ds,
\end{aligned}
\end{equation}
where $\Psi_0(x) = 0$ if $\zeta < \gamma - \frac{1}{2}$, and $\Psi_0(x)=\mathcal{G}(x)$ if $\zeta = \gamma-\frac{1}{2}$.
\item When $\gamma \in \left[\frac{5}{6}, 1\right)$ and $\zeta = 2-2\gamma$,
\begin{equation}\label{limit_Psi_2}
\begin{aligned}
\Psi_t(x) &= \Psi_0(x) + \int^t_0 \int_{\CX \times \CY} \alpha\paren{y-h_s(x')} L_s(B_{x,x'}(c,w)) \pi(dx',dy) ds\\
&\quad-\alpha\int_0^t \int_{\CX \times \CY} K_s(x')l_s(B_{x,x'}(c,w))\pi(dx',dy)ds
-\alpha \int^t_0 \int_{\CX \times \CY} \Psi_s(x') A_{x,x'} \pi(dx',dy) ds,
\end{aligned}
\end{equation}
where $\Psi_0(x) =0$  if $\gamma \in \paren{\frac{5}{6},1}$, $\Psi_0(x) = \mathcal{G}(x)$  if $\gamma = \frac{5}{6}$, $K_s$ satisfies equation \eqref{K_t for 1-gamma}, and $L_s$ satisfies \eqref{limit_Lt}.
\end{custlist}
\end{theorem}

\begin{theorem}\label{R:BiasToZero2}
Suppose parts (iii)-(iv) of Assumption \ref{assumption} and Assumption \ref{assumption1} hold. In both cases of Theorem \ref{thm::Psi},  $L_s(B_{x,x'}(c,w)) $ is uniformly bounded, and  $\Psi_t(x)\rightarrow 0$ exponentially fast for all $x\in\mathcal{X}$ as $t \rightarrow\infty$. In particular, there is a constant $C<\infty$, potentially depending on the normal random vector $\mathcal{G}$ and on the magnitude of the data $\hat{Y}$,
such that $\|\Psi_{t}\|\leq C(1+t+t^{2})e^{-\lambda_{0}t}$ where $\lambda_{0}>0$ is the smallest eigenvalue of the positive definite matrix $A$.
\end{theorem}

Theorems \ref{CLT:theorem} and \ref{thm::Psi} indicate that for large $N$ the neural network output behaves as
\begin{itemize}
\item $\gamma \in \left(\frac{1}{2}, \frac{3}{4} \right]$:
\begin{align}
h^N_t \approx h_t + \frac{1}{N^{\gamma-\frac{1}{2}}} K_t,\label{Eq:AsExp1}
\end{align}
where $K_t$ satisfies either of equations \eqref{CLT:evolution} or \eqref{K_t for 1-gamma} and has a Gaussian distribution.
\item $\gamma \in \left(\frac{3}{4}, \frac{5}{6} \right]$:
\[h^N_t \approx h_t + \frac{1}{N^{1-\gamma}} K_t + \frac{1}{N^{\gamma-\frac{1}{2}}} \Psi_t,\]
where $K_t$ satisfies equation \eqref{K_t for 1-gamma} with $K_0(x) =0$, $\Psi_t$ satisfies either equations \eqref{limit_Psi} or \eqref{limit_Psi_2} and has a Gaussian distribution.
\end{itemize}

These results indicate that there is an inductive process that gives the appropriate expansion of the neural network's output $h^{N,\gamma}_t$ as $N\rightarrow\infty$ in terms of $\gamma \in \left(\frac{2\nu-1}{2\nu}, \frac{2\nu+1}{2\nu+2}\right]$ pointwise with respect to  $\nu\in\mathbb{N}$. In Theorems \ref{CLT:theorem} and \ref{thm::Psi} we have provided the rigorous derivation for $\nu=1$ and $\nu=2$. Proceeding the same way, one can get the expansion for all $\gamma \in \left(\frac{2\nu-1}{2\nu}, \frac{2\nu+1}{2\nu+2}\right]$ with $\nu\in\mathbb{N}$.

In Appendix \ref{sec::higher order} we present an alternative, formal asymptotic expansion, that does exactly that, i.e., derives the asymptotic expansion of $h^{N,\gamma}_t$ as $N\rightarrow\infty$ for any $\gamma\in(1/2,1)$. Note that the methodology of Appendix \ref{sec::higher order} does recover the rigorous results of $h^{N,\gamma}_t$ as $N\rightarrow\infty$ for $\nu=1$ and $\nu=2$. 
The methodology of Appendix \ref{sec::higher order} demonstrates via an induction argument with respect to $\nu\in\{3,4,5,\cdots\}$ that indeed (\ref{Eq:FormalExpansion}) holds. Let us now briefly describe below the conclusion of the asymptotic expansion of Appendix \ref{sec::higher order}.

Let $\nu\in \mathbb{N}$ and let  $\mathcal{G}(x)$ be the Gaussian random variable defined in (\ref{limit_gaussian}). Then, when $\gamma \in \left[\frac{2\nu-1}{2\nu}, \frac{2\nu+1}{2\nu+2}\right)$,  we obtain that for any fixed $f\in C^{\infty}_b(\R^{1+d})$, as $N\rightarrow\infty$,
\begin{equation}\label{measure_expansion}
\ip{f,\mu^N_t} \approx \sum_{j=0}^{\nu-1} \frac{1}{N^{j(1-\gamma)}} l^j_t(f)+\text{ lower order terms in } N,
\end{equation}
where we have identified $l^0_t(f) = \ip{f,\mu_0}$, $l^1_t(f) = l_t(f)$, $l^2_t(f) = L_t(f)$. When $j \ge 3$,
\begin{equation}\label{Eq:l_equation}
\begin{aligned}
l^j_t(f) &= \alpha \int_0^t \int_{\CX\times \CY}  \paren{y-h_s(x')}l_s^{j-1}(C_{x'}^f(c,w)) \pi(dx',dy)ds\\
&\quad -\alpha \sum_{m=1}^{j-1}\int_0^t \int_{\CX\times \CY}  Q^{j-m}_s(x')l^{m-1}_s(C_{x'}^{f}(c,w)) \pi(dx',dy)ds.
\end{aligned}
\end{equation}
As $N\rightarrow\infty$ and when $\gamma \in \left(\frac{2\nu-1}{2\nu}, \frac{2\nu+1}{2\nu+2}\right]$, we have the asymptotic expansion
\begin{equation}\label{network_expansion}
h^N_t(x) \approx \sum_{j=0}^{\nu-1} \frac{1}{N^{j(1-\gamma)}}Q^j_t(x) + \frac{1}{N^{\gamma-\frac{1}{2}}}Q^{\nu}_t(x) +\text{ lower order terms in }N,
\end{equation}
where $Q^0_t = h_t$, $Q^1_t = K_t$, $Q^2_t = \Psi_t$, per Theorems \ref{LLN:theorem}, \ref{CLT:theorem} and \ref{thm::Psi} respectively. For $j = 0,\ldots, \nu-1$, $Q^j_t$ satisfy the deterministic evolution equations
\begin{equation}\label{Eq:Qj_formula1}
\begin{aligned}
Q^j_t(x) &= \alpha \int^t_0 \int_{\CX \times \CY} \paren{y-Q^0_s(x')} l^j_s(B_{x,x'}(c,w))\pi(dx',dy) ds\\
&\quad -\alpha\sum_{m=1}^j \int^t_0 \int_{\CX \times \CY} Q^m_s(x')l^{j-m}_s (B_{x,x'}(c,w)) \pi(dx',dy) ds.
\end{aligned}
\end{equation}
When $\gamma \in \paren{\frac{2\nu-1}{2\nu}, \frac{2\nu+1}{2\nu+2}}$,
\begin{equation}\label{Eq:Qk_formula1}
Q^{\nu}_t(x) = \mathcal{G}(x)-\alpha\int^t_0 \int_{\CX \times \CY} Q^{\nu}_s(x')A_{x,x'} \pi(dx',dy) ds,
\end{equation}
and when $\gamma = \frac{2\nu+1}{2\nu+2}$,
\begin{equation}\label{Eq:Qj_formula2}
\begin{aligned}
Q^{\nu}_t(x) &= \mathcal{G}(x) + \alpha \int^t_0 \int_{\CX \times \CY} \paren{y-Q^0_s(x')} l^{\nu}_s(B_{x,x'}(c,w))\pi(dx',dy) ds\\
&\quad -\alpha\sum_{j=1}^{\nu} \int^t_0 \int_{\CX \times \CY} Q^j_s(x')l^{\nu-j}_s (B_{x,x'}(c,w)) \pi(dx',dy) ds.
\end{aligned}
\end{equation}
Lastly, proceeding inductively based on Theorems \ref{R:BiasToZero1} and \ref{R:BiasToZero2} one gets that $Q^j_t(x)\rightarrow 0$, exponentially fast, as $t\rightarrow \infty$ for all fixed $j\geq 1$ and $x\in\mathcal{X}$.

\begin{remark}
A careful examination of the proofs in Appendixes \ref{sec::LLN}-\ref{sec::Psi} shows that the convergence Theorems \ref{LLN:theorem}, \ref{CLT:theorem} and \ref{thm::Psi} also hold if we assume that the sequence of data samples $(x_k,y_k)$ is i.i.d. from a probability distribution $\pi(dx,dy)$ that is compactly supported. We chose to make the somewhat more restrictive assumption of part (ii) of Assumption \ref{assumption} for reasons of brevity and uniformity in the presentation.
\end{remark}

\begin{remark}
For presentation purposes we have not explicitly denoted the bias term in the model (\ref{Eq:NN1}). However, it is clear that this is easily handled by designating the first component of the vector $x$ in (\ref{Eq:NN1}) to always be fixed, let's say equal to one, which would result in replacing (\ref{Eq:NN1}) by $g^N(x;\theta) = \frac{1}{N^{\gamma}} \sum_{i=1}^{N} C^i\sigma(W^i x+b^{i})$. We leave the rest of the related details to the interested reader.
\end{remark}
\section{Comparisons to the results for $\gamma=1/2$ and $\gamma=1$}\label{S:Review}

In Section \ref{sec::main_results} we presented our main results when $\gamma\in(1/2,1)$. Let us briefly review in this section the related existing results for $\gamma=1/2$, see \cite{Du1,HuangYau2020,SirignanoSpiliopoulosNN4},  and when $\gamma=1$, see \cite{Chizat2018,Montanari,RVE,SirignanoSpiliopoulosNN1,SirignanoSpiliopoulosNN2}, for completeness, for comparison and for consistency purposes.

\subsection{The case $\gamma=1/2$}
We first discuss the so-called Xavier normalization \cite{Xavier}, where $\gamma=1/2$.
\begin{theorem}[See \cite{SirignanoSpiliopoulosNN4}] \label{MainTheorem3}
Let Assumption \ref{assumption} hold, set $\alpha^{N}= \frac{\alpha}{N}$ with $0<\alpha<\infty$ and define $E = \mathcal{M}(\mathbb{R}^{1+d}) \times \mathbb{R}^{M}$. The process $(\mu_t^N, h_t^N)$ converges in distribution in the space $D_E([0,T])$ as $N \rightarrow \infty$ to $(\mu_t, h_t)$ which satisfies, for every $f \in C^2_b( \mathbb{R}^{1+d})$,
\begin{align}
h_t(x) =& \mathcal{G}(x) +   \alpha \int_{\mathcal{X} \times \mathcal{Y}} ( y - h_t(x') ) A_{x,x'}  \pi(dx',dy) dt, \qquad
\la f, \mu_t \ra = \la f, \mu_0 \ra.\notag
\end{align}

\end{theorem}

As we discussed after Theorem \ref{LLN:theorem}, as far the first order typical behavior of the neural network's output is concerned the only difference between the case $\gamma=1/2$, Theorem \ref{MainTheorem3}, and $\gamma\in(1/2,1)$,  Theorem \ref{LLN:theorem}, is that in the case of $\gamma=1/2$ $h_0(x) = \mathcal{G}(x)$, which means that the law of large numbers is random, whereas in the case $\gamma\in(1/2,1)$ it is deterministic.  Also, under the additional Assumption \ref{assumption1} the matrix $A$ is positive definite (see Lemma 3.3 of \cite{SirignanoSpiliopoulosNN4} and Proposition 2 in \cite{NTK}) which then implies that the neural network recovers the global minimum as $t\rightarrow\infty$, i.e. (\ref{Eq:ConverenceGM}) holds.

As it has already been observed in the applied literature, see \cite{Neal,SirignanoSpiliopoulosNN4}, in the case of $\gamma=1/2$, the randomness of the neural network's output in the limit is due to the randomness at initialization as specified by the Gaussian random variable $\mathcal{G}(x)\sim N(0,\tau^{2}(x))$ where we recall that $\tau^{2}(x)=\ip{|c\sigma(wx)|^{2},\mu_{0}}$.

Lastly, we mention here that by comparing the outcome of Theorems \ref{LLN:theorem} and \ref{CLT:theorem} to that of Theorem \ref{MainTheorem3} one sees that there is consistency of the results as $\gamma\rightarrow 1/2$. By this we mean that if one heursitically plugs in $\gamma=1/2$ to the outcome of Theorems \ref{LLN:theorem} and \ref{CLT:theorem}, i.e., plug in $\gamma=1/2$ in expansion (\ref{Eq:AsExp1}), one gets the result of Theorem \ref{MainTheorem3}.

\subsection{The case $\gamma=1$}
We now discuss the existing results for the mean field normalization, where $\gamma=1$. As we saw before in the case $\gamma\in[1/2,1)$ we have that $\mu^{N}\rightarrow \mu$ with $\mu_{t}$ being constant in time in the sense that for every $f \in C^2_b( \mathbb{R}^{1+d})$ $\la f, \mu_t \ra = \la f, \mu_0\ra$. In contrast to that, the behavior is different when $\gamma=1$. In particular, when $\gamma=1$, the limit measure $\mu_{t}$ is not constant over time. The details are in Theorem \ref{MainTheoremLLN}.
\begin{theorem}[See \cite{SirignanoSpiliopoulosNN1}] \label{MainTheoremLLN}
Let Assumption \ref{assumption} hold and additionally assume that $W_{0}^{i}$ have bounded fourth moments. Set the learning rate $\alpha^{N}= \alpha$ with $0<\alpha<\infty$ (independent of $N$). The scaled empirical measure $\mu^N_t$ converges in probability to $\mu_t$ in $D_E([0,T])$ as $N \rightarrow \infty$. For every $f\in C^{2}_{b}(\mathbb{R}^{1+d})$, $\mu$ satisfies the measure evolution equation
\begin{eqnarray}
\la f, \mu_t \ra  &=& \la f, \mu_0 \ra + \int_0^t   \bigg{(} \int_{\mathcal{X}\times\mathcal{Y}}   \alpha \big{(} y -  \la c \sigma(w\cdot x),  \mu_s \ra \big{)} \la C_{x}^{f}(c,w), \mu_s \ra   \pi(dx,dy)\bigg{)} ds.
\label{EvolutionEquationIntroduction}
\end{eqnarray}
\end{theorem}

The article \cite{SirignanoSpiliopoulosNN2} establishes the fluctuations around the mean-field limit of $\mu^N \overset{p} \rightarrow \mu$ as $N \rightarrow \infty$ defined by the fluctuation process
\begin{eqnarray}
\eta_t^N = \sqrt{N} ( \mu_t^N - \mu_t ).\notag
\end{eqnarray}

Then, by Theorem \ref{MainTheoremCLT} below we obtain that the fluctuations process $\{\eta_t^N\}_{N\in\mathbb{N}}$ in the limit as $N\rightarrow\infty$ behaves as a Gaussian process.
\begin{theorem}[See \cite{SirignanoSpiliopoulosNN2}] \label{MainTheoremCLT}
Let Assumption \ref{assumption} hold and additionally assume that the measure of the initial distribution $\mu_{0}$ has compact support. Set the learning rate $\alpha^{N}= \alpha$ with $0<\alpha<\infty$ (independent of $N$). Let $J \geq 3  \ceil*{\frac{d+1}{2}} + 7$.  Let $0<T<\infty$ be given. The sequence $\{\eta^{N}_{t},t\in[0,T]\}_{N\in\mathbb{N}}$ is relatively compact in $D_{W^{-J,2}}([0,T])$\footnote{$W^{-J,2}=W^{-J,2}(\Theta)$ is the dual of the Sobolev space $W^{J,2}_0=W^{J,2}_0(\Theta)$ with $\Theta \subset \R^{1+d}$ a bounded domain. For a bounded domain $\Theta \subset \R^{1+d}$, the space $W^{J,2}_0(\Theta)$ is the closure of functions of class $C^{\infty}_0(\Theta)$ in the norm defined by $\Vert f \Vert_J = ( \sum_{\vert k \vert \le J} \int_{\Theta} \vert D^kf(x) \vert^2 dx )^{1/2}$. $C^{\infty}_0(\Theta)$ is the space of all functions in $C^{\infty}(\Theta)$ with compact support. See \cite{SirignanoSpiliopoulosNN2} for details.}.  The sequence of processes $\{\eta^{N}_{t},t\in[0,T]\}_{N\in\mathbb{N}}$  converges in distribution in $D_{W^{-J,2}}([0,T])$ to the process $\{\bar{\eta}_{t},t\in[0,T]\}$, which, for every $f \in W_0^{J,2}$,  satisfies the stochastic partial differential equation
\begin{eqnarray}
\la f, \bar \eta_t \ra  &=& \la f, \bar \eta_0 \ra +    \int_0^t \int_{\mathcal{X}\times\mathcal{Y}} \alpha \big{(} y -  \la c \sigma(w x), \mu_{s} \ra \big{)} \la C_{x}^{f}(c,w), \bar \eta_{s} \ra \pi(dx,dy)  ds \notag \\
& &-  \int_0^t   \int_{\mathcal{X}\times\mathcal{Y}}  \alpha \la c \sigma(w x), \bar \eta_s \ra \la C_{x}^{f}(c,w), \mu_{s} \ra \pi(dx,dy)  ds +\la f, \bar M_t \ra.
\label{SPDEmain}
\end{eqnarray}

$\bar M_t$ is a mean-zero Gaussian process such that for every $f,g \in W_0^{J,2}$ the covariance structure is
\begin{eqnarray}
\textrm{Cov} \bigg{[} \la f, \bar M_{t} \ra, \la g, \bar M_{t} \ra \bigg{]} &=& \alpha^2 \int_0^{t} \bigg{[}  \int_{\mathcal{X}\times\mathcal{Y}} \bigg{(} \mathcal{R}_{x,y,\mu_{s}}[f] -\int_{\mathcal{X}\times\mathcal{Y}}  \mathcal{R}_{x,y,\mu_{s}}[f]\pi(dx,dy)\bigg{)}\times\nonumber\\
& &\qquad\qquad\times
\bigg{(} \mathcal{R}_{x,y,\mu_{s}}[g] -\int_{\mathcal{X}\times\mathcal{Y}}  \mathcal{R}_{x,y,\mu_{s}}[g]\pi(dx,dy)\bigg{)}
  \pi(dx,dy)\bigg{]} ds.\notag
\end{eqnarray}
where for $h\in\mathcal{C}^{1}_{0}(\mathbb{R}^{1+d})$ we have defined
$\mathcal{R}_{x,y,\mu}[h]=(y - \la c \sigma(w x), \mu \ra )  \la C_{x}^{h}(c,w), \mu \ra.$ Finally, the stochastic evolution equation (\ref{SPDEmain}) has a unique solution in $W^{-J,2}$, which implies that $\bar{\eta}$ is unique.
\end{theorem}

Let us choose now in Theorem \ref{MainTheoremLLN}, the test function to be $f(c,w;x)=c\sigma(wx)$. Then we have that the neural network's output in the limit as $N\rightarrow\infty$ satisfy the equation
\begin{eqnarray}
h_{t}(x)  &=&  \int_0^t   \bigg{(} \int_{\mathcal{X}\times\mathcal{Y}}   \alpha \big{(} y -  h_{s}(x') \big{)} \la B_{x,x'}(c,w), \mu_s \ra   \pi(dx',dy)\bigg{)} ds,
\label{Eq:LLN_gamma1}
\end{eqnarray}
where the measure $\mu_s$ is not constant any more in time and is given as the unique weak solution of the stochastic evolution equation (\ref{EvolutionEquationIntroduction}). Comparing now (\ref{Eq:LLN_gamma1}) with the result of Theorems \ref{LLN:theorem} and \ref{MainTheorem3} we notice that the main difference is that in the case of $\gamma=1$, $h_{t}$ is driven by the measure $\mu_{s}$ for $s\in[0,t]$ which in the case of $\gamma=1$ is not a constant, whereas when $\gamma\in[1/2,1)$ the limiting neural network's output $h_{t}$ is driven by the constant measure $\mu_{0}$ at initialization.

In regards to the first order fluctuations (see Theorem \ref{MainTheoremCLT}) we have that in the case $\gamma=1$, $\phi=1/2$ and that the limit of $K^N_t = \sqrt{N} \left(h^N_t - h_t\right)$ as $N\rightarrow\infty$, $K_{t}$ satisfies
\begin{eqnarray}
K_{t}(x)  &=& \mathcal{G}(x) +    \int_0^t \int_{\mathcal{X}\times\mathcal{Y}} \alpha \big{(} y -  h_{s}(x') \big{)} \la \nabla(c\sigma(w  x)) \cdot \nabla(c\sigma(w x')), \bar \eta_{s} \ra \pi(dx,dy)  ds \notag \\
& &-  \int_0^t   \int_{\mathcal{X}\times\mathcal{Y}}  \alpha \la c \sigma(w x), \bar \eta_s \ra \la \nabla(c\sigma(w  x))\cdot \nabla(c\sigma(w  x')), \mu_{s} \ra \pi(dx,dy)  ds +\la c\sigma(wx), \bar M_t \ra.
\label{SPDEmainEx}
\end{eqnarray}
where the limiting $\mu_{s}$ and $\eta_s$ are given as the unique weak solutions to the evolution equations (\ref{EvolutionEquationIntroduction}) and (\ref{SPDEmain}) respectively.

Notice that in the case $\gamma=1$, the fluctuations are of the order $\frac{1}{\sqrt{N}}$ as opposed to the order $\frac{1}{N^{\gamma-1/2}}$ for $\gamma\in(1/2,1)$ as we saw in Theorem \ref{CLT:theorem}. This means that for fixed finite $N$, the order of the variance monotonically decreases in $\gamma\in[1/2,1]$ with the smallest order observed as $\gamma\rightarrow 1$. In addition, the  limiting fluctuations $K_{t}$ as described by (\ref{SPDEmainEx}) have two sources of randomness, one is the Gaussian randomeness due to the initialization given by the Gaussian field $\mathcal{G}(x)$ and the other one being the zero mean Gaussian noise given by the martingale term $\la c\sigma(wx), \bar M_t \ra$ with variance stricture given by Theorem \ref{MainTheoremCLT} with $f(c,w;x)=c\sigma(wx)$.

Lastly, even though we do not show this here, as there is consistency in the results as $\gamma\rightarrow 1/2$, there is consistency in the results as $\gamma\rightarrow 1$ as well. In particular, by putting together (\ref{measure_expansion})-(\ref{Eq:l_equation}) and (\ref{network_expansion})-(\ref{Eq:Qj_formula1})-(\ref{Eq:Qk_formula1}) and then taking $\gamma\rightarrow 1$ and $N\rightarrow\infty$ one recovers, at least at a heuristic level, the $\gamma=1$ results of Theorem \ref{MainTheoremLLN} and (\ref{Eq:LLN_gamma1}). Given that the theoretical analysis for the $\gamma=1$ case is considerably more complicated than the theoretical analysis for the $\gamma<1$ case, see for example \cite{Chizat2018,Montanari,RVE,SirignanoSpiliopoulosNN1,SirignanoSpiliopoulosNN2,SirignanoSpiliopoulosNN3}, the latter observation may be useful in studying the $\gamma=1$ case by viewing it as an approximation of the $\gamma<1$ case as $\gamma\rightarrow 1$. Studying this here is beyond the scope of this work and thus we leave this for future work.

\section{Further discussion of the theoretical results and numerical studies}\label{S:Numerics}

The results of Section \ref{sec::main_results} establish that the neural network's output $h^{N,\gamma}_t$ has an expansion as $N\rightarrow\infty $ given by \eqref{Eq:FormalExpansion}. Beyond the observations made in Section \ref{S:Review} we can also draw the following general conclusions.
\begin{itemize}
\item{The theoretical results, and thus expansion \eqref{Eq:FormalExpansion}, hold if the learning rate is chosen to be of the order of $\alpha^N = \alpha/N^{2-2\gamma}$ for $0<\alpha<\infty$ a constant of order one.}
\item{Expansion \eqref{Eq:FormalExpansion} shows that, to leading order in $N$, there is no bias-variance trade-off, in the sense that as the number of hidden units $N \to \infty$ and the training time $t \to \infty$, both the variance and the bias decrease exponentially fast for all $\gamma \in (1/2,1)$. Indeed, to leading order in $N$, the expression \eqref{Eq:FormalExpansion} shows that the norm of the variance-covariance matrix of the vector $h_t^{N,\gamma}$ is of the order of $N^{-2(\gamma-1/2)}\Vert e^{-At}\text{Var}(\mathcal{G})e^{-A^{\top}t}\Vert$ and the bias is of the order of $\sum_{j=1}^{\nu-1} N^{-j(1-\gamma)}\Vert Q^j_t \Vert$. This directly implies that for fixed $\gamma$ both variance and bias decay as $N,t\rightarrow\infty$. Recall that, under Assumption \ref{assumption1},  the matrix $A$ is positive definite and for each $j=1,\ldots, \nu-1$, and fixed $\nu\in\mathbb{N}$, $\Vert Q^j_t \Vert$ decay exponentially fast to zero as $t\rightarrow\infty$. }
\item{For fixed (but large) $N<\infty$ and $t<\infty$, we see that the magnitude of the variance of the network output $N^{-2(\gamma-1/2)}\Vert e^{-At}\text{Var}(\mathcal{G})e^{-A^{\top}t}\Vert$ is monotonically decreasing  as $\gamma \to 1$. In addition, as also implied by Theorem \ref{MainTheorem3}, when $\gamma=1/2$, the variance, to leading order, is of order $\Vert e^{-At}\text{Var}(\mathcal{G})e^{-A^{\top}t}\Vert$, i.e., is independent of $N$.}
\item{When $\gamma > 3/4$ (which corresponds to $\nu>1$), there is, to leading order in $N$, a bias effect, which for fixed $N$ is larger when $\gamma \to 1$. Notice that for $\gamma\in(1/2,3/4)$ there is no bias involving $N$ to leading order. In particular, $\nu=1$ in this case and there is no contribution from the term  $\sum_{j=1}^{\nu-1} N^{-j(1-\gamma)}\Vert Q^j_t\Vert$.}
\item{As the following numerical results demonstrate, both train and test accuracy of the fitted neural networks increases monotonically as  $\gamma$ approaches the value 1 which corresponds to the mean field normalization.}
\end{itemize}

In this section, we investigate numerically the performance of neural networks scaled by $1/N^{\gamma}$ with $\gamma\in[1/2,1]$. Apart from characterizing performance based on the behavior of bias and variance of the neural network's output from the leading order terms, one would also like to characterize performance based on train and test accuracy. For this reason, we have performed a number of numerical studies to compare train and test accuracy for neural networks scaled by the normalization $1/N^{\gamma}$. Our numerical studies involve the well known {MNIST \cite{MNIST} and the CIFAR10  \cite{Krizhevsky} data sets}. To be more precise, for each numerical study in this section, we train scaled neural networks with different hidden units $N=100$, $N=500$, $N=1,000$, and $N=3,000$, and for each $N$, with $\gamma = 0.5$, $\gamma = 0.6$, $\gamma = 0.7$, $\gamma = 0.8$, $\gamma = 0.9$, and $\gamma = 1$. We compare the test accuracy for different $\gamma$-scaled networks for each hidden unit $N$.
We find numerically that both train and test accuracy of the fitted neural networks increases monotonically in $\gamma \in [1/2,1]$. This suggests that the mean-field normalization $1/N$ that corresponds to $\gamma=1$ may have certain advantages over scalings $1/N^{\gamma}$ for $\gamma\in[1/2,1)$ when it comes to test accuracy.

The result of the first numerical study is presented in Figure \ref{Fig:numeric_mnist_ce}. Scaled single layer neural networks are trained via cross entropy loss function to classify images from the MNIST dataset \cite{MNIST}, which includes 70,000 images of handwritten integers from 0 to 9. The learning rate is taken to be $\alpha^N = 1/N^{2-2\gamma}$, as suggested by our theoretical analysis. The neural networks are trained to identify the handwritten numbers using the image pixels as an input. In the MNIST dataset, each image has 784 pixels, 60,000 images are used as train images and 10,000 images are test images. We train scaled single layer neural networks with different hidden units for each $\gamma$, and the test accuracy for each network as the number of epochs increases are shown in Figure \ref{Fig:numeric_mnist_ce}. While for each number of hidden units, all scaled networks eventually achieve reasonably high test accuracy, we can see that as $\gamma$ increases from 0.5 to 1,  the test accuracy increases. As the number of hidden units increases, the dispersion in test accuracy for different values of $\gamma$ increases.
\begin{figure}[ht!]
  \centering
  \begin{subfigure}[b]{0.45\linewidth}
    \includegraphics[width=\linewidth]{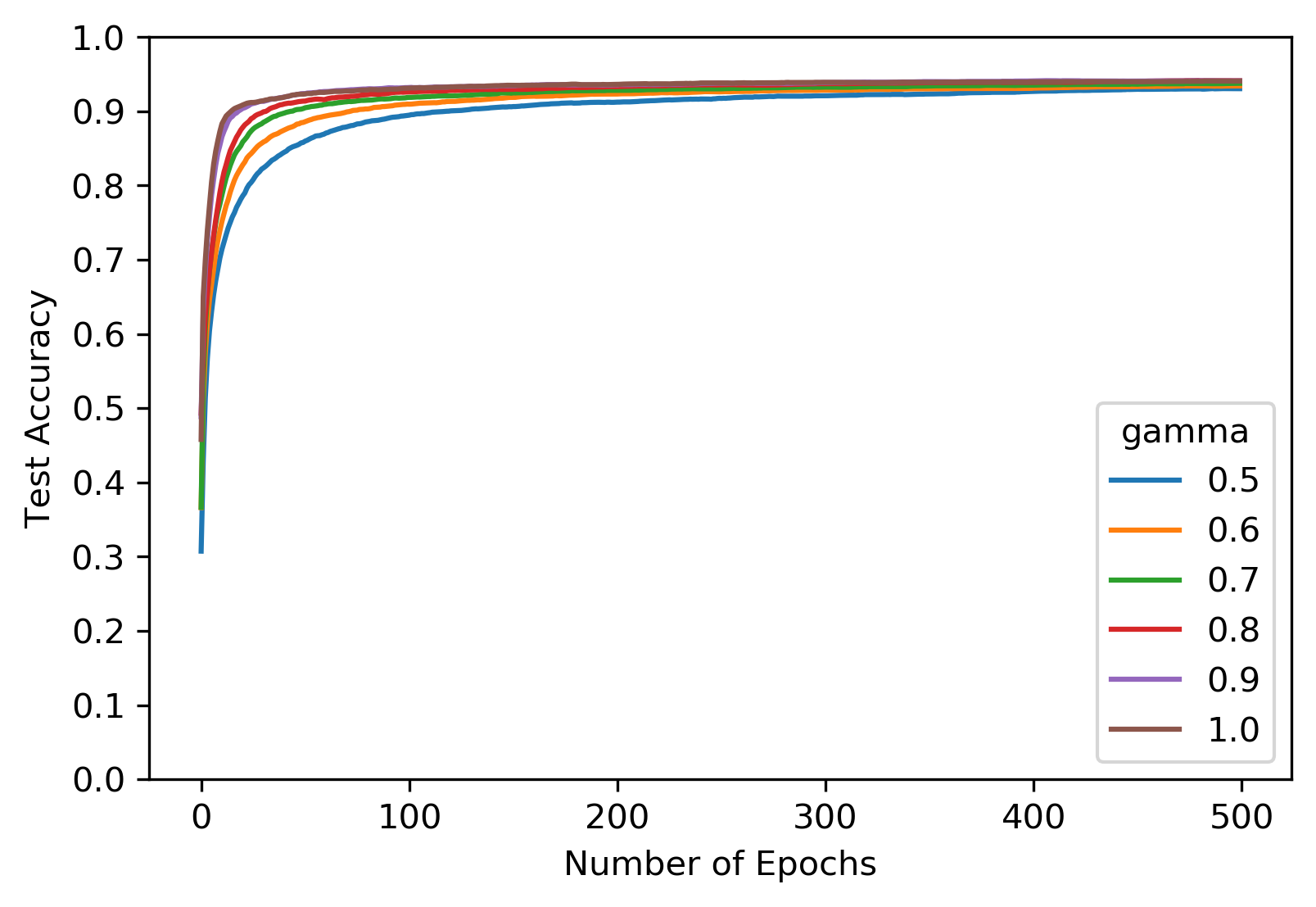}
     \caption{$N=100$ hidden units.}
  \end{subfigure}
  \begin{subfigure}[b]{0.45\linewidth}
    \includegraphics[width=\linewidth]{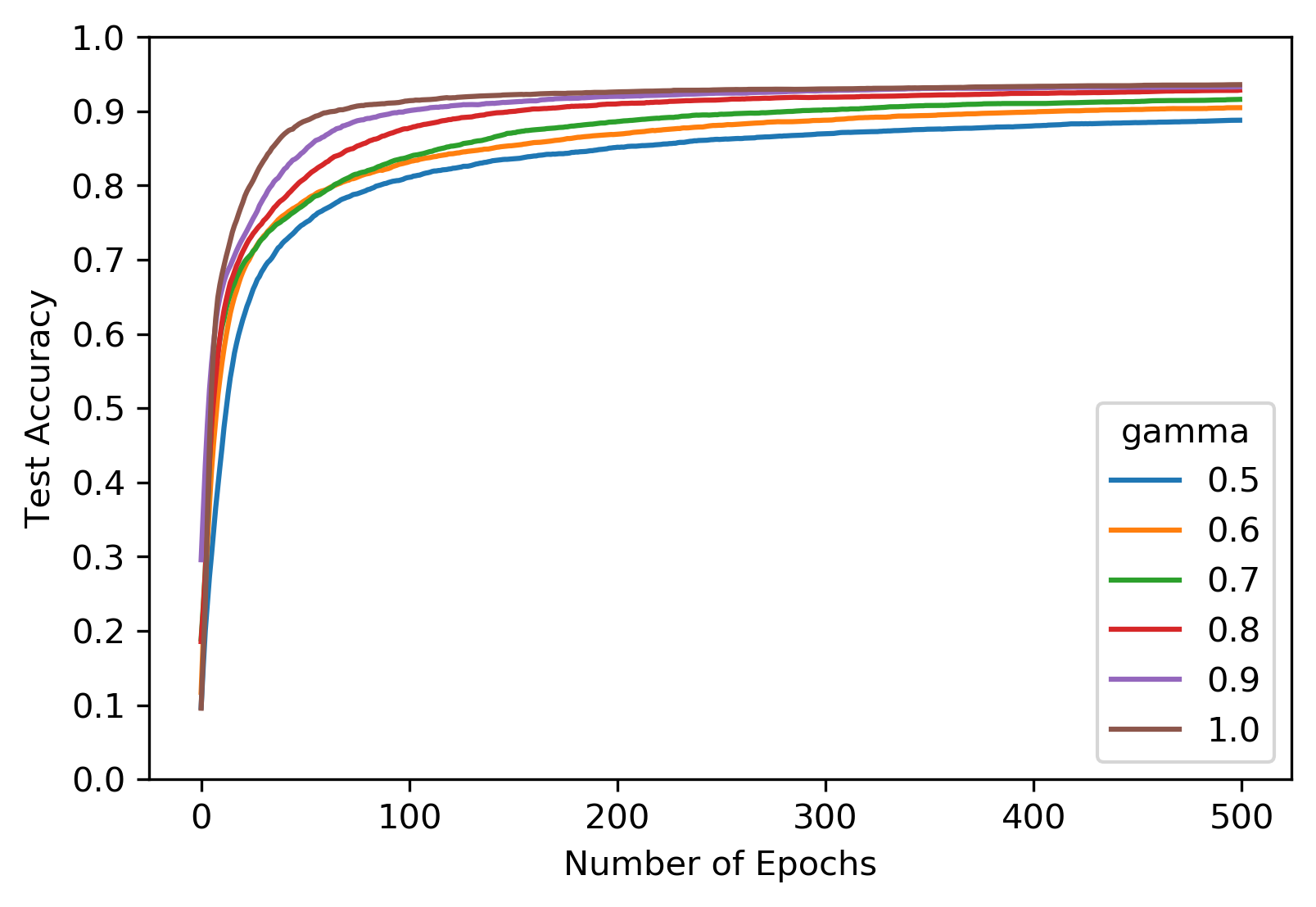}
    \caption{$N=500$ hidden units.}
  \end{subfigure}
  \begin{subfigure}[b]{0.45\linewidth}
    \includegraphics[width=\linewidth]{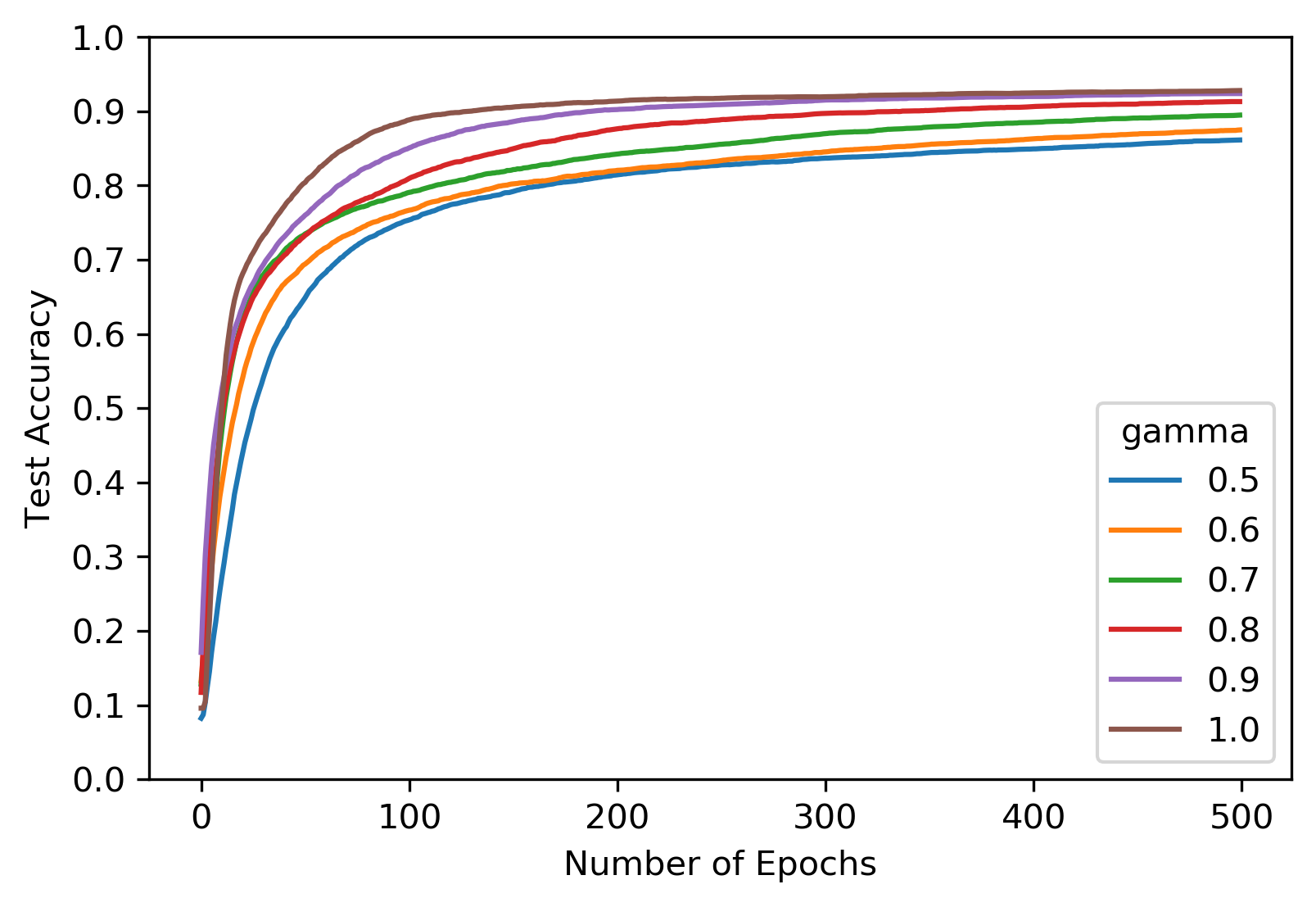}
    \caption{$N=1,000$ hidden units.}
  \end{subfigure}
  \begin{subfigure}[b]{0.45\linewidth}
    \includegraphics[width=\linewidth]{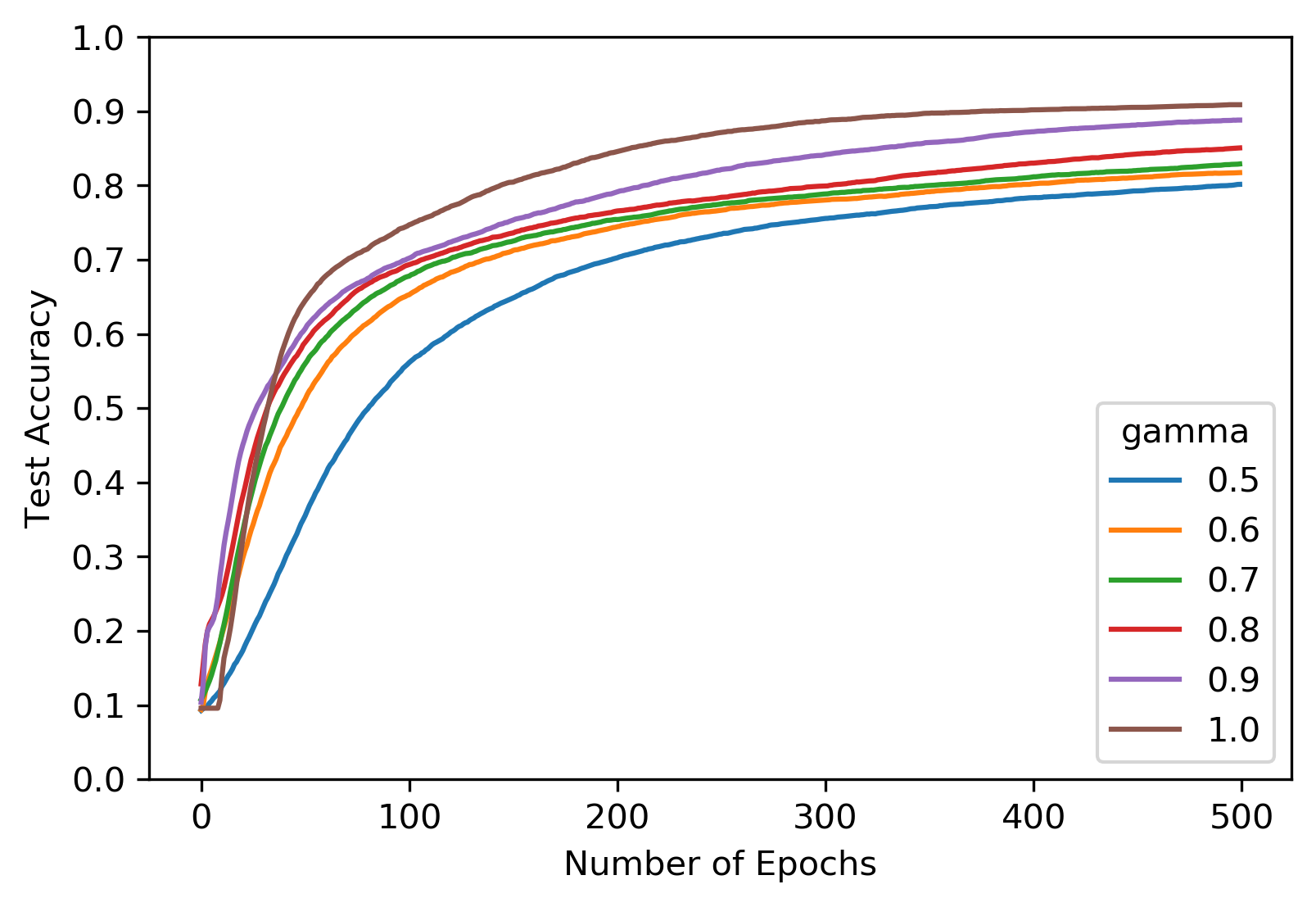}
    \caption{$N=3,000$ hidden units.}
  \end{subfigure}
  \caption{Performance of scaled neural networks on MNIST test dataset (cross entropy loss).}
  \label{Fig:numeric_mnist_ce}
\end{figure}

The result of the second numerical study is presented in Figure \ref{Fig:numeric_mnist_mse}. Scaled single layer neural networks are trained via mean squared error loss function to classify images from the MNIST dataset. Specifically, we apply the standard one-hot encoding to the image labels and apply the MSE loss to the encoded labels and the model output. The learning rate is taken to be $\alpha^N = 1/N^{2-2\gamma}$ as in the first numerical study. We train scaled neural networks for each $\gamma$ and hidden units $N$, and the test accuracy for each network as the number of epochs increases are shown in Figure \ref{Fig:numeric_mnist_mse}. As for the models trained via cross entropy loss, the test accuracy increases as $\gamma$ increases from 0.5 to 1, and the dispersion in test accuracy for different values of $\gamma$ increases as the number of hidden units increases. As expected though, the test accuracy is consistently worse for the models trained via mean squared error loss than for those trained via cross entropy loss.
\begin{figure}[ht!]
  \centering
  \begin{subfigure}[b]{0.45\linewidth}
    \includegraphics[width=\linewidth]{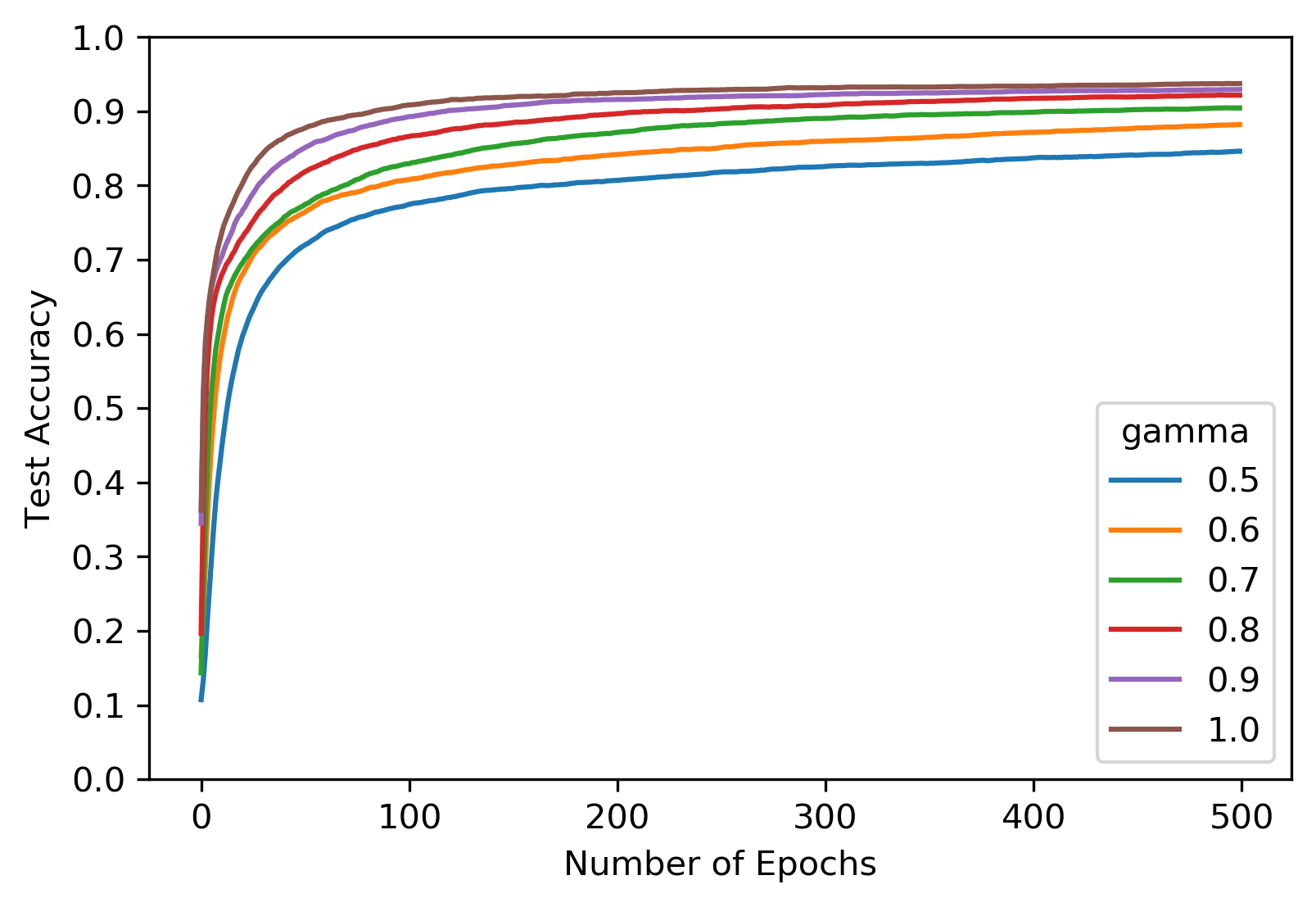}
     \caption{$N=100$ hidden units.}
  \end{subfigure}
  \begin{subfigure}[b]{0.45\linewidth}
    \includegraphics[width=\linewidth]{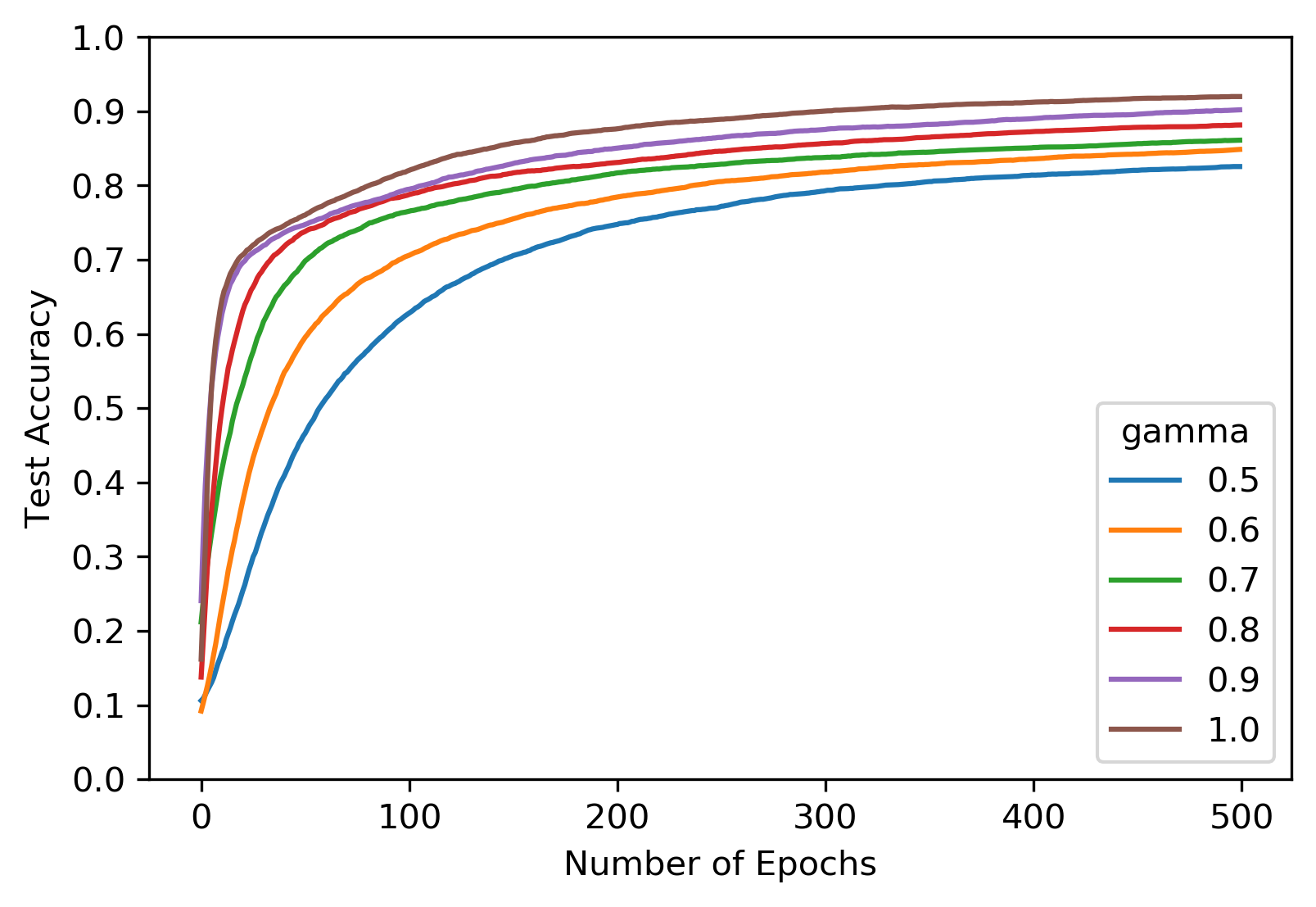}
    \caption{$N=500$ hidden units.}
  \end{subfigure}
  \begin{subfigure}[b]{0.45\linewidth}
    \includegraphics[width=\linewidth]{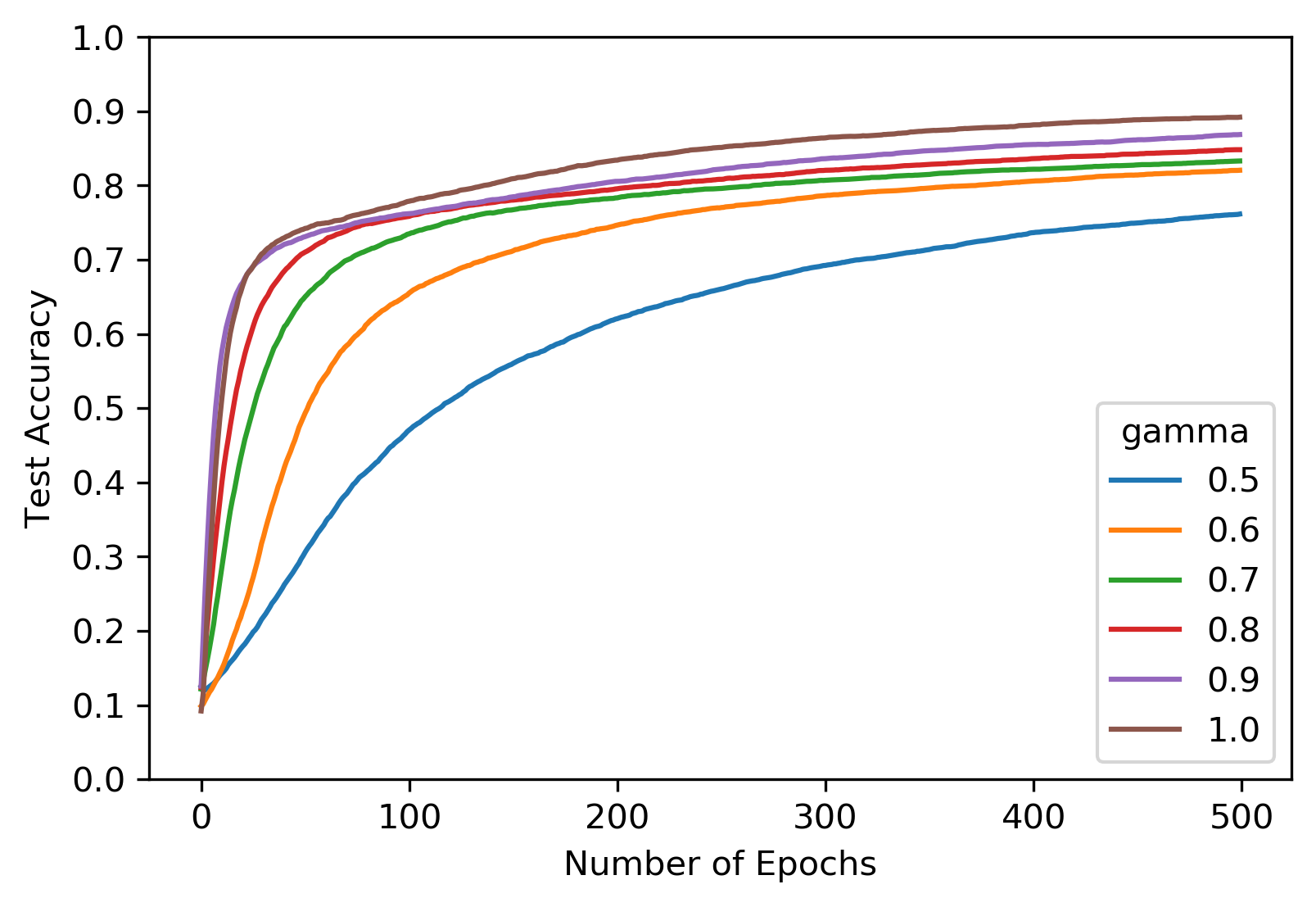}
    \caption{$N=1,000$ hidden units.}
  \end{subfigure}
  \begin{subfigure}[b]{0.45\linewidth}
    \includegraphics[width=\linewidth]{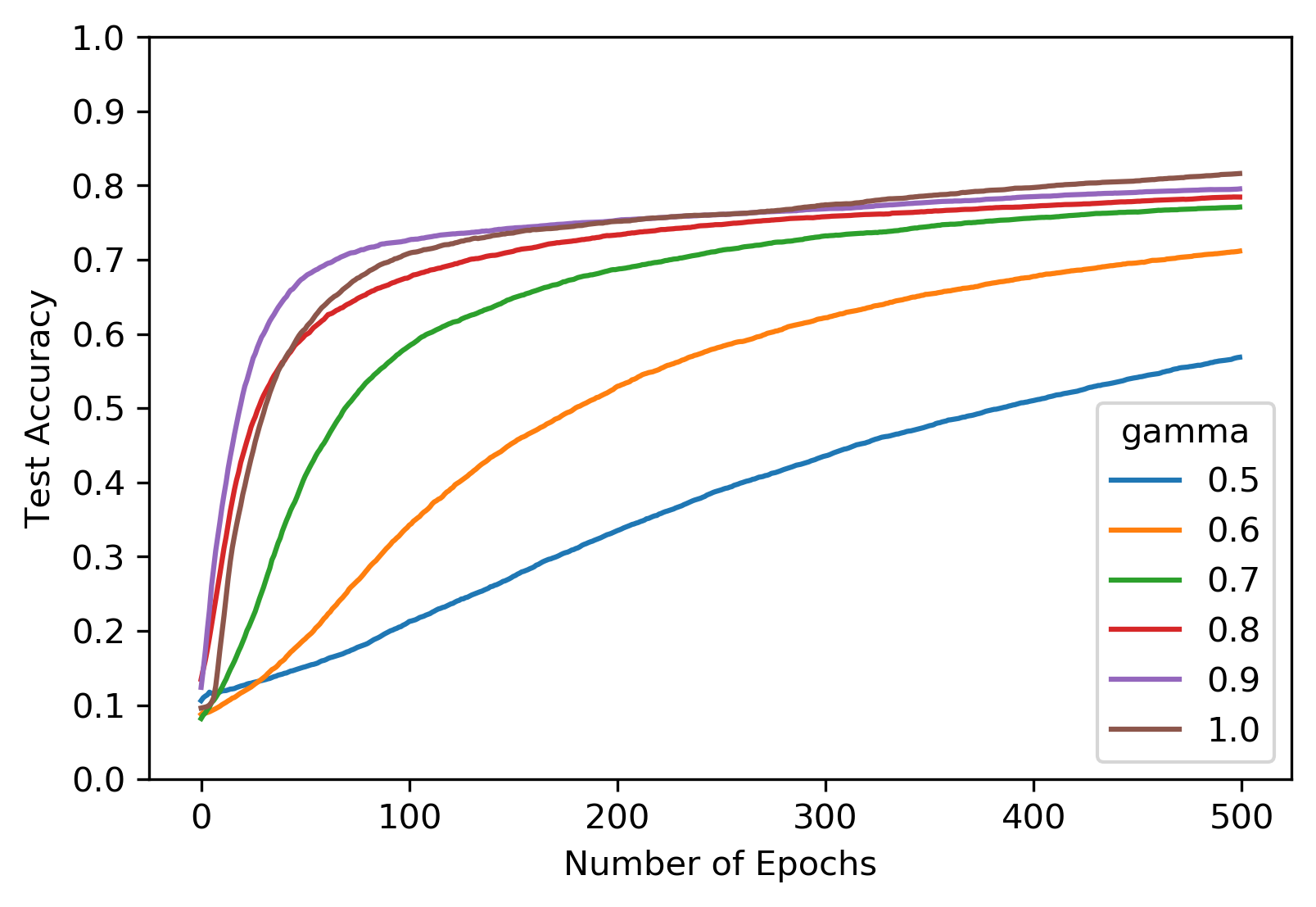}
    \caption{$N=3,000$ hidden units.}
  \end{subfigure}
  \caption{Performance of scaled neural networks on MNIST test dataset (MSE loss).}
  \label{Fig:numeric_mnist_mse}
\end{figure}

For the third numerical study, scaled neural networks are trained to classify images from the CIFAR10 dataset  \cite{Krizhevsky}, which contains 60,000 color images in 10 classes (airplane, automobile, bird, cat, deer, dog, frog, horse, ship, and truck). The dataset is divided into 50,000 training images and 10,000 test images, and each image has $32 \times 32 \times 3$ pixels. Neural networks are trained to correctly classify each image using image pixels as the input. We train scaled convolutional neural networks for the CIFAR10 datatset, and the results are shown in Figure \ref{Fig:numeric_cnn_cifar10}. For each of the networks, there are 8 convolution layers which each have 64 channels. A fully-connected layer then follows. We note that for all cases shown in Figure  \ref{Fig:numeric_cnn_cifar10}, the test accuracy is improved  as $\gamma$ gets closer to 1. The dispersion in test accuracy for different values of $\gamma$ does not seem to change significantly as the number of hidden units increases. 
\begin{figure}[ht!]
  \centering
  \begin{subfigure}[b]{0.45\linewidth}
    \includegraphics[width=\linewidth]{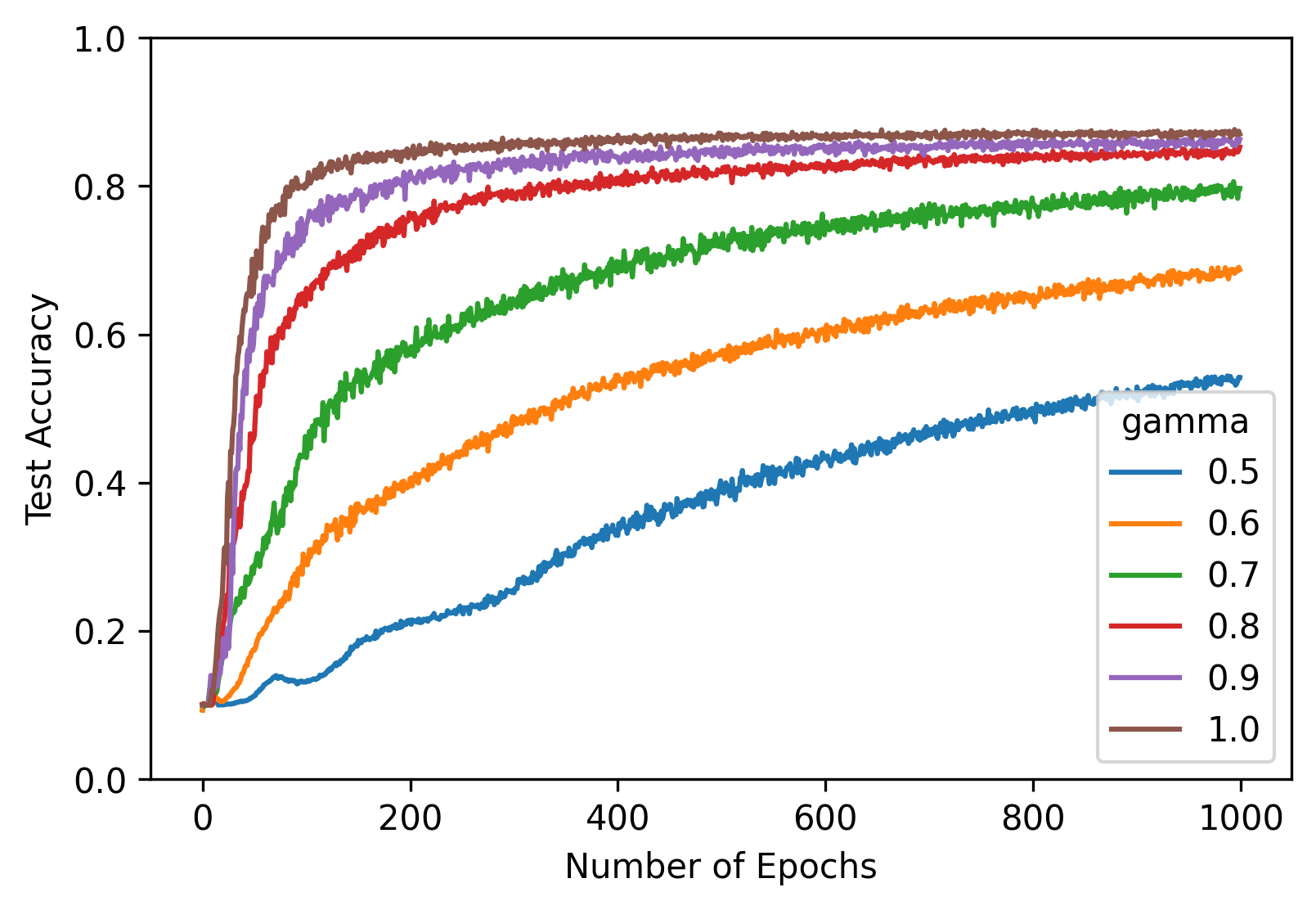}
     \caption{$N=100$ hidden units.}
  \end{subfigure}
  \begin{subfigure}[b]{0.45\linewidth}
    \includegraphics[width=\linewidth]{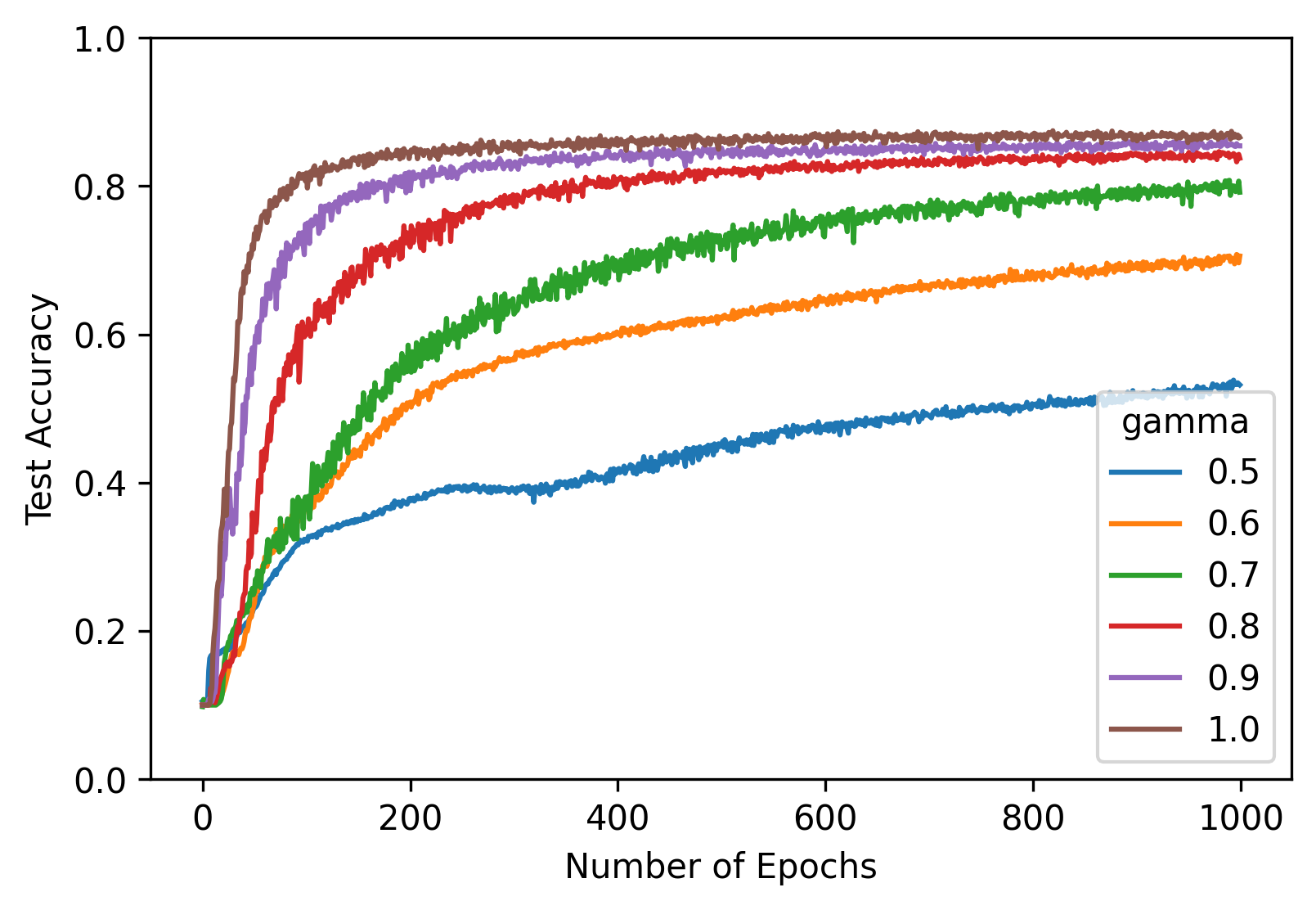}
    \caption{$N=500$ hidden units.}
  \end{subfigure}
  \begin{subfigure}[b]{0.45\linewidth}
    \includegraphics[width=\linewidth]{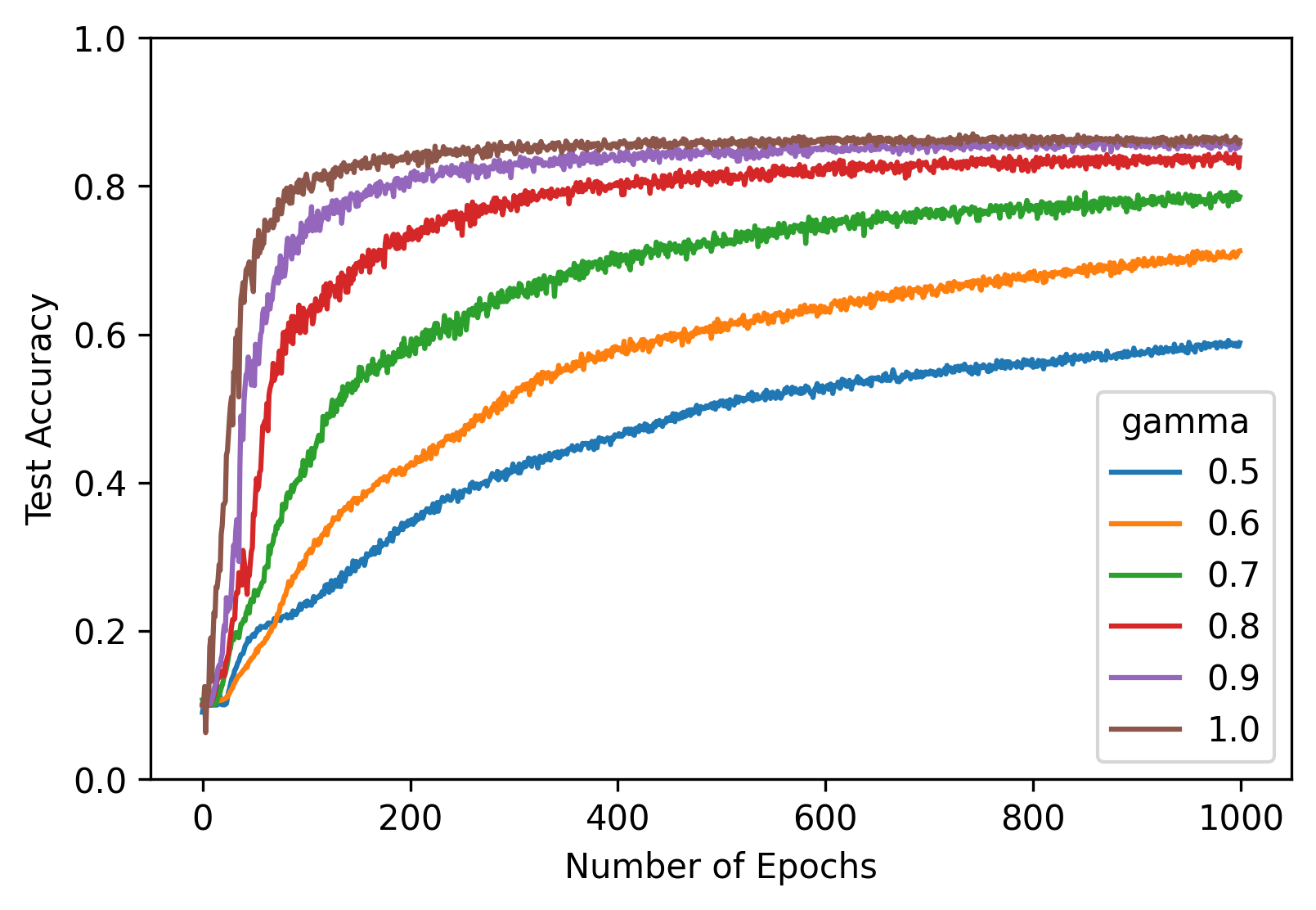}
    \caption{$N=1,000$ hidden units.}
  \end{subfigure}
  \begin{subfigure}[b]{0.45\linewidth}
    \includegraphics[width=\linewidth]{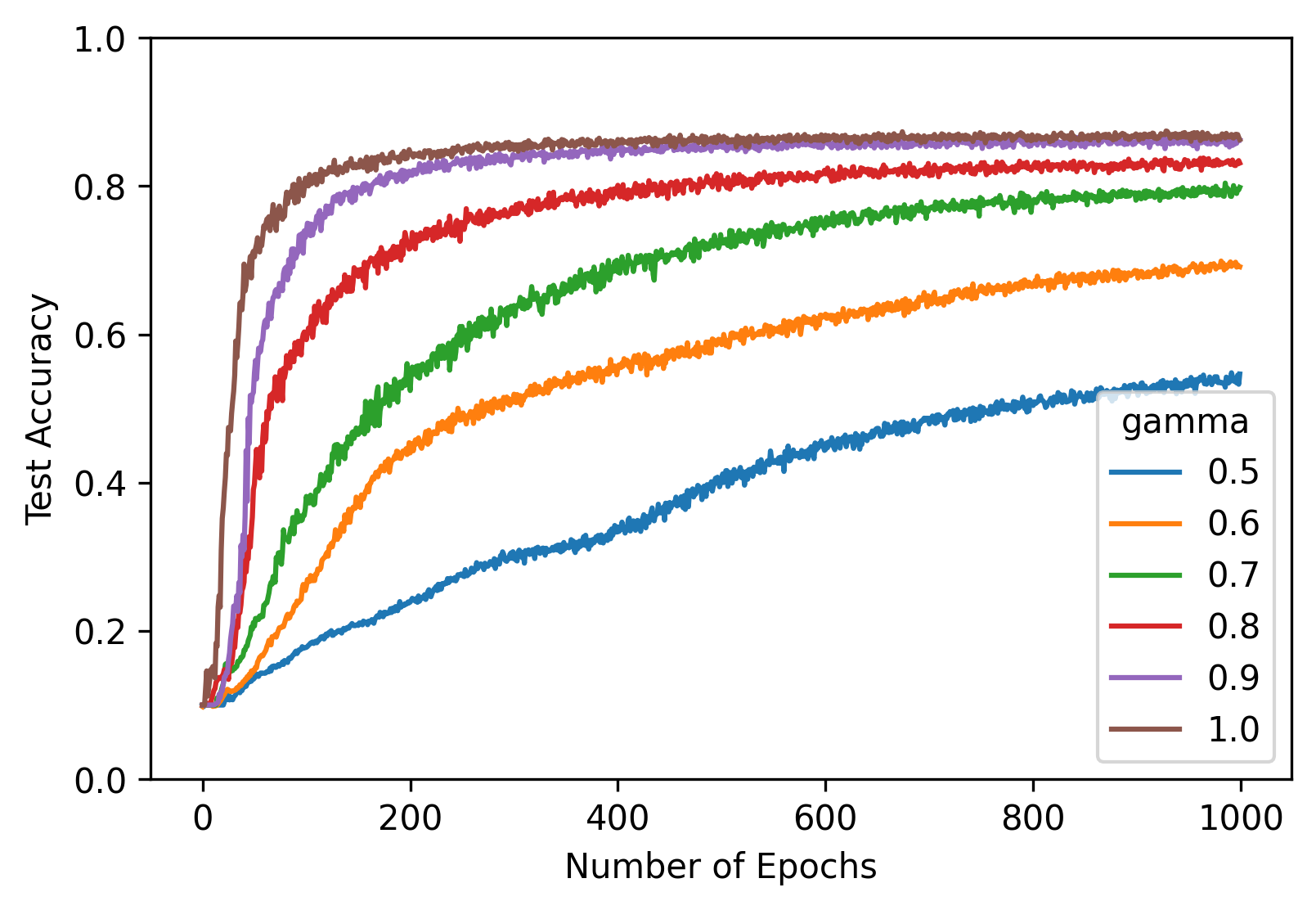}
    \caption{$N=3,000$ hidden units.}
  \end{subfigure}
  \caption{Performance of scaled convolutional neural networks on CIFAR10 test dataset (cross entropy loss).}
  \label{Fig:numeric_cnn_cifar10}
\end{figure}

Finally, we considered in-sample training accuracy for all models. In each case, the performance showed the same pattern as for the out-of-sample data. Namely, the in-sample accuracy increases monotonically in $\gamma\in[1/2,1]$. Two examples are shown in Figure \ref{Fig:numeric_mnist_mse_train} for the models trained using the mean squared error loss for the MNIST dataset.

\begin{figure}[ht!]
  \centering
  \begin{subfigure}[b]{0.45\linewidth}
    \includegraphics[width=\linewidth]{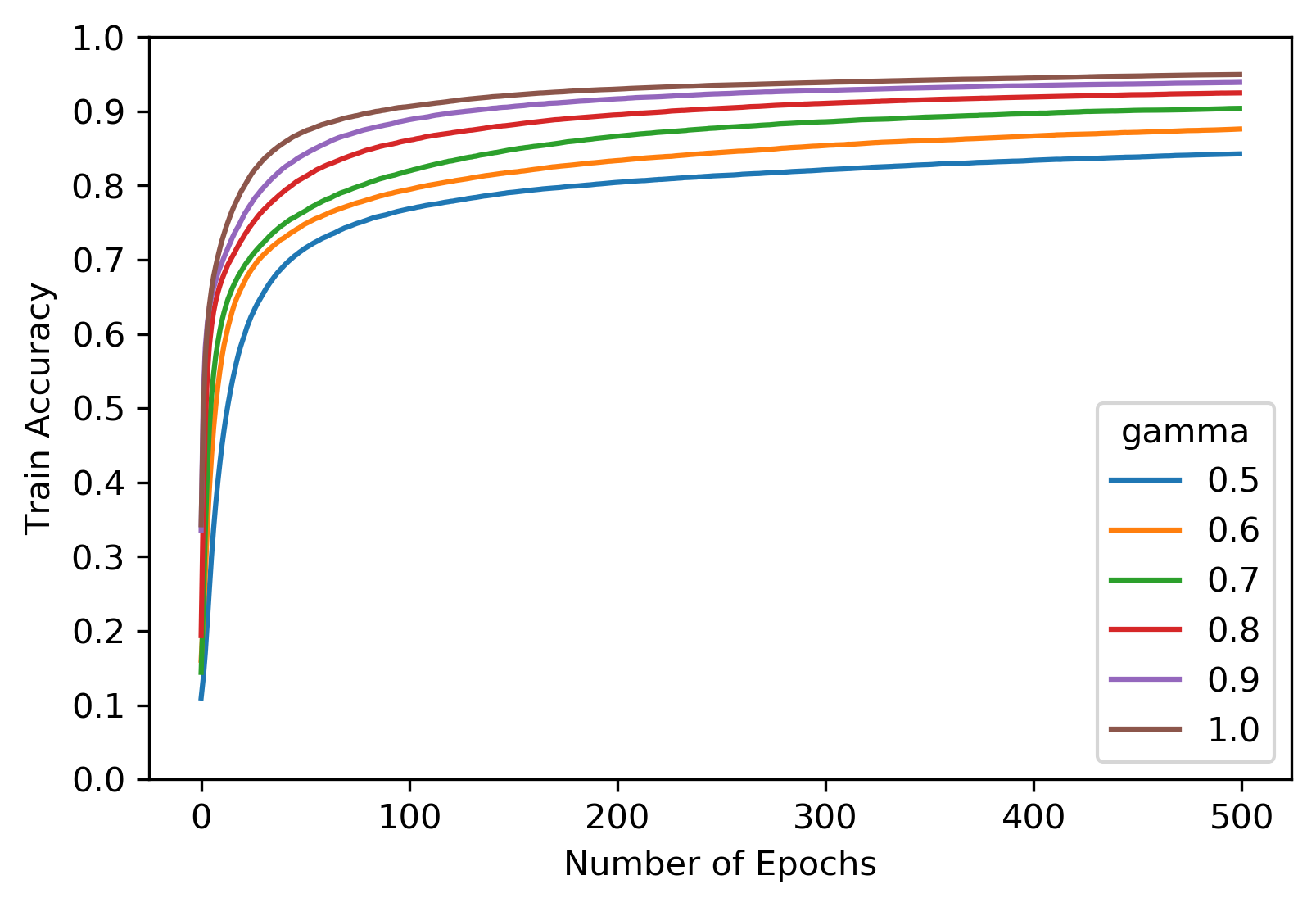}
     \caption{$N=100$ hidden units.}
  \end{subfigure}
  \begin{subfigure}[b]{0.45\linewidth}
    \includegraphics[width=\linewidth]{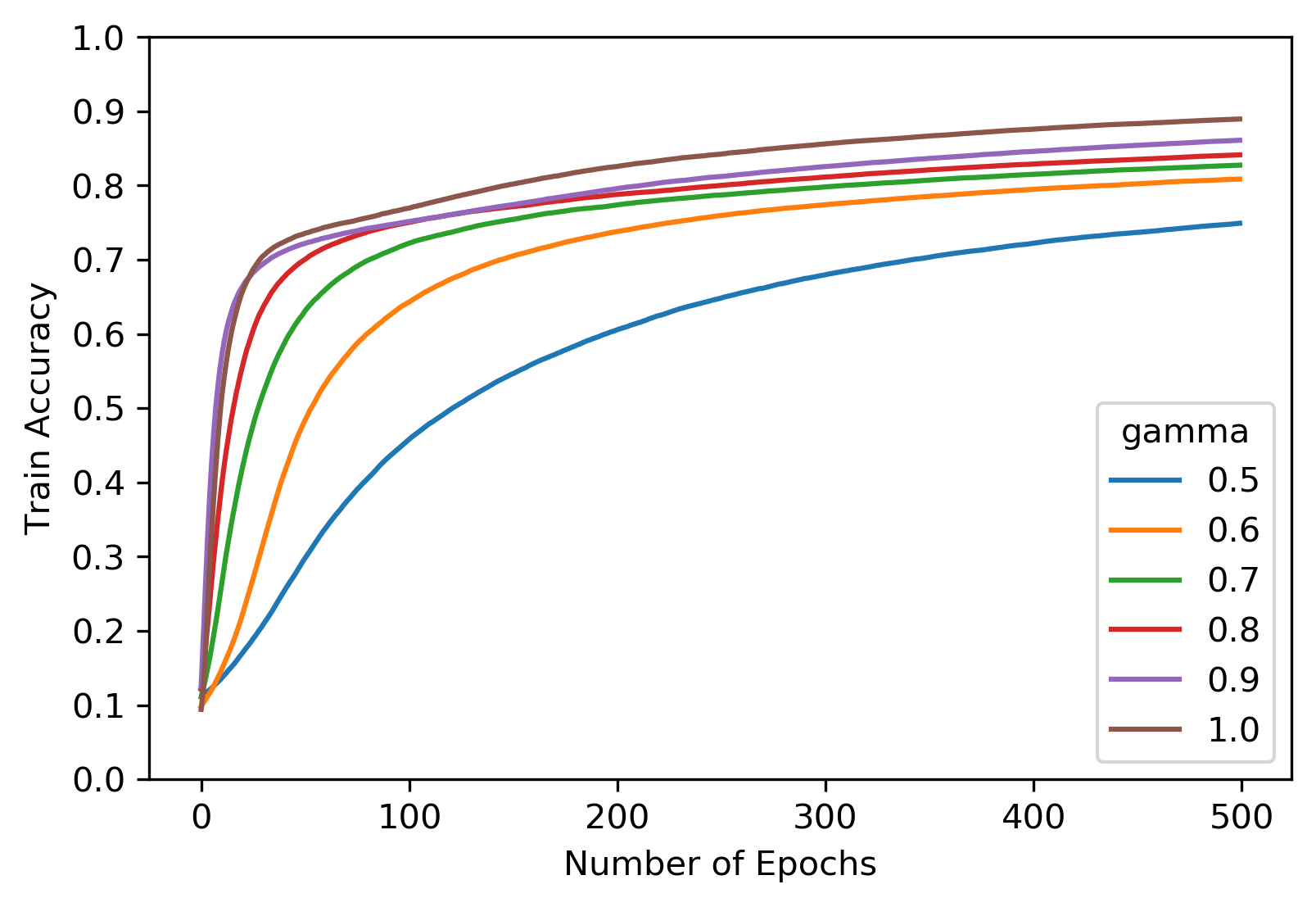}
    \caption{$N=1000$ hidden units.}
  \end{subfigure}
  \caption{Performance of scaled neural networks on MNIST training dataset (MSE loss).}
  \label{Fig:numeric_mnist_mse_train}
\end{figure}

\section{Conclusions}\label{S:Conclusions}
In this work we studied the behavior of shallow neural networks as the number of hidden units $N$ and gradient descent steps grow to infinity. We provided a stochastic Taylor kind of expansion up to second order of the neural network's output around its limit as the number of hidden units grow to infinity. The asymptotic expansion reveals what the main contribution to the variance and the bias of the neural network's output is, depending on the normalization as this ranges from the square-root normalization to the mean-field normalization. We found that, to leading order in $N$, there is no variance-bias trade-off. Also, we found empirically for the MNIST and CIFAR10 data sets that both train and test accuracy improve
 monotonically as the normalization approaches the mean-field normalization.


\appendix

\section{Proof of Theorem \ref{LLN:theorem}: Convergence of the Network Output}\label{sec::LLN}
In this section, we will study convergence in distribution of the random process $(\mu_t^N,h_t^N)$ as $N \to \infty$ in $D_E([0,T])$ where $E = \CM(\R^{1+d}) \times \R^M$. After making the appropriate identifications, the proofs in this section follow as in \cite{SirignanoSpiliopoulosNN1}. Instead of repeating the proofs here, we will provide sufficient motivation for the results and  mention what is needed for the later sections. Our focus in this section is to recall the main arguments needed for the proof of Theorem \ref{LLN:theorem}, but mainly to highlight the different terms in the expansions for $h_t^{N}$ and $\left<f,\mu^{N}_{t}\right>$ as those terms will play a crucial role in the proofs that follow in the later sections \ref{sec::CLT} and \ref{sec::Psi}.

Recall that $\alpha_N = {\alpha}/{N^{\beta}}$ for $\beta=2-2\gamma$. A useful a-priori bound for the parameters $(C^i_k,W^i_k)$ given by \eqref{SGD_iteration} is given in Lemma \ref{lemma_bound}, whose proof follows along the same lines as that of Lemma 2.1 in \cite{SirignanoSpiliopoulosNN1}.

\begin{lemma}\label{lemma_bound}
For $k \le \floor{TN}$, $i \in \mathbb{N}$ and $p \ge 1$, there exists a constant $C(T,p)<\infty$  such that
\begin{align*}
\sup_{N \in \mathbb{N},k \le \floor{TN}}  \left[\abs{C^i_k}^p +\E\paren{\norm{W^i_k}}\right] \le C(T,p).
\end{align*}
In addition, we have that
\begin{align*}
\sup_{i,k\in\mathbb{N}}\abs{C^i_{k+1} - C^i_k} &= O \paren{N^{-(2\gamma + \beta -1)}} \textrm{ and }
\sup_{i,k\in\mathbb{N}} \norm{W^i_{k+1} - W^i_k} = O \paren{N^{-(2\gamma + \beta -1)}}.
\end{align*}
\end{lemma}

\subsection{Evolution of the Pre-limit Process}\label{LLN:sec:pre_evolution}

Let us next consider the evolution of the network's output $g^N_k(x)$. Using a Taylor expansion, we have
\begin{equation}\label{network_evolution}
\begin{aligned}
g^N_{k+1}(x) - g^N_{k}(x) &= \frac{1}{N^{\gamma}} \sum_{i=1}^N C^i_{k+1} \sigma\paren{W^i_{k+1}  x} - \frac{1}{N^{\gamma}} \sum_{i=1}^N C^i_{k} \sigma\paren{W^i_{k} x}\\
& = \frac{1}{N^{\gamma}} \sum_{i=1}^N \paren{C^i_{k+1} - C^i_{k}} \left[  \sigma(W^i_{k} x) +  \sigma'(W^{i*}_{k}  x) \paren{W^i_{k+1} - W^i_{k}}  x\right]\\
&\quad + \frac{1}{N^{\gamma}} \sum_{i=1}^N \brac{\sigma'(W^{i*}_k x) \paren{W^i_{k+1} - W^i_k}x + \frac{1}{2} \sigma''(W^{i**}_k x) \left[\paren{W^i_{k+1} - W^i_k}x\right]^2}C^i_k,
\end{aligned}
\end{equation}
for some points $W^{i*}_{k}$ and $W^{i**}_k$ in the line segment connecting $W^i_k$ and $W^i_{k+1}$. Substituting \eqref{SGD_iteration} in \eqref{network_evolution}, we have
\begin{equation}\label{g_evolution}
\begin{aligned}
g^N_{k+1}(x)- g^N_{k}(x)  &=  \frac{\alpha}{N^{2\gamma+\beta}} \sum_{i=1}^N \paren{y_k-g^N_k(x_k)}  \sigma(W^i_{k} x_k) \sigma(W^i_{k} x)\\
&\quad +\frac{\alpha}{N^{2\gamma+\beta}} \sum_{i=1}^N \sigma'(W^i_{k} x)\paren{y_k-g^N_k(x_k)} (C^i_k)^2 \sigma'(W^i_{k} x_k)x_k x + \frac{G^N_k(x)}{N^{5\gamma + 2\beta -3}},
\end{aligned}
\end{equation}
where
\begin{align*}
G^N_k (x) = N^{4\gamma+2\beta-3} \sum_{i=1}^N \left\lbrace\paren{C^i_{k+1} - C^i_k}\sigma'(W^{i*}_k x)\paren{W^i_{k+1} - W^i_{k}}x +\frac{1}{2} C^i_k\sigma''(W^{i**}_k x) \left[\paren{W^i_{k+1} - W^i_k}x\right]^2\right\rbrace.
\end{align*}

Due to Lemma \ref{lemma_bound}, we have $\sup_{x\in\mathcal{X}}\abs{G^N_k (x)} < C<\infty$. Rewriting the evolution of network in terms of the empirical measure $\nu_k^N$  gives
\begin{equation*}
\begin{aligned}
g^N_{k+1}(x)- g^N_{k}(x)   &= \frac{\alpha}{N^{2\gamma+\beta-1}} \paren{y_k-g^N_k(x_k)}  \ip{\sigma(w x_k) \sigma(w x),\nu^N_k}\\
&\quad +\frac{\alpha}{N^{2\gamma+\beta-1}} \paren{y_k-g^N_k(x_k)}x_k x \ip{\sigma'(w x)  \sigma'(w x_k)c^2,\nu^N_k} +  \frac{G^N_k (x)}{N^{5\gamma + 2\beta -3}}
\end{aligned}
\end{equation*}
We can then write the evolution of $h_t^N(x)$ for $t\in[0,T]$ as
\begin{equation*}
\begin{aligned}
h^N_{t}(x)- h^N_{0}(x) &= \sum_{k=0}^{\floor{Nt}-1} \paren{g^N_{k+1}(x) -g^N_{k}(x)}= \frac{\alpha}{N^{2\gamma+\beta-1}} \sum_{k=0}^{\floor{Nt}-1}\paren{y_k-g^N_k(x_k)}  \ip{\sigma(w x_k) \sigma(w x),\nu^N_k}\\
&\quad +\frac{\alpha}{N^{2\gamma+\beta-1}}\sum_{k=0}^{\floor{Nt}-1} \paren{y_k-g^N_k(x_k)}x_k x \ip{\sigma'(w x)  \sigma'(w x_k)c^2,\nu^N_k}  + \frac{1}{N^{5\gamma + 2\beta -3}} \sum_{k=0}^{\floor{Nt}-1} G^N_{k}(x)
\end{aligned}
\end{equation*}

The summations in the above equation can be decomposed into a drift and martingale component:
\begin{equation*}
\begin{aligned}
&h^N_{t}(x) -h^N_{0}(x) =\frac{\alpha}{N^{2\gamma+\beta-1}} \sum_{k=0}^{\floor{Nt}-1} \int_{\CX \times \CY} \paren{y-g^N_k(x')}  \ip{\sigma(w x') \sigma(w x),\nu^N_k} \pi(dx',dy)\\
&\quad +\frac{\alpha}{N^{2\gamma+\beta-1}}\sum_{k=0}^{\floor{Nt}-1}\int_{\CX \times \CY} \paren{y-g^N_k(x')}x' x \ip{\sigma'(w x)  \sigma'(w x')c^2,\nu^N_k} \pi(dx',dy)\\
&\quad +   \frac{\alpha}{N^{2\gamma+\beta-1}}\left \lbrace \sum_{k=0}^{\floor{Nt}-1}\paren{y_k-g^N_k(x_k)}  \ip{\sigma(w x_k) \sigma(w x),\nu^N_k}\right.\\
&\qquad \qquad \qquad \qquad - \int_{\CX \times \CY} \paren{y-g^N_k(x')}  \ip{\sigma(w x') \sigma(w x),\nu^N_k} \pi(dx',dy)  \\
&\qquad \qquad \qquad \qquad +  \sum_{k=0}^{\floor{Nt}-1} \paren{y_k-g^N_k(x_k)}x_k x \ip{\sigma'(w x)  \sigma'(w x_k)c^2,\nu^N_k}\\
&\qquad \qquad \qquad \qquad\left. -\int_{\CX \times \CY} \paren{y-g^N_k(x')}x' x \ip{\sigma'(w x)  \sigma'(w x')c^2,\nu^N_k} \pi(dx',dy)  \right \rbrace\\
&\quad + \frac{1}{N^{5\gamma + 2\beta -3}} \sum_{k=0}^{\floor{Nt}-1} G^N_{k}(x).
\end{aligned}
\end{equation*}

Since $2\gamma + \beta -1 =1$, we can rewrite the equation for $h^N_t(x)$ in terms of a Riemann integral and the scaled measure $\mu^N_t$. In particular, we get
\begin{equation}\label{LLN:evolution_h}
\begin{aligned}
h^N_{t}(x) - h^N_{0}(x) &= \alpha \int_0^t  \int_{\CX \times \CY} \paren{y-h_s^N(x')}  \ip{B_{x,x'}(c,w),\mu^N_s} \pi(dx',dy) ds\\
&\quad  + V_t^{N}(x) + M_t^{N}(x) + \frac{1}{N^{\gamma+1}} \sum_{k=0}^{\floor{Nt}-1} G^N_{k} (x),
\end{aligned}
\end{equation}
where
$M_t^{N}(x) = M_t^{1,N}(x) + M_t^{2,N}(x)$,
\begin{align*}
V_t^{N}(x) &= - \int_{\frac{\floor{Nt}}{N}}^t \int_{\CX \times \CY} \alpha \paren{y-h^N_s(x')}\ip{B_{x,x'}(c,w),\mu^N_s}\pi(dx',dy)ds,\\
M_t^{1,N}(x) &=  \frac{\alpha}{N}\sum_{k=0}^{\floor{Nt}-1}\paren{y_k-g^N_k(x_k)}  \ip{\sigma(w x_k) \sigma(w x),\nu^N_k} \\
&\quad - \int_{\CX \times \CY} \paren{y-g^N_k(x')}  \ip{\sigma(w x') \sigma(w x),\nu^N_k} \pi(dx',dy), \\
M_t^{2,N}(x) &= \frac{\alpha}{N}\sum_{k=0}^{\floor{Nt}-1} \paren{y_k-g^N_k(x_k)}x_k x \ip{\sigma'(w x)  \sigma'(w x_k)c^2,\nu^N_k}\\
&\quad -\int_{\CX \times \CY} \paren{y-g^N_k(x')}x' x \ip{\sigma'(w x)  \sigma'(w x')c^2,\nu^N_k} \pi(dx',dy).\\
\end{align*}

Since $h^N_s(x) = g^N_{\floor{Ns}}(x) = \ip{c\sigma(wx), N^{1-\gamma}\mu^N_s}$, by Assumption \ref{assumption}, Lemma \ref{lemma_bound}, and Lemma \ref{lemma_g}, we have  $\sup_{t\in[0,T]}\E\paren{|V_{t}^{N}|^2}\leq \frac{C}{N^2}$ for some $C<\infty$.
We also note that  $\frac{1}{N^{\gamma+1}} \sum_{k=0}^{\floor{Nt}-1} G^N_k = O(N^{-\gamma})$.  As in Lemma 5.6 of \cite{SirignanoSpiliopoulosNN4}, we can obtain the following uniform bound for the network.
\begin{lemma}\label{lemma_g}
For any $k\le \floor{TN}$ and any $x \in \mathcal{X}$,
\[\sup_{N \in \mathbb{N}, k\le \floor{TN}}\E \left(\abs{g^N_k(x)}^2\right) \le \sup_{N \in \mathbb{N}, k\le \floor{TN}} \sum_{x\in\mathcal{X}}\E \paren{\abs{g^N_k(x)}^2} < C,\]
for some constant $C<\infty$.
\end{lemma}

By using conditional independence of the terms in $M_t^{1,N}$ and $M_t^{2,N}$ and the results from Lemmas \ref{lemma_bound} and \ref{lemma_g}, we can also derive the following lemma. The proof for this lemma follows that of Lemma 3.1 in \cite{SirignanoSpiliopoulosNN1}, and thus is omitted here.
\begin{lemma}\label{LLN:lemma:Mt_bound}
For any $N \in \mathbb{N}$, $t \in [0,T]$ and $x\in\mathcal{X}$, there exists positive constant $C<\infty$ such that
\begin{align*}
&\E \left[\paren{M_t^{N,1}(x)}^2 +\paren{M_t^{N,2}(x)}^2\right] \le \frac{C}{N}.
\end{align*}
\end{lemma}

Next, we analyze the evolution of the empirical measure $\nu_k^N$ in terms of test functions $f\in C^2_b(\R^{1+d})$. First order Taylor expansion gives that
\begin{equation}\label{tfn_evolution}
\begin{aligned}
\ip{f, \nu^N_{k+1}} - \ip{f, \nu^N_{k}} &= \frac{1}{N}\sum_{i=1}^N \left[ f(C^i_{k+1},W^i_{k+1}) - f(C^i_{k},W^i_{k})\right]\\
&=\frac{1}{N} \sum_{i=1}^N \partial_c f(C^i_{k},W^i_{k}) \paren{C^i_{k+1} - C^i_k} + \frac{1}{N} \sum_{i=1}^N \nabla_w f(C^i_{k},W^i_{k}) \paren{W^i_{k+1} - W^i_k}\\
&\quad + \frac{1}{N} \sum_{i=1}^N \partial^2_c f(\tilde{C}^i_{k},\tilde{W}^i_{k}) \paren{C^i_{k+1} - C^i_k}^2 + \frac{1}{N} \sum_{i=1}^N \paren{W^i_{k+1} - W^i_k}^T\nabla^2_w f(\hat{C}^i_{k},\hat{W}^i_{k}) \paren{W^i_{k+1} - W^i_k}\\
&\quad +\frac{1}{N} \sum_{i=1}^N \paren{C^i_{k+1} - C^i_k}\nabla_{cw} f(\bar{C}^i_{k},\bar{W}^i_{k}) \paren{W^i_{k+1} - W^i_k},
\end{aligned}
\end{equation}
for some points $(\tilde{C}^i_{k},\tilde{W}^i_{k})$, $(\hat{C}^i_{k},\hat{W}^i_{k}) $ and $(\bar{C}^i_{k},\bar{W}^i_{k})$ in the line segments connecting $C^i_{k}$ and $C^i_{k+1}$ or $W^i_{k}$ and $W^i_{k+1}$, respectively. Similar to the analysis for the evolution of the network output, we substitute \eqref{SGD_iteration} into \eqref{tfn_evolution} and rewrite the evolution in terms of the empirical measure $\nu^N_k$,
\begin{equation*}
\begin{aligned}
\ip{f, \nu^N_{k+1}} - \ip{f, \nu^N_{k}} &= \frac{\alpha}{N^{\beta + \gamma}} \paren{y_k - g^N_k(x_k)}\ip{\partial_cf(c,w)\sigma(wx_k), \nu^N_k}\\
&\quad + \frac{\alpha}{N^{\beta + \gamma}} \paren{y_i - g^N_k(x_k)}\ip{c\sigma'(wx_k)\nabla_w f(c,w)x_k, \nu^N_k} +  O\paren{N^{-2(2\gamma+\beta-1)}}\\
&= \frac{\alpha}{N^{2- \gamma}} \paren{y_k - g^N_k(x_k)}\ip{\partial_cf(c,w)\sigma(wx_k), \nu^N_k}\\
&\quad + \frac{\alpha}{N^{2- \gamma}} \paren{y_k - g^N_k(x_k)}\ip{c\sigma'(wx_k)\nabla_w f(c,w)x_k, \nu^N_k} +  O\paren{N^{-2}},
\end{aligned}
\end{equation*}
where the last equality holds because $\gamma + \beta = 2-\gamma$. To write the evolution in terms of the scaled measure $\mu_t^N$, for $t \in [0,T]$, we have for large $N$
\begin{equation}\label{LLN:evolution_mu}
\begin{aligned}
\ip{f,\mu^N_t} - \ip{f,\mu^N_0} &= \sum_{k=0}^{\floor{Nt}-1}\paren{\ip{f, \nu^N_{k+1}} - \ip{f, \nu^N_{k}} }\\
&= \frac{\alpha}{N^{2- \gamma}} \sum_{k=0}^{\floor{Nt}-1} \paren{y_k - g^N_k(x_k)}\ip{\partial_cf(c,w)\sigma(wx_k), \nu^N_k}\\
&\quad + \frac{\alpha}{N^{2- \gamma}}\sum_{k=0}^{\floor{Nt}-1} \paren{y_k - g^N_k(x_k)}\ip{c\sigma'(wx_k)\nabla_w f(c,w)x_k, \nu^N_k} +  O\paren{N^{-1}}.
\end{aligned}
\end{equation}

By similar analysis as for the network, one can decompose the above equation into a drift and martingale component and rewrite the equation for $\ip{f,\mu^N_t}$ in terms of a Riemann integral and the scaled measure $\mu^N_t$.
\begin{equation}\label{LLN:evolution_mu_integral}
\begin{aligned}
\ip{f,\mu^N_t} - \ip{f,\mu^N_0}
&= \frac{\alpha}{N^{1- \gamma}} \int_0^t \int_{\CX\times \CY} \paren{y - h^N_s(x')}\ip{C^f_{x'}(c,w), \mu^N_s} \pi(dx',dy)ds\\
 &+  O_p\paren{N^{\gamma -\frac{3}{2}}}+  O\paren{N^{-1}}.
\end{aligned}
\end{equation}

\subsection{Proof of Theorem \ref{LLN:theorem}}
 We note now that as in Lemma 5.10 in \cite{SirignanoSpiliopoulosNN4} one obtains the relative compactness of the processes $\{\mu^N,h^N\}_{N\in\mathbb{N}}$. The proof is similar to that of Lemma 5.10 in \cite{SirignanoSpiliopoulosNN1}, and thus it is omitted.
\begin{lemma}\label{LLN:lemma:relative_compact}
The sequence of processes $\{\mu^N,h^N\}_{N\in\mathbb{N}}$ is relatively compact in $D_E([0,T])$.
\end{lemma}

It remains to identify the limit in order to prove Theorem \ref{LLN:theorem}. We denote $\pi^N \in \mathcal{M}\paren{D_E([0,T])}$ the probability measure corresponding to $(\mu^N,h^N)$. Relative compactness  implies that there is a subsequence $\pi^{N_k}$ that converges weakly. In the following lemma, we show that the limit point $\pi$ of any convergent subsequence $\pi^{N_k}$ will satisfy the evolution equation \eqref{LLN:limit_evolution}.
\begin{lemma}\label{LLN:lemma:id_limit}
The limit point $\pi$ of a convergent subsequence $\pi^{N_k}$ is a Dirac measure concentrated on $(\mu,h) \in D_E([0,T])$, which satisfies equation \eqref{LLN:limit_evolution}. Furthermore, for any $t \in [0,T]$ and $f \in C_b^2(\R^{1+d})$, $\ip{f,\mu_t} = \ip{f,\mu_0}$.
\end{lemma}
\begin{proof}For any $t \in [0,T]$, $f \in C^2_b(\R^{1+d})$, $g_1,\ldots,g_p \in C_b(\R^{1+d})$, $q_1, \ldots, q_p \in C_b(\R^{M})$, and $0 \le s_1 < \cdots < s_p \le t$, we define $F(\mu,h): D_E([0,T]) \to \R_+$ as
\begin{equation}\label{LNN:identify_limit_eq}
\begin{aligned}
F(\mu,h) &= \abs{\paren{\ip{f,\mu_t}-\ip{f,\mu_0}} \times \ip{g_1,\mu_{s_1}}\times \cdots \times \ip{g_p,\mu_{s_p}}}\\
& \quad + \sum_{x\in \mathcal{X}} \left\vert \left(h_t(x) - h_0(x) - \alpha \int_0^t  \int_{\CX \times \CY} \paren{y-h_s(x')}  \ip{\sigma(w x') \sigma(w x),\mu_s}  \pi(dx',dy) ds \right.\right.\\
& \qquad \left.\left.-\alpha\int_0^t \int_{\CX \times \CY} \paren{y-h_s(x')}x' x \ip{\sigma'(w x)  \sigma'(w x')c^2,\mu_s} \pi(dx',dy)ds \right) \times q_1(h_{s_1}) \times \cdots \times q_p(h_{s_p}) \right \vert.
\end{aligned}
\end{equation}

By equations \eqref{LLN:evolution_h}, \eqref{LLN:evolution_mu}, Lemma \ref{LLN:lemma:Mt_bound} and the Cauchy-Schwartz inequality, we have
\begin{align*}
\E_{\pi^N} \left[F(\mu,h)\right] &= \E\left[F(\mu^N,h^N)\right]\\
&= \E\left[\abs{O\paren{N^{-(1-\gamma)}} \times \prod_{i=1}^{p}\ip{g_i,\mu_{s_i}^N}}\right] + \E\left[\abs{ \paren{M_t^{N,1} + M_t^{N,2} + O\paren{N^{-\gamma}}}\times \prod_{i=1}^p q_i\paren{h^N_{s_i}}}\right]\\
&\le C \paren{\E\left[\abs{M_t^{N,1}}^2\right]^{\frac{1}{2}}+\E\left[\abs{M_t^{N,1}}^2\right]^{\frac{1}{2}}} + O\paren{N^{-(1-\gamma)}}\\
&\le C\paren{\frac{1}{N^{1/2}} + \frac{1}{N^{1-\gamma}}}.
\end{align*}

Therefore, $\lim_{N\to \infty} \E_{\pi^N} \left[F(\mu,h)\right] = 0$. Since $F(\cdot)$ is continuous and $F(\mu^N,h^N)$ is uniformly bounded, we have $\E_{\pi} \left[F(\mu,h)\right] = 0$. Hence, $(\mu,h)$ satisfies equation \eqref{LLN:limit_evolution} and $\ip{f,\mu_t} = \ip{f,\mu_0}$ for any $f \in C_b^2(\R^{1+d})$.
\end{proof}

We now finish the proof for Theorem \ref{LLN:theorem}.  Relative compactness proved in Lemma \ref{LLN:lemma:relative_compact} implies that every subsequence $\pi^{N_{k}}$ has a further sub-subsequence $\pi^{N_{k_m}}$ which weakly converges. By Lemma \ref{LLN:lemma:id_limit}, the limit point $\pi$ of any convergent subsequence $\pi^{N_{k_m}}$ satisfies the evolution stated in Theorem \ref{LLN:theorem}. Since equation \eqref{LLN:limit_evolution} is a finite-dimensional, linear equation, it has a unique solution. By Prokhorov's theorem, $\pi^N$ converges weakly to $\pi$, which is the distribution of $(\mu,h)$, the unique solution of \eqref{LLN:limit_evolution}. Hence, $(\mu^N,h^N)$ converges in distribution to $(\mu,h)$. Recall that $h^N_0(x) = g^N_0(x)\xrightarrow{d} 0$, which implies $h_0(x)=0$ for any $x\in \mathcal{X}$.

\section{Proof of Theorem \ref{CLT:theorem}: Convergence of the First Order Fluctuation Process}\label{sec::CLT}
Recall from Theorem \ref{LLN:theorem}, for any $t\in [0,T]$, the limit $\mu_t$ of the empirical measure $\mu^N_t$ satisfies the equation $\ip{f,\mu_t} = \ip{f,\mu_0}$, for any $f\in C_b^2(\R^{1+d})$. In addition, since $h_0(x) = 0$, \eqref{LLN:limit_evolution} can be reduced to (\ref{LLN:limit_evolution2}).

In this section, we prove a central limit theorem for the scaled one-layer neural network as the width of the network and the number of training steps become large. In particular, we study the convergence of the first order fluctuation process of the network's output
$K^N_t = N^{\varphi} \left(h^N_t - h_t\right)$,
where the parameter $\varphi$ is dependent on the scaling parameter $\gamma$. For our analysis in this section, we also denote $\eta^N_t = N^{\varphi}\paren{\mu^N_t - \mu_0}$.
\subsection{Evolution of the Fluctuation Process}\label{CLT:sec::evo}
 Using the evolution analysis of the network in Section \ref{LLN:sec:pre_evolution}, for $t \in [0,T]$ and $x\in\mathcal{X}$, the evolution of $K^N_t(x)$ can be written as
\begin{equation}
\label{CLT:network_fluc_evol}
\begin{aligned}
K^N_t(x)&= N^{\varphi} \left(h^N_t(x) - h_t(x)\right) = N^{\varphi} \left[\paren{h^N_t(x) - h_0^N(x)} - h_t(x) + h^N_0(x)\right]\\
&= \alpha\int^t_0 \int_{\CX \times \CY} \paren{y-h_s(x')} \ip{B_{x,x'}(c,w),\eta^N_s} \pi(dx',dy) ds -  \alpha\int^t_0 \int_{\CX \times \CY}  K^N_s(x') A_{x,x'} \pi(dx',dy) ds\\
&\quad + \Gamma^N_{t}(x) + K^N_0(x) + N^{\varphi}V_t^{N}(x) + N^{\varphi}M_t^{N}(x) + \frac{1}{N^{\gamma+1-\varphi}} \sum_{k=0}^{\floor{Nt}-1} G^N_{k}(x),
\end{aligned}
\end{equation}
where $K^N_0(x) = N^{\varphi}h^N_0(x) =  N^{1-\gamma + \varphi}\ip{c\sigma(wx),\mu^N_0}$ and
\begin{align*}
\Gamma^{N}_{t}(x) &= -\frac{\alpha}{N^{\varphi}}\int_0^t \int_{\CX \times \CY} K^N_s(x')\ip{B_{x,x'}(c,w),\eta^N_s}\pi(dx',dy)ds.
\end{align*}

We also consider the evolution of $l_t^N(f) = \ip{f,\eta^N_t}$ for a fixed function $f\in C^2_b(\R^{1+d})$. By the analysis of the evolution of the empirical measure $\mu^N_t$ in Section \ref{LLN:sec:pre_evolution}, we have for $N$ large enough
\begin{equation}\label{CLT:eta_evolution}
\begin{aligned}
\ip{f,\eta^N_t} - \ip{f,\eta^N_0}
&= \frac{\alpha}{N^{2- \gamma -\varphi}} \sum_{k=0}^{\floor{Nt}-1} \paren{y_k - g^N_k(x_k)}\ip{\partial_cf(c,w)\sigma(wx_k), \nu^N_k}\\
&\quad + \frac{\alpha}{N^{2- \gamma-\varphi}}\sum_{k=0}^{\floor{Nt}-1} \paren{y_k - g^N_k(x_k)}\ip{c\sigma'(wx_k)\nabla_w f(c,w)x_k, \nu^N_k} +  O\paren{N^{-(1-\varphi)}}.
\end{aligned}
\end{equation}

\subsection{Proof of Proposition \ref{prop::l_t}: Convergence of $\ip{f,\eta^N_t}$}\label{CLT:sec:eta_properties}
In this section, we establish the convergence of the process $\ip{f,\eta^N_t}$ as $N \to \infty$ in $D_{\R}([0,T])$ for any fixed $f\in C^2_b(\R^{1+d})$.

To show that relative compactness holds, we need to show that the compact containment condition and appropriate regularity for the process holds. Then, we will be able to conclude by Theorem 8.6 of Chapter 3 of \cite{EthierAndKurtz} that the relative compactness claim holds. The following lemma implies compact containment of the process $\{\ip{f,\eta^N_t}\}$.
\begin{lemma}\label{CLT:lemma:eta_compact_contatinment}
For any fixed $f\in C^2_b(\R^{1+d})$, when $\varphi \le 1-\gamma$, there exist a constant $C<\infty$, such that
\[\sup_{N \in \mathbb{N}, 0\le t\le T} \E \left[\abs{\ip{f,\eta^N_t}}^2 \right] < C.\]
Furthermore, for any $\epsilon>0$, there exist a compact subset $U \subset \R$ such that
$\sup_{N\in \mathbb{N}, 0\le t\le T} \P\paren{\ip{f,\eta^N_t} \notin U} < \epsilon. $
\end{lemma}
\begin{proof}
By equation \eqref{CLT:eta_evolution}, we have
\begin{align*}
\abs{\ip{f,\eta^N_t}}  &\le \abs{\ip{f,\eta^N_0}} + \frac{C}{N^{2- \gamma -\varphi}}\sum_{k=0}^{\floor{Nt}-1} \abs{y_k - g^N_k(x_k)} + \frac{C}{N^{1-\varphi}}\\
&\le  \abs{\ip{f,\eta^N_0}} + \frac{C}{N^{1- \gamma -\varphi}} \abs{g^N_k(x_k)} + \frac{C}{N^{1- \gamma -\varphi}}.
\end{align*}

Squaring both sides, we get
\begin{equation}\label{eta_compact}
\begin{aligned}
\abs{\ip{f,\eta^N_t}}^2  &\le 3 \paren{\abs{\ip{f,\eta^N_0}}^2 + \frac{C}{N^{2(1- \gamma -\varphi)}} \abs{g^N_k(x_k)}^2 + \frac{C}{N^{2(1- \gamma -\varphi})}}.
\end{aligned}
\end{equation}

Since $\ip{f,\eta^N_0} = N^{\varphi - \frac{1}{2}} \ip{f, \sqrt{N} \paren{\mu^N_0 - \mu_0}}$,
$\E\left[\abs{\ip{f,\eta^N_0}}^2\right] \le C(N^{2\varphi-1})$.
Taking expectation on both sides of equation \eqref{eta_compact}, by Lemma \ref{lemma_g}, we have
\[\sup_{N \in \mathbb{N}, 0\le t\le T} \E \left[\abs{\ip{f,\eta^N_t}}^2 \right] < C,\]
for some $C<\infty$. By Markov's inequality, the compact containment condition of $\ip{f,\eta^N_t}$ follows.
\end{proof}

We now establish the regularity of $\ip{f,\eta^N_t}$. For the following lemma, we define the function $q(z_1,z_2) = \min\{\abs{z_1 - z_2},1\}$, where $z_1,z_2 \in \R$.
\begin{lemma}\label{CLT:eta_regularity}
For $f\in C^2_b(\R^{1+d})$, $\delta \in (0,1)$, there exist a constant $C <\infty$ such that for any $0\le u \le \delta$, $0\le v\le \delta \wedge t$, and $t \in [0,T]$,
\[\E \left[q\paren{\ip{f,\eta^N_{t+u}},\ip{f,\eta^N_{t}}}q\paren{\ip{f,\eta^N_{t}},\ip{f,\eta^N_{t-v}}}\vert \CF^N_t\right]\le \frac{C\delta}{N^{1-\gamma-\varphi}}+\frac{C}{N^{2-\gamma-\varphi}},\]
where $\varphi \le  1-\gamma$.
\end{lemma}
\begin{proof}Recall that $\ip{f,\mu_t} = \ip{f,\mu_0}$ for any $t \in [0,T]$ and $f\in C^2_b(\R^{1+d})$. For any $0 \le s < t \le T$, we have
\[\abs{\ip{f,\eta^N_t} - \ip{f,\eta^N_s}} = \abs{\ip{f,N^{\varphi}\paren{\mu^N_t - \mu^N_s}}} = N^{\varphi} \abs{\paren{\ip{f,\mu^N_t} - \ip{f,\mu^N_s}}}.\]

By the regularity result for $\mu^N_t$ proved in Lemma 5.8 of \cite{SirignanoSpiliopoulosNN4}, we have
\begin{align*}
\E \left[\abs{\ip{f,\eta^N_t} - \ip{f,\eta^N_s}}\vert \CF^N_s\right] = N^{\varphi}\E \left[\abs{\ip{f,\mu^N_t} - \ip{f,\mu^N_s}}\vert \CF^N_s\right] \le \frac{C\delta}{N^{1-\gamma-\varphi}}+\frac{C}{N^{2-\gamma-\varphi}},
\end{align*}
for $0 <s <t \le T$ with $0 < t-s \le \delta <1$. If $\varphi \le 1-\gamma$, both terms in the last inequality above are bounded as $N$ grows. The statement of the lemma follows.
\end{proof}
Using Lemmas \ref{CLT:lemma:eta_compact_contatinment} and \ref{CLT:eta_regularity}, we are now ready to present the proof of Proposition \ref{prop::l_t}. We first show the convergence of $l^N_t(f)$ for the case when $\varphi < 1-\gamma$.
\begin{proof}[Proof of Proposition \ref{prop::l_t} Case 1: $\varphi < 1-\gamma$]\

For fixed $f\in C^2_b(\R^{1+d})$, when $\varphi \le 1-\gamma $, the family of processes $\{\ip{f,\eta^N_t}, t\in[0,T]\}_{N\in \mathbb{N}}$ is relatively compact in $D_{\R}([0,T])$ due to Lemmas \ref{CLT:lemma:eta_compact_contatinment}, \ref{CLT:eta_regularity}, and Theorem 8.6 of Chapter 3 of \cite{EthierAndKurtz}.

 For simplicity, denote $l^N_t = \ip{f,\eta^N_t}$.  Let $\pi^N \in \mathcal{M}\paren{D_{\R}([0,T]}$ be the probability measure corresponding to $l^N_t$. Relative compactness implies that there is a subsequence $\pi^{N_k}$ that converges weakly to a limit point $\pi$. We show that $\pi$ is a Dirac measure concentrated on $l = 0 \in D_{\R}([0,T])$ when $\varphi < 1-\gamma$. For $t \in [0,T]$, $g_1, \ldots, g_p \in C_b(\R)$, and $0 \le s_1< \cdots < s_p \le t$, define a map $F(l):D_{\R}([0,T]) \to \R_+$ as
\[F(l) = \abs{\paren{l_t - 0} \times g_1(l_{s_1})\times \cdots \times g_p(l_{s_p})}.\]
By equation \eqref{CLT:eta_evolution} and the fact that $\ip{f,\eta^N_0} = N^{\varphi - \frac{1}{2}} \ip{f, \sqrt{N} \paren{\mu^N_0 - \mu_0}} = O_p(N^{\varphi -\frac{1}{2}})$, we have
\begin{align*}
\E_{\pi^N} \left[F(l)\right] &= \E\left[ F(l^N) \right] \\
&= \E \left[ \abs{\paren{\ip{f,\eta^N_t} - \ip{f,\eta^N_0} + \ip{f,\eta^N_0}} \times \prod_{i=1}^p g_i(l^N_{s_i})}\right]\\
&\le \E \left[ \abs{\paren{\ip{f,\eta^N_t} - \ip{f,\eta^N_0}} \times \prod_{i=1}^p g_i(l^N_{s_i})}\right] + \E \left[ \abs{ \ip{f,\eta^N_0} \times \prod_{i=1}^p g_i(l^N_{s_i})}\right]\\
&\le C \paren{\frac{1}{N^{1-\gamma - \varphi}} + \frac{1}{N^{1-\varphi}} + \frac{1}{N^{\frac{1}{2}-\varphi}}}.
\end{align*}
Note that  $\gamma > \frac{1}{2}$ implies that $\varphi < 1-\gamma<\frac{1}{2}$ and $1-\gamma - \varphi >0$. Since $F(\cdot)$ is continuous and $F(l^N)$ is uniformly bounded, we have
\[\lim_{N\to \infty} \E_{\pi^N} \left[F(l) \right] = \E_{\pi} \left[F(l)\right] = 0,\]
where $\pi$ is the Dirac measure concentrated on 0.
We have shown that the limit point $\pi$ of any convergence subsequence (which exists due to relative compactness) is the Dirac measure concentrated on 0. Therefore, by Prokhorov's theorem, $\pi^N$ weakly converges to 0. That is, $l^N = \ip{f,\eta^N} \xrightarrow{d} 0$ and thus the limit is in probability. This concludes the proof of Proposition \ref{prop::l_t} in  case 1: $\varphi < 1-\gamma$.
\end{proof}

The proof of Proposition \ref{prop::l_t} in  case 2: $\varphi = 1-\gamma$ is more subtle and is given in different steps below. We see that equation \eqref{CLT:eta_evolution} becomes
\begin{equation}
\begin{aligned}
\ip{f,\eta^N_t} - \ip{f,\eta^N_0} &= \ip{f,N^{\varphi}\paren{\mu^N_t - \mu^N_0}} = N^{\varphi} \paren{\ip{f,\mu^N_t} - \ip{f,\mu^N_0}}\\
&= \frac{\alpha}{N}\sum_{k=0}^{\floor{Nt}-1} \int_{\CX \times \CY} \paren{y-g^N_k(x')}  \ip{ \partial_cf(c,w)\sigma(w x'),\nu^N_k} \pi(dx',dy)\\
&\quad +\frac{\alpha}{N}\sum_{k=0}^{\floor{Nt}-1}\int_{\CX \times \CY} \paren{y-g^N_k(x')}\ip{c\sigma'(wx')\nabla_w f(c,w)x',\nu^N_k} \pi(dx',dy)\\
&\quad + M^{1,N}_{f,t} + M^{2,N}_{f,t}+ O\paren{N^{-\gamma}},
\end{aligned}
\end{equation}
where
\begin{align*}
M^{1,N}_{f,t} &=  \frac{\alpha}{N}\left \lbrace \sum_{k=0}^{\floor{Nt}-1}\paren{y_k-g^N_k(x_k)}  \ip{\partial_cf(c,w)\sigma(wx_k), \nu^N_k}\right.\\
&\qquad \quad \left. - \int_{\CX \times \CY} \paren{y-g^N_k(x')}  \ip{\partial_cf(c,w) \sigma(w x') ,\nu^N_k} \pi(dx',dy) \right \rbrace, \\
M^{2,N}_{f,t} &=  \frac{\alpha}{N} \left \lbrace \sum_{k=0}^{\floor{Nt}-1} \paren{y_k-g^N_k(x_k)} \ip{c\sigma'(wx_k)\nabla_w f(c,w)x_k, \nu^N_k}\right.\\
&\qquad \qquad \left. -\int_{\CX \times \CY} \paren{y-g^N_k(x')} \ip{c\sigma'(wx')\nabla_w f(c,w)x',\nu^N_k} \pi(dx',dy)  \right \rbrace.
\end{align*}

As $N$ grows, we can rewrite this equation in terms of Riemann integrals and scaled measure $\mu^N_t$,
\begin{equation}\label{CLT:eta_evolution2}
\begin{aligned}
\ip{f,\eta^N_t} - \ip{f,\eta^N_0}
&= \int_0^t \int_{\CX \times \CY} \alpha \paren{y-h^N_s(x')}  \ip{ C_{x'}^{f}(c,w),\mu^N_s} \pi(dx',dy) ds\\
&\quad + M^{1,N}_{f,t} + M^{2,N}_{f,t}+ O\paren{N^{-\gamma}},
\end{aligned}
\end{equation}

Fir any fixed $f \in C^2_b(\R^{1+d})$, similar analysis as in Lemma 3.1 in \cite{SirignanoSpiliopoulosNN1}, we have the following bound for terms $ M^{1,N}_{f,t}, M^{2,N}_{f,t}$.
\begin{lemma}\label{CLT:lemma:martingale2}
For any $N \in \mathbb{N}$, there is a constant $C<\infty$ such that
\begin{align*}
\E\left[\sup_{t\in [0,T]} \abs{M^{1,N}_{f,t}}^2 +\sup_{t\in [0,T]} \abs{ M^{2,N}_{f,t}}^2 \right] \le \frac{C}{N}.
\end{align*}
\end{lemma}
Denote $l^N_t(f) = \ip{f,\eta^N_t}$. From equation \eqref{CLT:eta_evolution2}, we see that the evolution of $l^N_t(f)$ involves the evolution of $\mu^N_t$ and $h^N_t$. In the next lemma, we prove the convergence of the processes $(\mu^N_t, h^N_t,l^N_t(f))$ in distribution in the space $D_{E'}([0,T])$, where $E' = \CM(\R^{1+d}) \times \R^M \times \R$.

The result for Proposition \ref{prop::l_t} Case 2: $\varphi = 1-\gamma$ then follows from Lemma \ref{CLT:lemma:lt_1-gamma}.
\begin{lemma}\label{CLT:lemma:lt_1-gamma}
For any fixed $f\in C^2_b(\R^{1+d})$, if $\varphi = 1-\gamma$, the processes $(\mu^N_t, h^N_t,l^N_t(f))$ converges in distribution in $D_{E'}([0,T])$ to $(\mu_0,h_t,l_{t}(f))$, where $h_t$ satisfies equation \eqref{LLN:limit_evolution} and $l_t(f)$ is given by
\begin{equation}\label{CLT:evolution_l}
\begin{aligned}
l_{t}(f)
&= \int_0^t \int_{\CX \times \CY} \alpha \paren{y-h_s(x')}  \ip{ C^f_{x'}(c,w),\mu_0} \pi(dx',dy) ds.
\end{aligned}
\end{equation}
\end{lemma}
\begin{proof}By Lemmas \ref{CLT:lemma:eta_compact_contatinment} and \ref{CLT:eta_regularity}, $\{l^N(f)\}_{N\in\mathbb{N}}$ is relatively compact in $D_{\R}([0,T])$. By Lemma \ref{LLN:lemma:relative_compact}, $\{\mu^N,h^N\}_{N \in \mathbb{N}}$ is relatively compact in $D_E([0,T])$, where $E= \CM(\R^{1+d}) \times \R^M $. Since relative compactness is equivalent to tightness, we have that the probability measures of the family of processes $\{l^N(f)\}_{N\in\mathbb{N}}$ and the probability measures of the family of processes $\{\mu^N,h^N\}_{N \in \mathbb{N}}$ are tight. Therefore, $\{\mu^N,h^N,l^N(f)\}_{N \in \mathbb{N}}$ is tight. Hence, $\{\mu^N,h^N,l^N(f)\}_{N \in \mathbb{N}}$ is also relatively compact.

Denote $\pi^N \in \CM(D_{E'}([0,T])$ the probability measure corresponding to $(\mu^N,h^N,l^N(f))$. 
We now show that any limit point $\pi$ of a convergent subsequence (existing due to relative compactness) $\pi^{N_k}$ is a Dirac measure concentrated on $(\mu,h,l(f))\in D_{E'}([0,T])$, where $(\mu,h,l(f))$ satisfies equations \eqref{LLN:limit_evolution} and \eqref{CLT:evolution_l}. We define the map $F_1(\mu,h,l(f)):D_{E'}([0,T]) \to \R_+$ for each $t \in [0,T]$, $m_1,\ldots,m_p \in C_b(\R)$, and $0 \le s_1 < \cdots < s_p \le t$.
\begin{equation}\label{F_1}
\begin{aligned}
F_1(\mu,h,l(f)) &= F(\mu,h) + \left\vert \left(l_t(f) - \int_0^t \int_{\CX \times \CY} \alpha \paren{y-h_s(x')}  \ip{ C^f_{x'}(c,w)),\mu_s} \pi(dx',dy) ds \right)\right. \\
&\qquad \qquad \qquad \left. \times m_1(l_{s_1}(f)) \times \cdots \times  m_p(l_{s_p}(f))\right\vert,
\end{aligned}
\end{equation}
where $F(\mu,h)$ is as given in equation \eqref{LNN:identify_limit_eq}. Using equation \eqref{CLT:eta_evolution2}, Lemma \ref{CLT:lemma:martingale2}, the analysis of $F(\mu,h)$ in the proof of Lemma \ref{LLN:lemma:id_limit} and the fact that $\ip{f,\eta^N_0} = N^{\varphi - \frac{1}{2}} \ip{f, \sqrt{N} \paren{\mu^N_0 - \mu_0}} = O_p(N^{\varphi -\frac{1}{2}})$, we obtain
\begin{align*}
\E_{\pi^N} \left[F_1(\mu,h,l(f))\right] &= \E\left[F(\mu^N,h^N)\right] + \E\left[\abs{ \paren{\ip{f,\eta^N_0} + M_{f,t}^{N,1} + M_{f,t}^{N,2} + O\paren{N^{-\gamma}}}\times \prod_{i=1}^p m_i\paren{l^N_{s_i}(f)}}\right]\\
&\le C\paren{\frac{1}{N^{\frac{1}{2}}} + \frac{1}{N^{1-\gamma}}} +  C \paren{\E\left[\abs{M_{f,t}^{N,1}}^2\right]^{\frac{1}{2}}+\E\left[\abs{M_{f,t}^{N,1}}^2\right]^{\frac{1}{2}}} + C\paren{\frac{1}{N^{\frac{1}{2}-\varphi}}} \\
&\le C\paren{\frac{1}{N^{1-\gamma}}}.
\end{align*}

Therefore, $\lim_{N\to \infty} \E_{\pi^N}[F_1(\mu,h,l(f))] = 0$. Since $F(\cdot)$ is continuous and $F(\mu^N,h^N)$ is uniformly bounded, together with analysis in Section \ref{CLT:sec:eta_properties}, we have that $F_1(\cdot)$ is continuous and $F_1(\mu^N,h^N,l^N(f))$ is uniformly bounded. Hence,
\[\lim_{N\to \infty} \E_{\pi^N}\left[F_1(\mu,h,l(f))\right] = \E_{\pi}\left[F_1(\mu,h,l(f))\right] = 0.\]

We have shown that any limit point $\pi$ of a convergent subsequence must be a Dirac measure concentrated $(\mu,h,l(f))\in D_{E'}([0,T])$, where $(\mu,h,l(f))$ satisfies equations \eqref{LLN:limit_evolution}, \eqref{CLT:evolution_l} and $\mu_t = \mu_0$ weakly. Since the solution to equation \eqref{LLN:limit_evolution} is unique, by Prokhorov's theorem, the processes $(\mu^N_t, h^N_t,l^N_t(f))$ converges in distribution to $(\mu_0,h_t,l_{t}(f))$.
\end{proof}

\subsection{Bounds for Remainder Terms}\label{CLT:sec::remainder}

Since $h^N_s(x) = g^N_{\floor{Ns}}(x) = \ip{c\sigma(wx), N^{1-\gamma}\mu^N_s}$, by Assumption \ref{assumption}, Lemma \ref{lemma_bound}, Lemma \ref{lemma_g}, and our analysis in Section \ref{LLN:sec:pre_evolution}, $\frac{1}{N^{\gamma+1}} \sum_{k=0}^{\floor{Nt}-1} G^N_k = O(N^{-\gamma})$ and $\sup_{t\in[0,T]}\E\paren{|V_{t}^{N}|^2}\leq \frac{C}{N^2}$ for some $C<\infty$. It is easy to see that if $\varphi <\gamma$,
\begin{align*}
\frac{1}{N^{\gamma+1-\varphi}} \sum_{k=0}^{\floor{Nt}-1} G^N_k = O\paren{N^{-(\gamma-\varphi)}} \quad \text{and} \quad N^{\varphi}V_t^N = O_p\paren{N^{-(1-\varphi)}}.
\end{align*}

In addition, if $\varphi=\gamma - \frac{1}{2}$, by the classical central limit Theorem for i.i.d random variables,
\[K^N_0(x) = N^{\varphi}h^N_0(x) = N^{1-\gamma + \varphi}\ip{c\sigma(wx),\mu^N_0}= N^{1-\gamma + \varphi-\frac{1}{2}} \ip{c\sigma(wx), \sqrt{N} \mu^N_0} \xrightarrow{d} \mathcal{G}(x)\]
where $\mathcal{G}(x) \in \R$ is the Gaussian random variable defined in \eqref{limit_gaussian}. On the other hand, if $\varphi < \gamma - \frac{1}{2}$, then $N^{\varphi}h^N_0(x) \xrightarrow{d} 0$. Note that combining with our analysis in Lemma \ref{CLT:eta_regularity}, we have $\varphi \le \min \{1-\gamma, \gamma - \frac{1}{2}\}$. If $\gamma < \frac{3}{4}$, we can take $\varphi = \gamma -\frac{1}{2} < 1-\gamma$ in order to obtain a limiting Gaussian process for $K^N_t$. If $\gamma \ge \frac{3}{4}$, the limiting process for $K^N_t$ is Gaussian only if $\gamma = \frac{3}{4}$ and $\varphi = 1-\gamma = \gamma - \frac{1}{2}$.

We finish this section by proving the following lemma for the terms $N^{\varphi}M^{1,N}_t$ and $N^{\varphi}M^{2,N}_t$.
\begin{lemma}\label{CLT:lemma:martingale_bound}
For any $N \in \mathbb{N}$ and $x\in\mathcal{X}$, there is a constant $C<\infty$ such that
\begin{align*}
\E\left[\sup_{t\in [0,T]} \abs{N^{\varphi} M^{1,N}_t(x)}^2 + \sup_{t\in [0,T]} \abs{N^{\varphi} M^{2,N}_t(x)}^2 \right] \le \frac{C}{N^{1-2\varphi}},
\end{align*}
\end{lemma}
\begin{proof}Let $\mathfrak{F}_t$ be the $\sigma$-algebra generated by $\mu^N_s$, $M^{1,N}_s$ and $M^{2,N}_s$ for $s\le t$. Since for any $t > r$, we have
\begin{align*}
& \E \left[N^{\varphi} \paren{M^{1,N}_t(x) - M^{1,N}_r(x)} \vert \mathfrak{F}_r\right]+\E \left[N^{\varphi} \paren{M^{2,N}_t(x) - M^{2,N}_r(x)} \vert \mathfrak{F}_r\right]\\
&\quad=\frac{\alpha}{N^{1-\varphi}} \sum_{k=\floor{Nr}}^{\floor{Nt}-1}\E \left[\paren{y_k-g^N_k(x_k)}  \ip{\sigma(w x_k) \sigma(w x),\nu^N_k} \right. \\
&\qquad\qquad \qquad \qquad \left.- \int_{\CX \times \CY} \paren{y-g^N_k(x')}  \ip{\sigma(w x') \sigma(w x),\nu^N_k} \pi(dx',dy)\vert \mathcal{F}^N_r\right]\\
&\quad + \frac{\alpha}{N^{1-\varphi}} \sum_{k=\floor{Nr}}^{\floor{Nt}-1}\E \left[\paren{y_k-g^N_k(x_k)}  \ip{c^2\sigma'(w x)  \sigma'(w x_k)x_k x ,\nu^N_k} \right. \\
&\qquad\qquad \qquad \qquad \quad \left.- \int_{\CX \times \CY} \paren{y-g^N_k(x')}   \ip{c^2\sigma'(w x)  \sigma'(w x')x' x,\nu^N_k} \pi(dx',dy)\vert \mathcal{F}^N_r\right]\\
&\quad= \frac{\alpha}{N^{1-\varphi}} \cdot 0 = 0.
\end{align*}

Therefore, for $i=1,2$, we have
\[\E \left[N^{\varphi} M^{i,N}_t(x)  \vert \mathfrak{F}_r\right] = \E \left[N^{\varphi} \paren{M^{i,N}_t(x) - M^{i,N}_r(x)} \vert \mathfrak{F}_r\right]+\E \left[N^{\varphi} M^{i,N}_r(x)  \vert \mathfrak{F}_r\right] = 0 + N^{\varphi}M^{i,N}_r(x),\]
proving the martingale property for the process $N^{\varphi} M^{i,N}_t(x)$ with $i=1,2$ and $x\in\mathcal{X}$. 
Hence, by Lemma \ref{LLN:lemma:Mt_bound} and Doob's martingale inequality, we have for $i=1,2$
\[\E\left[\sup_{t\in [0,T]} \abs{N^{\varphi} M^{i,N}_t(x)}^2 \right] \le CN^{2\varphi} \E\left[ \abs{ M^{i,N}_T(x)}^2 \right] \le \frac{C}{N^{1-2\varphi}},\]
where the constant $C<\infty$. 
Note that since $\gamma <1$ and $\varphi \le \gamma-\frac{1}{2}$, we have $1-2\varphi >0$.
\end{proof}

\subsection{Relative Compactness of $\{K^N_t, t\in[0,T]\}_{N \in \mathbb{N}}$}\label{CLT:sec::relative_compatness}
In this section, we prove relative compactness of the family $\{K^N_t, t\in[0,T]\}_{N \in \mathbb{N}}$ in $D_{\R^M} ([0,T])$. As discussed earlier, it is sufficient to show compact containment and regularity of $K^N_t$. We first show compact containment in the next lemma.
\begin{lemma}\label{CLT:lemma:bound of ex_Kt}
There exist a constant $C<\infty$, such that for each $x\in\mathcal{X}$,
\[\sup_{N \in \mathbb{N}, 0\le t\le T} \E \left[\abs{K^N_t(x)}^2 \right] < C.\]
In particular, for any $\epsilon >0$, there exist a compact subset $U \subset \R^M$ such that
$\sup_{N\in \mathbb{N}, 0\le t\le T} \P\paren{K^N_t \notin U} < \epsilon.$
\end{lemma}
\begin{proof}Recall from \eqref{CLT:network_fluc_evol}, we have
\begin{align*}
\abs{K^N_t(x)}^2
&\le C \brac{(I)^2 + (II)^2 + \abs{\Gamma^N_t}^2 + \abs{N^{1-\gamma + \varphi}\ip{c\sigma(wx),\mu^N_0}}^2  + \abs{N^{\varphi}M_t^{N}}^2 + \abs{N^{\varphi} V^N_t}^2 + O \paren{N^{-2(\gamma-\varphi)}}},
\end{align*}
where
\begin{align*}
(I) &= \int^t_0 \int_{\CX \times \CY} \alpha\abs{y-h_s(x')} \abs{\ip{B_{x,x'}(c,w),\eta^N_s}} \pi(dx',dy) ds,\\
(II) &=  \int^t_0 \int_{\CX \times \CY} \alpha \abs{K^N_s(x')} \abs{A_{x,x'}} \pi(dx',dy) ds.
\end{align*}

By the Cauchy-Schwarz inequality, Theorem \ref{LLN:theorem}, Assumption \ref{assumption}, and Lemma \ref{lemma_bound}, we have
\begin{align*}
\abs{h_t(x)}^2 &\le C \left[ \paren{ \int^t_0 \int_{\CX \times \CY} \abs{y} \abs{A_{x,x'}} \pi(dx',dy) ds}^2 + \paren{\int^t_0 \int_{\CX \times \CY} \abs{h_s(x')} \abs{A_{x,x'}} \pi(dx',dy) ds}^2 \right],\\
&\le C_1 t^2 + C_2t \int^t_0 \int_{\CX \times \CY} \abs{h_s(x')}^2  \pi(dx',dy) ds,
\end{align*}
which implies that,
\[\sup_{t \in [0,T]}\int_{\CX \times \CY} \abs{h_t(x)}^2  \pi(dx,dy) \le C_1 T^2 + C_2 T \int^t_0 \int_{\CX \times \CY} \abs{h_s(x')}^2  \pi(dx',dy) ds.\]

By Gr\"onwall's inequality,
\begin{equation}\label{CLT:network_bound}
\sup_{0\le t\le T}\int_{\CX \times \CY} \abs{h_t(x)}^2  \pi(dx,dy) \le  \sup_{0\le t\le T} C_1T^2 \exp(C_2Tt) < C(T),
\end{equation}
for some constant $C(T)<\infty$ depending on $T$.
 Since $\sigma \in C^3_b(\R)$, for $f(w,c) = B_{x,x'}(c,w)$, by Lemma \ref{CLT:lemma:eta_compact_contatinment}, we have
\begin{equation}\label{CLT:network_eta_bound}
\E\left[\abs{\ip{\sigma(wx')\sigma(wx),\eta^N_t}}^2\right] < C, \quad \E\left[\abs{\ip{c^2\sigma'(wx')\sigma'(wx)xx',\eta^N_t}}^2\right]<C,
\end{equation}
for $t \in [0,T]$, and some constant $C < \infty$. By the Cauchy-Schwarz inequality, equations \eqref{CLT:network_bound}, \eqref{CLT:network_eta_bound}, and Assumption \ref{assumption}, we have
\begin{align*}
\E\left[(I)\right] &\le Ct^2,\quad (II) \le C_3t  \int^t_0 \int_{\CX \times \CY} \abs{K^N_s(x')}^2  \pi(dx',dy) ds.
\end{align*}
By the definition of $\Gamma^N_t$ in Section \ref{CLT:sec::evo} and Assumption \ref{assumption},
\begin{align*}
\abs{\Gamma^N_t}^2 & \le C_4\int^t_0 \int_{\CX \times \CY} \abs{K^N_s(x')}^2 \pi(dx',dy) ds \int^t_0 \int_{\CX \times \CY} \abs{\ip{B_{x,x'}(c,w),\mu^N_s-\mu_0}}^2 \pi(dx',dy) ds\\
& \le C_5t\int^t_0 \int_{\CX \times \CY} \abs{K^N_s(x')}^2 \pi(dx',dy) ds.
\end{align*}
Therefore, by the definition of $\pi(dx,dy)$, we have
\begin{align*}
\E\left(\abs{K^N_t(x)}^2 \right)
&\le  Ct^2  + \frac{C_3+C_5}{M} t \int^t_0 \sum_{x'\in \mathcal{X}} \E \paren{\abs{K^N_s(x')}^2 } ds+ \E\paren{\abs{N^{1-\gamma + \varphi}\ip{c\sigma(wx),\mu^N_0}}^2}
 \\
&\quad   + \E\paren{\abs{N^{\varphi}M_t^{N}(x)}^2} + O\paren{N^{-2(1-\varphi)}}+ O \paren{N^{-2(\gamma-\varphi)}} .
\end{align*}
Summing both side of the above inequality over all $x \in \mathcal{X}$, where $\mathcal{X}$ is a fixed data set of size $M$ gives 
\begin{equation}\label{sum_KN}
\begin{aligned}
\sum_{x\in\mathcal{X}} \E\left(\abs{K^N_t(x)}^2 \right)
&\le  CMT^2  + (C_3+C_5)T \int^t_0 \sum_{x'\in \mathcal{X}} \E \paren{\abs{K^N_s(x')}^2 }   ds+ \sum_{x\in \mathcal{X}}\E\paren{\abs{N^{1-\gamma + \varphi}\ip{c\sigma(wx),\mu^N_0}}^2} \\
&\quad  + \sum_{x\in \mathcal{X}}\E\paren{\abs{N^{\varphi}M_t^{N}(x)}^2} + O\paren{N^{-2(1-\varphi)}}+ O \paren{N^{-2(\gamma-\varphi)}}.
\end{aligned}
\end{equation}
Since for $\varphi \le \gamma - \frac{1}{2}$, $2(\gamma -\varphi) \ge 1$, we have
\begin{align*}
\E\paren{\abs{N^{1-\gamma + \varphi}\ip{c\sigma(wx),\mu^N_0}}^2} &\le  \E \left[ \abs{\frac{1}{N^{\gamma-\varphi}}\sum_{i=1}^N C^i_0 \sigma(W^i_0x)}^2\right] \le \frac{C}{N^{2(\gamma-\varphi)}} \sum_{i=1}^N \E \paren{\abs{C^i_0}^2}\le C.
\end{align*}

Then, by applying Gr\"onwall's inequality to equation \eqref{sum_KN} and using Lemma \ref{CLT:lemma:martingale_bound} we have
\begin{align*}
\sum_{x\in\mathcal{X}} \E\left(\abs{K^N_t(x)}^2 \right) \le C(M) T^2 \exp\left[\tilde{C} Tt\right],
\end{align*}
where $C(M), \tilde{C}$ are some finite constants. Hence,
for any $x\in \mathcal{X}$, there exist $C<\infty$ such that
\[\sup_{N \in \mathbb{N}, 0\le t\le T} \E \left[\abs{K^N_t(x)}^2 \right] <  C(M) T^2\exp\left[\tilde{C} T^2\right] \le C.\]

By Markov's inequality,  the compact containment condition for $K^N_t$ follows, concluding the proof of the lemma.
\end{proof}

We next establish the regularity of the process $K^N_t$ in $D_{\R^M}([0,T])$. For the purpose of this lemma, we denote $q(z_1,z_2) = \min\{\norm{ z_1-z_2}_{l^1},1\}$ for $z_1,z_2 \in \R^M$.
\begin{lemma}\label{CLT:lemma:regularity}
For any $\delta \in (0,1)$, there is a constant $C<\infty$ such that for $0\le u \le \delta$, $0\le v\le \delta \wedge t$, and $t \in [0,T]$,
\[\E \left[q\paren{K^N_{t+u},K^N_{t}}q\paren{K^N_{t},K^N_{t-v}}\vert \CF^N_t\right]\le {C\delta}+\frac{C}{N^{1-\varphi}}.\]
\end{lemma}
\begin{proof}For $0\le s<t\le T$, by equation (\ref{CLT:network_fluc_evol}), we have
\begin{align*}
&\abs{K^N_t(x) - K^N_s(x)}\le \int^t_s \int_{\CX \times \CY} \alpha\abs{y-h_\tau(x')} \abs{\ip{B_{x,x'}(c,w),\eta^N_\tau}} \pi(dx',dy) d\tau\\
&\quad +  \int^t_s \int_{\CX \times \CY} \alpha \abs{K^N_\tau(x')} \abs{\ip{B_{x,x'}(c,w),\mu^N_{\tau}}} \pi(dx',dy) d\tau\\
&\quad  +\abs{\Gamma^N_t(x) - \Gamma^N_s(x)}+ N^{\varphi}\abs{V_t^N(x)-V_s^N(x)} + N^{\varphi}\abs{M_t^{N}(x)-M_s^N(x)} + \frac{1}{N^{\gamma+1-\varphi}} \sum_{k=\floor{Ns}}^{\floor{Nt}-1}\abs{ G^N_k(x)}.
\end{align*}

Taking expectation on both sides of the above inequality, by Assumption \ref{assumption}, Lemma \ref{lemma_bound}, and analysis in Lemmas \ref{CLT:lemma:martingale_bound} and \ref{CLT:lemma:bound of ex_Kt}, we have for $0\leq t-s\leq \delta<1$
\begin{align*}
\E\left[\abs{K^N_t(x) - K^N_s(x)}\big\vert \CF^N_s \right]&\le C(t-s)  +  C_1 \int^t_s \int_{\CX \times \CY} \E\left[\abs{K^N_\tau(x')}\big\vert \CF^N_s\right]  \pi(dx',dy) d\tau\\
&\quad + N^{\varphi}\E\left[\abs{V_t^N(x)-V_s^N(x)}\big\vert \CF^N_s \right]+ C\E\left[\abs{N^{\varphi} \paren{M_t^N(x)-M_s^N(x)}}^2 \big\vert \CF^N_s \right]^{\frac{1}{2}}\\
&\quad   + \frac{1}{N^{\gamma+1-\varphi}} \sum_{k=\floor{Ns}}^{\floor{Nt}-1}\E\left[\abs{ G^N_k(x)}\big\vert \CF^N_s\right]\\
&\le C\delta  + \frac{C}{N^{1-\varphi}}.
\end{align*}

Note that
\[\E\left[\abs{N^{\varphi} \paren{M_t^N(x)-M_s^N(x)}}^2  \big \vert \CF^N_s \right] \le \frac{C \delta}{N^{1-2\varphi}} + \frac{C}{N^{2-2\varphi}},\]
following an analysis similar to Lemma 3.1 of \cite{SirignanoSpiliopoulosNN1}.
Since $x\in \mathcal{X}$ is arbitrary, we get the statement of the lemma.

\end{proof}
Combining Lemmas \ref{CLT:lemma:bound of ex_Kt} and \ref{CLT:lemma:regularity}, we have the following lemma for the relative compactness of the processes $\{K^N_t,t\in[0,T]\}_{N\in\mathbb{N}}$. The result then follows from Theorem 8.6 of Chapter 3 of \cite{EthierAndKurtz}.
\begin{lemma}\label{CLT:lemma:relative_compact}
The sequence of processes $\{K^N_t,t\in[0,T]\}_{N\in\mathbb{N}}$ is relatively compact in $D_{\R^M}([0,T])$.
\end{lemma}

\subsection{Proof of Convergence for $\{K^N_t, t\in[0,T]\}_{N \in \mathbb{N}}$}\label{CLT:sec:pf_conv_K}
In this section, we show that the processes $(\mu^N_t, h^N_t, l^N_t(B_{x,x'}(c,w)), K^N_t)$ converges in distribution in $D_{E_1}([0,T])$ to $(\mu_0, h_t, l_t(B_{x,x'}(c,w)), K_t)$, where $E_1 = \CM(\R^{1+d}) \times \R^M \times \R \times \R^M$, and $l_t(\cdot), K_t$ are as given in Proposition \ref{prop::l_t} and Theorem \ref{CLT:theorem}. For simplification, we denote $l_{B,t} = l_t(B_{x,x'}(c,w))$ and $l^N_{B,t} = l^N_t(B_{x,x'}(c,w))=\ip{B_{x,x'}(c,w),\eta^N_t}$ in this section.

By Lemmas \ref{LLN:lemma:relative_compact}, \ref{CLT:lemma:eta_compact_contatinment} and \ref{CLT:eta_regularity}, $\{\mu^N, h^N,l^N_B\}_{N\in\mathbb{N}}$ is relatively compact in $D_{E'}([0,T])$, where $E'=\CM(\R^{1+d})\times \R^M \times \R$. By Lemma \ref{CLT:lemma:relative_compact}, $\{K^N\}_{N\in\mathbb{N}}$ is relatively compact in $D_{\R^M}([0,T])$. Since relative compactness is equivalent to tightness, we have that the probability measures of the family of processes $\{\mu^N,h^N,l^N_B\}_{N\in\mathbb{N}}$ and the probability measures of the family of processes $\{K^N\}_{N \in \mathbb{N}}$ are tight. Therefore, $\{\mu^N,h^N,l^N_B,K^N\}_{N \in \mathbb{N}}$ is tight. Hence, $\{\mu^N,h^N,l^N_B,K^N\}_{N \in \mathbb{N}}$ is also relatively compact.

Denote $\pi^N \in \CM(D_{E_1}([0,T])$ the probability measure corresponding to $(\mu^N, h^N, l^N_B, K^N)$. 
We now show that any limit point $\pi$ of a convergence subsequence $\pi^{N_k}$ is a Dirac measure concentrated on $(\mu, h, l_B, K)\in D_{E_1}([0,T])$, where $(\mu,h)$ satisfies equation \eqref{LLN:limit_evolution} and $(l_B, K)$ satisfies equations given in in Proposition \ref{prop::l_t} and Theorem \ref{CLT:theorem} for different values of $\gamma$ and $\varphi$.

\begin{custlist}[Case]
\item When $\gamma \in \paren{\frac{1}{2}, \frac{3}{4}}$ and $\varphi \le \gamma - \frac{1}{2}$, or when $\gamma \in \left[\frac{3}{4}, 1\right)$ and $\varphi < 1-\gamma \le \gamma - \frac{1}{2}$, for any $t \in [0,T]$, $m_1,\ldots,m_p \in C_b(\R)$, $z_1, \ldots, z_p \in C_b(\R^{M})$, and $0 \le s_1 < \cdots < s_p \le t$, we define $F_2(\mu, h, l_B, K): D_{E_1}([0,T]) \to \R_+$ as
\begin{align*}
F_2(\mu, h, l_B, K) &= F(\mu,h) + \abs{\paren{l_{B,t} - 0} \times m_1(l_{B,s_1})\times \cdots \times m_p(l_{B,s_p})}\\
& \quad + \sum_{x\in \mathcal{X}} \left\vert \left(K_t(x) - K_0(x) - \alpha \int_0^t  \int_{\CX \times \CY} \paren{y-h_s(x')}  l_{B,s}  \pi(dx',dy) ds \right.\right.\\
& \qquad \left.\left.+\alpha\int_0^t \int_{\CX \times \CY} K_s(x') \ip{B_{x,x'}(c,w),\mu_0}\pi(dx',dy)ds \right) \times z_1(K_{s_1}) \times \cdots \times z_p(K_{s_p}) \right \vert,
\end{align*}
where $F(\mu,h)$ is as given in equation \eqref{LNN:identify_limit_eq}. We now note that for any $x \in \mathcal{X}$, by equation \eqref{CLT:network_fluc_evol},
\begin{align*}
&K^N_t(x) - K^N_0(x) - \alpha \int_0^t  \int_{\CX \times \CY} \paren{y-h^N_s(x')}  l^N_{B,s}  \pi(dx',dy) ds+\alpha\int_0^t \int_{\CX \times \CY} K^N_s(x') \ip{B_{x,x'}(c,w),\mu_0^N} \pi(dx',dy)ds\\
&= \frac{\alpha}{N^{\varphi}} \int_0^t  \int_{\CX \times \CY} K^N_s(x') \ip{B_{x,x'}(c,w),\eta^N_0} \pi(dx',dy) ds + \Gamma^{N}_{t}(x)+N^{\varphi}V_t^N(x) + N^{\varphi}M_t^{N}(x) + \frac{1}{N^{\gamma+1-\varphi}} \sum_{k=0}^{\floor{Nt}-1} G^N_k (x),
\end{align*}
and by the Cauchy-Schwarz inequality, Lemmas \ref{CLT:lemma:eta_compact_contatinment} and \ref{CLT:lemma:bound of ex_Kt}, for any $t \in [0,T]$,
\begin{equation}\label{eq:K^N:temp}
\begin{aligned}
& \E \paren{\abs{\frac{\alpha}{N^{\varphi}} \int_0^t  \int_{\CX \times \CY} K^N_s(x') \ip{B_{x,x'}(c,w),\eta^N_0} \pi(dx',dy) ds}}\\
&\le \frac{C}{N^{\varphi}} \int_0^t  \int_{\CX \times \CY}\E \paren{\abs{ K^N_s(x')}\abs{ \ip{B_{x,x'}(c,w),\eta^N_0}}} \pi(dx',dy) ds\\
&\le \frac{C}{N^{\varphi}} \int_0^t  \int_{\CX \times \CY}\E \paren{\abs{ K^N_s(x')}^2 }^{\frac{1}{2}} \E \paren{\abs{ \ip{B_{x,x'}(c,w),\eta^N_0}}^2}^{\frac{1}{2}} \pi(dx',dy) ds\\
& \le \frac{C(T)}{N^{\varphi}},
\end{aligned}
\end{equation}
where $C(T)<\infty$ is some finite constant depending on $T$.
By equation \eqref{eq:K^N:temp}, the analysis in the proofs of Lemma \ref{LLN:lemma:id_limit} and Proposition \ref{prop::l_t} Case 1, and Section \ref{CLT:sec::remainder}, we have
\begin{equation*}
\begin{aligned}
\E_{\pi^N} \left[ F_2(\mu, h, l_B, K)\right]& = \E_{\pi^N} \left[ F(\mu, h)\right] + \E \left[ \abs{\paren{\ip{f,\eta^N_t} - \ip{f,\eta^N_0} + \ip{f,\eta^N_0}} \times \prod_{i=1}^p m_i(l^N_{B,s_i})}\right]\\
& \quad + \sum_{x\in \mathcal{X}} \E \left\lbrace \left\vert \left(K^N_t(x) - K^N_0(x) - \alpha \int_0^t  \int_{\CX \times \CY} \paren{y-h^N_s(x')}  l^N_{B,s}  \pi(dx',dy) ds \right.\right.\right.\\
& \qquad \left.\left.\left.+\alpha\int_0^t \int_{\CX \times \CY} K^N_s(x')\ip{B_{x,x'}(c,w),\mu^N_0} \pi(dx',dy)ds \right) \times  \prod_{i=1}^p z_i(K^N_{s_i})  \right \vert\right\rbrace\\
&\le C\paren{\frac{1}{N^{\frac{1}{2}}} + \frac{1}{N^{1-\gamma}}} + C \paren{\frac{1}{N^{1-\gamma - \varphi}} + \frac{1}{N^{1-\varphi}} + \frac{1}{N^{\frac{1}{2}-\varphi}}} \\
&\quad + C\E\left[\abs{N^{\varphi}M_t^{N}}^2\right]^{\frac{1}{2}} +C\paren{\frac{1}{N^{1-\varphi}}+ \frac{1}{N^{\gamma-\varphi}} + \frac{1}{N^{\varphi}}}\\
&\le C \paren{\frac{1}{N^{1-\gamma - \varphi}} + \frac{1}{N^{\frac{1}{2}-\varphi}} +\frac{1}{N^{\varphi}}}.
\end{aligned}
\end{equation*}

Therefore, $\lim_{N\to \infty} \E_{\pi^N}[F_2(\mu,h,l_B,K)] = 0$. Since $F(\cdot)$ is continuous and $F(\mu^N,h^N)$ is uniformly bounded, together with analysis in Sections \ref{CLT:sec:eta_properties} and \ref{CLT:sec::relative_compatness}, we have that $F_2(\cdot)$ is continuous and $F_2(\mu^N,h^N,l^N_B,K^N)$ is uniformly bounded. Hence, by weak convergence we have
\[\lim_{N\to \infty} \E_{\pi^N}\left[F_2(\mu,h,l_B, K)\right] = \E_{\pi}\left[F_2(\mu,h,l_B,K)\right] = 0.\]

We have shown that any limit point $\pi$ of a convergence sequence must be a Dirac measure concentrated $(\mu,h,l_B,K)\in D_{E_1}([0,T])$, where $(\mu,h,l_B,K)$ satisfies equations \eqref{LLN:limit_evolution}, $l_{B,t} = 0$, and \eqref{CLT:evolution}. Since the solutions to equations \eqref{LLN:limit_evolution} and \eqref{CLT:evolution} are unique, by Prokhorov's theorem, the processes $(\mu^N_t, h^N_t,l^N_{B,t},K^N_t)$ converges in distribution to $(\mu_0,h_t,0,K_t)$.

\item When $\gamma \in \left[\frac{3}{4}, 1\right)$ and $\varphi = 1-\gamma$, for any $t \in [0,T]$, $z_1, \ldots, z_p \in C_b(\R^{M})$, and $0 \le s_1 < \cdots < s_p \le t$, we define $F_3(\mu, h, l_B, K): D_{E_1}([0,T]) \to \R_+$ as
\begin{equation}\label{F_3}
\begin{aligned}
F_3(\mu, h, l_B, K) &= F_1(\mu,h,l_B) + \sum_{x\in \mathcal{X}} \left\vert \left(K_t(x) - K_0(x) - \alpha \int_0^t  \int_{\CX \times \CY} \paren{y-h_s(x')}  l_{B,s}  \pi(dx',dy) ds \right.\right.\\
& \qquad \left.\left.+\alpha\int_0^t \int_{\CX \times \CY} K_s(x') \ip{B_{x,x'}(c,w),\mu_0} \pi(dx',dy)ds \right) \times z_1(K_{s_1}) \times \cdots \times z_p(K_{s_p}) \right \vert,
\end{aligned}
\end{equation}
where $F_1(\mu,h,l_B)$ is as given in equation \eqref{F_1}. By equation \eqref{eq:K^N:temp}, and the analysis in the proof of Lemma \ref{CLT:lemma:lt_1-gamma} and Section \ref{CLT:sec::remainder}, we obtain
\begin{align*}
\E_{\pi^N}\left[F_3(\mu, h, l_B, K)\right]&=\E_{\pi^N}\left[F_1(\mu,h,l_B)\right]\\
&\quad + \sum_{x\in \mathcal{X}} \E \left\lbrace \left\vert \left(K^N_t(x) - K^N_0(x) - \alpha \int_0^t  \int_{\CX \times \CY} \paren{y-h^N_s(x')}  l^N_{B,s}  \pi(dx',dy) ds \right.\right.\right.\\
& \qquad \left.\left.\left.+\alpha\int_0^t \int_{\CX \times \CY} K^N_s(x') A_{x,x'} \pi(dx',dy)ds \right) \times  \prod_{i=1}^p z_i(K^N_{s_i})  \right \vert\right\rbrace\\
&\le C\paren{\frac{1}{N^{1-\gamma}}+ \frac{1}{N^{\frac{1}{2}-\varphi}}+\frac{1}{N^{1-\varphi}}+ \frac{1}{N^{\gamma-\varphi}}+ \frac{1}{N^{\varphi}}}.
\end{align*}

Therefore, $\lim_{N\to \infty} \E_{\pi^N}[F_3(\mu,h,l_B,K)] = 0$. Since $F_1(\cdot)$ is continuous and $F_1(\mu^N,h^N,l^N_B)$ is uniformly bounded, together with analysis in Sections \ref{CLT:sec:eta_properties} and \ref{CLT:sec::relative_compatness}, we have that $F_3(\cdot)$ is continuous and $F_3(\mu^N,h^N,l^N_B,K^N)$ is uniformly bounded. Hence,
\[\lim_{N\to \infty} \E_{\pi^N}\left[F_3(\mu,h,l_B, K)\right] = \E_{\pi}\left[F_3(\mu,h,l_B,K)\right] = 0.\]

We have shown that any limit point $\pi$ of a convergence sequence must be a Dirac measure concentrated $(\mu,h,l_B,K)\in D_{E_1}([0,T])$, where $(\mu,h,l_B,K)$ satisfies equation \eqref{LLN:limit_evolution},\eqref{l_t limit}, and \eqref{K_t for 1-gamma}. Since the solutions to equations \eqref{LLN:limit_evolution} and \eqref{K_t for 1-gamma} are unique, by Prokhorov's theorem, the processes $(\mu^N_t, h^N_t,l^N_{B,t},K^N_t)$ converges in distribution to $(\mu_0,h_t,l_{B,t},K_t)$.
\end{custlist}

\subsection{Proof of Theorem \ref{R:BiasToZero1}}\label{sec::Kt_to_0}
Let $\tilde{h}_t = h_t - \hat{Y}$, where $\hat{Y} = \paren{y^{(1)},\ldots, y^{(M)}}$. By Theorem \ref{LLN:theorem}, we have that for any $\gamma \in \paren{\frac{1}{2},1}$
\begin{align*}
d \tilde{h}_t &= - A \tilde{h}_t dt, \text{ and }\tilde{h}_0 = -\hat{Y}.
\end{align*}

Since $A \in \R^{M\times M}$, this initial value problem has unique solution $\tilde{h}_t = e^{-tA}\tilde{h}_0$. Since, under the additional Assumption \ref{assumption1}, Lemma 3.3. of \cite{SirignanoSpiliopoulosNN4} and Proposition 2 in \cite{NTK} guarantee that $A$ is positive definite, there exits $\lambda_0>0$ such that for any eigenvalue $\lambda$ of $A$, $\lambda > \lambda_0 > 0$. For any vector $v$, there exist a constant $C\ge 1$ such that
\[\norm{e^{-tA}v} \le Ce^{-\lambda_0t}\norm{v}, \text{ for any }t\ge 0.\]
Considering $v = \tilde{h}_0$, we have $\Vert \tilde{h}_t\Vert \le C e^{-\lambda_0t} \Vert \tilde{h}_0\Vert $, which implies that $\tilde{h}_t \to 0$ as $t \to \infty$.

To show that $K_t = \paren{K_t(x^{(1)}),\ldots, K_t(x^{(M)})} \to 0$ as $t \to \infty$, we consider the two cases in Theorem \ref{CLT:theorem} separately.
\begin{custlist}[Case]
\item When $\gamma \in \paren{\frac{1}{2}, \frac{3}{4}}$ and $\varphi \le \gamma - \frac{1}{2}$, or when $\gamma \in \left[\frac{3}{4}, 1\right)$ and $\varphi < 1-\gamma \le \gamma - \frac{1}{2}$, by equation \eqref{CLT:evolution}, we have
\begin{equation*}
\begin{aligned}
d K_t &= -A K_t dt, \text{ with } K_0 = \begin{cases}
\mathcal{G}, &\text{if } \gamma \in\paren{\frac{1}{2}, \frac{3}{4}}, \varphi = \gamma - \frac{1}{2},\\
0, &\text{if } \gamma \in \paren{\frac{1}{2}, 1}, \varphi < \gamma - \frac{1}{2}, \end{cases}
\end{aligned}
\end{equation*}
where $\mathcal{G} \in \R^M$ is a Gaussian random variable with elements $\mathcal{G}(x)$ as given in equation \eqref{limit_gaussian} for $x \in \mathcal{X}$.
Since $A \in \R^{M \times M }$ is positive definite, we have $K_t = e^{-tA} K_0 \to 0$ as $t \to \infty$ when $\gamma\in(1/2,3/4)$.

\item When $\gamma \in \left[\frac{3}{4}, 1\right)$ and $\varphi = 1-\gamma$, by \eqref{K_t for 1-gamma}, we have
\begin{equation*}
\begin{aligned}
d K_t &= -A K_t - B_t \tilde{h}_t dt, \text{ with } K_0 = \begin{cases}
\mathcal{G}, &\text{if } \gamma = \frac{3}{4},\\
0, &\text{if } \gamma \in \paren{\frac{3}{4}, 1}, \end{cases}
\end{aligned}
\end{equation*}
where $B_t \in \R^{M \times M}$ has elements $\alpha l_t(B_{x,x'}(c,w))$ for $x,x' \in \mathcal{X}$.  The solution of this equation can be written as
\[K_t = e^{-tA}K_0 - \int_0^t e^{-(t-s)A} B_s \tilde{h}_s ds.\]

To find the bound for $\Vert K_t \Vert$, we first need to show that $l_t(B_{x,x'}(c,w))$ is uniformly bounded. Since $B_{x,x'}(c,w) \in C^2_b(\R^{1+d})$ and for any $t\ge 0$, $f \in C^2_b(\R^{1+d})$, by equation \eqref{l_t limit} and Assumption \ref{assumption}, we have
\begin{equation}\label{lt_uni_bounded}
\begin{aligned}
\abs{l_t(f)} &\le C \int_{0}^t \int_{\CX \times \CY} \abs{y-h_s(x')} \pi(dx',dy) ds \le C \int_0^t \frac{1}{M} \sum_{x'\in \mathcal{X}} \abs{\tilde{h}_s(x')} ds\\
& \le C \int_0^t e^{-\lambda_0 s} ds = -\frac{C}{\lambda_0} \left[e^{-\lambda_0t} -1  \right]\le \frac{C}{\lambda_0},
\end{aligned}
\end{equation}
where the unimportant finite constant $0<C<\infty$ may change from line to line. Then, by the fact that $\Vert \tilde{h}_t \Vert \le C e^{-\lambda_0 t} \Vert \tilde{h}_0 \Vert$ and equation \eqref{lt_uni_bounded}, we have
\begin{equation*}
\begin{aligned}
\norm{K_t} &\le Ce^{-\lambda_0t} \norm{K_0} + C \int_0^t e^{-(t-s)\lambda_0} \norm{B_s \tilde{h}_s} ds\\
&\le Ce^{-\lambda_0t} \norm{K_0} + C \int_0^t e^{-(t-s)\lambda_0} \norm{\tilde{h}_s} ds\\
&\le Ce^{-\lambda_0t} \norm{K_0} + C te^{-\lambda_0t} \norm{\tilde{h}_0}.
\end{aligned}
\end{equation*}

Since $\lim_{t\to \infty} t e^{-\lambda_0 t} =0$, we have $\Vert K_t \Vert \to 0$ as $t \to \infty$. Hence, $|K_t(x)| \to 0$ exponentially fast for all $x \in \mathcal{X}$ as $t \to \infty$.

\end{custlist}

\section{Proof of Theorem \ref{thm::Psi}: Convergence of the Second Order Fluctuation Process}\label{sec::Psi}

For $\gamma\in \paren{\frac{3}{4},1}, \varphi = 1-\gamma$, we can further look at the fluctuation process $\Psi^N_t = N^{\zeta - \varphi} (K^N_t - K_t)$, for $\zeta>\varphi$. The evolution of $\Psi^N_t(x)$ can be written as
\begin{equation}\label{Psi^N_t}
\begin{aligned}
\Psi^N_t(x) &= N^{\zeta - \varphi} (K^N_t(x) - K_t(x))\\
&= \int^t_0 \int_{\CX \times \CY} \alpha\paren{y-h_s(x')} N^{\zeta - \varphi}\left[ l^N_s(B_{x,x'}(c,w))-l_s(B_{x,x'}(c,w))\right] \pi(dx',dy) ds\\
&\quad -  \int^t_0 \int_{\CX \times \CY} \alpha \Psi^N_s(x') A_{x,x'} \pi(dx',dy) ds+ N^{\zeta-\varphi}\Gamma^N_t(x) + \Psi^N_0(x)\\
&\quad + N^{\zeta}V_t^N (x)+ N^{\zeta}M_t^{N} (x)+ \frac{1}{N^{\gamma+1-\zeta}} \sum_{k=0}^{\floor{Nt}-1} G^N_k (x),
\end{aligned}
\end{equation}
where $\Psi^N_0(x)= N^{1-\gamma + \zeta}\ip{c\sigma(wx),\mu^N_0}$, and $\Gamma^N_t$, $V^N_t$, $M^N_t, G^N_k$ are as given in Sections \ref{sec::LLN} and \ref{sec::CLT}. In particular, we can write
\begin{equation*}
\begin{aligned}
N^{\zeta-\varphi} \Gamma^N_t (x)
&= -\frac{\alpha}{N^{\varphi}} \int_0^t \int_{\CX \times \CY} \Psi^N_s(x') l^N_s(B_{x,x'}(c,w)) \pi(dx',dy)ds  \\
&\quad -\frac{\alpha}{N^{2\varphi-\zeta}} \int_0^t \int_{\CX \times \CY} K_s(x') l^N_s(B_{x,x'}(c,w)) \pi(dx',dy)ds.\label{Eq:Gamma_2}
\end{aligned}
\end{equation*}

We see that if $\zeta \le \gamma - \frac{1}{2}$, the last 3 remainder terms in equation \eqref{Psi^N_t} converge to 0 as $N\to \infty$ by the analysis in Section \ref{CLT:sec::remainder}. In addition, if $\zeta = \gamma -\frac{1}{2}$,
\[\Psi^N_0(x)=N^{1-\gamma + \zeta}\ip{c\sigma(wx),\mu^N_0} = \ip{c\sigma(wx),\sqrt{N}\mu^N_0} \xrightarrow{d} \mathcal{G}(x),\]
where $\mathcal{G}(x)$ is the Gaussian random variable defined in \eqref{limit_gaussian}.

 For any fixed $f\in C^3_b(\R^{1+d})$, let $L^
N_t(f) = N^{\zeta - \varphi}\left[ l^N_t(f)-l_t(f)\right]$ its  evolution can be written as
\begin{equation}\label{L^N_fluc}
\begin{aligned}
L^
N_t(f) &= N^{\zeta - \varphi}\left[ l^N_t(f)-l^N_0(f)-l_t(f)+l^N_0(f)\right]\\
&=\frac{\alpha}{N^{2\varphi-\zeta}}\int_0^t \int_{\CX \times \CY} \paren{y-h_s(x')}  \ip{ C_{x'}^{f}(c,w), N^{ \varphi}(\mu^N_s-\mu_0)} \pi(dx',dy) ds\\
&\quad - \frac{\alpha}{N^{2\varphi-\zeta}} \int_0^t \int_{\CX \times \CY}  N^{ \varphi}\paren{h^N_s(x')-h_s(x')}  \ip{ C_{x'}^{f}(c,w),\mu_0} \pi(dx',dy) ds\\
&\quad + \Gamma^N_{2,t}+ N^{\zeta-\frac{1}{2}}\ip{f,\sqrt{N}(\mu^N_0 - \mu_0)}+ N^{\zeta +\gamma -1}M^{1,N}_{f,t} + N^{\zeta +\gamma -1}M^{2,N}_{f,t}+ O\paren{N^{-1+\zeta}},
\end{aligned}
\end{equation}
where $\Gamma^N_{2,t}= -\frac{\alpha}{N^{2\varphi-\zeta}}\int_0^t \int_{\CX \times \CY} K^N_s(x')\ip{C^f_{x'}(c,w),(\mu^N_s-\mu_0)}\pi(dx',dy)ds$.

Following lemmas show compact containment and regularity of $L^N_t(f)$ for any fixed $f\in C^3_b(\R^{1+d})$.
\begin{lemma}\label{lemma:L^N_compact containment}
When $\zeta \le 2-2\gamma$, for any fixed $f \in C^3_b(\R^{1+d})$, there exists a constant $C<\infty$, such that
\[\sup_{N \in \mathbb{N}, 0\le t\le T} \E \left[\abs{L^
N_t(f)}^2 \right] < C.\]
Thus, for any $\epsilon>0$, there exist a compact interval  $U \subset \R$, such that
$\sup_{N\in \mathbb{N}, 0\le t\le T} \P\paren{L^N_t(f) \notin U} < \epsilon. $
\end{lemma}
\begin{proof}By equation \eqref{L^N_fluc} and the Cauchy-Schartz inequality, we have
\begin{equation*}
\begin{aligned}
\abs{L^N_t(f)}^2
&\le \frac{C}{N^{2(2\varphi-\zeta)}}\int_0^t \int_{\CX \times \CY}  \abs{y-h_s(x')}^2 \pi(dx',dy) ds \int_0^t \int_{\CX \times \CY}\abs{\ip{ C^f_{x'}(c,w), \eta^N_s}}^2 \pi(dx',dy) ds\\
&\quad + \frac{C}{N^{2(2\varphi-\zeta)}} \int_0^t \int_{\CX \times \CY}  \abs{K^N_s(x')}^2  \pi(dx',dy)ds \int_0^t \int_{\CX \times \CY}\abs{\ip{ C^f_{x'}(c,w),\mu_0}}^2 \pi(dx',dy) ds\\
&\quad + C \paren{\abs{\Gamma^N_{2,t}}^2+ \abs{N^{\zeta-\frac{1}{2}}\ip{f,\sqrt{N}(\mu^N_0 - \mu_0)}}^2+ \abs{N^{\zeta +\gamma -1}M^{1,N}_{f,t}}^2 + \abs{N^{\zeta +\gamma -1}M^{2,N}_{f,t}}^2+ O\paren{N^{-2+2\zeta}}},
\end{aligned}
\end{equation*}

When $\zeta \le 2\varphi = 2-2\gamma$, $0\le t\le T$, the expectation of the first two terms and $\abs{\Gamma^N_{2,t}}^2$  are bounded by  Assumption \ref{assumption}, Lemmas \ref{CLT:lemma:eta_compact_contatinment} and \ref{CLT:lemma:bound of ex_Kt}.
 Since $\gamma > \frac{3}{4}$, $\zeta < \frac{1}{2}$ and $\zeta + \gamma - 1 \le 1-\gamma <\frac{1}{2}$, by similar analysis as in Section \ref{CLT:sec::remainder}, the remainder terms all converges to 0 as $N\to \infty$. The result of the lemma follows.
\end{proof}

\begin{lemma}\label{lemma:L^N_regularity}
When $\zeta \le 2 - 2\gamma$, for any  $f \in C^3_b(\R^{1+d})$, $\delta \in (0,1)$, there is a constant $C<\infty$ such that for $0\le u \le \delta$, $0\le v\le \delta \wedge t$, and $t \in [0,T]$,
\[\E \left[q\paren{L^N_{t+u}(f),L^N_{t}(f)}q\paren{L^N_{t}(f),L^N_{t-v}(f)}\vert \CF^N_t\right]\le \frac{C}{N^{2-2\gamma-\zeta}}\delta  + \frac{C}{N^{2-\zeta-\gamma}}.\]
\end{lemma}
\begin{proof} For $0\le s<t\le T$, the leading terms of $L^N_t$ in equation \eqref{L^N_fluc} gives
\begin{align*}
\E\left[\abs{L^N_t(x) - L^N_s(x)}\big\vert \CF^N_s \right]&\le \frac{C}{N^{2\varphi-\zeta}}(t-s)  +  \frac{C_1}{N^{2\varphi-\zeta}} \int^t_s \int_{\CX \times \CY} \E\left[\abs{K^N_s(x')} \big\vert \CF^N_s\right]  \pi(dx',dy) d\tau\\
&\quad  + \E\left[\abs{N^{\zeta +\gamma -1}\paren{M^{1,N}_{f,t}-M^{1,N}_{f,s}}} \big \vert \CF^N_s \right] +\E\left[\abs{N^{\zeta +\gamma -1}\paren{M^{2,N}_{f,t}-M^{2,N}_{f,s}}} \big \vert \CF^N_s \right]\\
&\le  \frac{C}{N^{2\varphi-\zeta}}(t-s) +\E\left[\abs{N^{\zeta +\gamma -1}\paren{M^{1,N}_{f,t}-M^{1,N}_{f,s}}}^2 \big \vert \CF^N_s \right]^{\frac{1}{2}}\\
&\quad +\E\left[\abs{N^{\zeta +\gamma -1}\paren{M^{2,N}_{f,t}-M^{2,N}_{f,s}}}^2 \big \vert \CF^N_s \right]^{\frac{1}{2}} \\
&\le  \frac{C}{N^{2-2\gamma-\zeta}}\delta  + \frac{C}{N^{2-\zeta-\gamma}} .
\end{align*}
The last inequality holds because one can show that for $i=1,2$
\[\E\left[\abs{N^{\zeta +\gamma -1}\paren{M^{i,N}_{f,t}-M^{i,N}_{f,s}}}^2 \big \vert \CF^N_s \right] \le \frac{C \delta}{N^{1-2(\zeta+\gamma-1)}} + \frac{C}{N^{2-2(\zeta+\gamma-1)}},\]
following an analysis similar to Lemma 3.1 of \cite{SirignanoSpiliopoulosNN1}.
\end{proof}

Denote $\mathfrak{K}^N_t=(\mu^N_t, h^N_t, l^N_t(B_{x,x'}(c,w)), K^N_t)$. In the next lemma, we prove the convergence of the processes $(\mathfrak{K}^N_t,L^N_t(f))$ in distribution in the space $D_{E_2}([0,T])$, where $E_2 = \CM(\R^{1+d}) \times \R^M \times \R \times \R^M \times \R$.
\begin{lemma}\label{lemma:limit_L}
When $\gamma \in \paren{\frac{3}{4},1}$, $\varphi = 1-\gamma$ and $\zeta\le 2\varphi$, for any fixed $f\in C^3_b(\R^{1+d})$, the processes $(\mathfrak{K}^N_t,L^N_t(f))$ in distribution in the space $D_{E_2}([0,T])$ to $(\mathfrak{K}_t,L_t(f))$, where $\mathfrak{K}_t = (\mu_t, h_t, l_t(B_{x,x'}(c,w)), K_t)$ satisfying equations \eqref{LLN:limit_evolution}, \eqref{l_t limit}, and \eqref{K_t for 1-gamma}. When $\zeta <2\varphi$, $L_t(f)=0$. When $\zeta = 2\varphi$, $L_t(f)$ satisfies equation \eqref{limit_Lt}.
\end{lemma}
\begin{proof}By the analysis in Section \ref{CLT:sec:pf_conv_K}, $\{\mathfrak{K}^N\}_{N\in\mathbb{N}}$ is relatively compace in $D_{E_1}([0,T])$, where $E_1= \CM(\R^{1+d}) \times \R^M \times \R \times \R^M$. By Lemmas \ref{lemma:L^N_compact containment} and \ref{lemma:L^N_regularity}, $\{L^N(f)\}_{N\in\mathbb{N}}$ is relatively compact in $D_{\R}([0,T])$. These implies that the probability measures of the family of processes $\{\mathfrak{K}^N\}_{N\in\mathbb{N}}$ and the probability measures of the family of processes $\{L^N(f)\}_{N\in\mathbb{N}}$ are tight. Therefore, $\{\mathfrak{K}^N,L^N(f)\}_{N\in\mathbb{N}}$ is tight. Hence, $\{\mathfrak{K}^N,L^N(f)\}_{N\in\mathbb{N}}$ is relatively compact in $D_{E_2}([0,T])$.

Denote $\pi^N \in \CM(D_{E_2}([0,T])$ the probability measure corresponding to $(\mathfrak{K}^N,L^N(f))$. 
We now show that any limit point $\pi$ of a convergence subsequence $\pi^{N_k}$ is a Dirac measure concentrated on $(\mathfrak{K},L(f))\in D_{E_2}([0,T])$.

\begin{custlist}[Case]
\item When $\zeta < 2\varphi$, for any $t \in [0,T]$, $b_1, \ldots, b_p \in C_b(\R)$, and $0 \le s_1 < \cdots < s_p \le t$, we define $F_4(\mathfrak{K},L(f)): D_{E_2}([0,T]) \to \R_+$ as
\begin{equation}\label{F4}
\begin{aligned}
F_4(\mathfrak{K},L(f)) &= F_3(\mu, h, l_B, K) + \left\vert \left(L_t(f)-0 \right) \times b_1(L_{s_1}(f)) \times \cdots \times b_p(L_{s_p}(f)) \right \vert,
\end{aligned}
\end{equation}
where $F_3(\mu, h, l_B, K)$ is as given in equation \eqref{F_3}. By equation \eqref{L^N_fluc}, Lemma \ref{CLT:lemma:martingale2}, and similar analysis as in Lemma \ref{lemma:L^N_compact containment}, we have
\begin{align*}
\E_{\pi^N} \left[F_4(\mathfrak{K},L(f))\right] &= E_{\pi^N} \left[F_3(\mu, h, l_B, K)\right] + \E \left[ \left\vert \left(L^N_t(f)-0 \right) \times \prod_{i=1}^{p} b_i(L^N_{s_i}(f))  \right \vert\right] \\
&\le C\paren{\frac{1}{N^{1-\gamma}}+ \frac{1}{N^{\frac{1}{2}-\varphi}}+\frac{1}{N^{1-\varphi}}+ \frac{1}{N^{\gamma-\varphi}}} \\
&\quad + \frac{C}{N^{2\varphi-\zeta}} + \frac{C}{N^{\frac{1}{2}-\zeta}} + \E\left[\abs{N^{\zeta+\gamma-1}M^{1,N}_{f,t}}^2\right]^{\frac{1}{2}} +\E\left[\abs{N^{\zeta+\gamma-1}M^{2,N}_{f,t}}^2\right]^{\frac{1}{2}}+ \frac{C}{N^{1-\zeta}}\\
&\le C \paren{\frac{1}{N^{1-\gamma}} + \frac{1}{N^{2\varphi-\zeta}} }.
\end{align*}
Therefore, $\lim_{N\to \infty} \E_{\pi^N}[F_4(\mathfrak{K},L(f))] = 0$. Since $F_4(\cdot)$ is continuous and uniformly bounded,
\[\lim_{N\to \infty} \E_{\pi^N}\left[F_4(\mathfrak{K},L(f))\right] = \E_{\pi}\left[F_4(\mathfrak{K},L(f))\right] = 0.\]

We have shown that any limit point $\pi$ of a convergence sequence must be a Dirac measure concentrated $(\mathfrak{K},L(f))\in D_{E_2}([0,T])$, where $\mathfrak{K}=(\mu,h,l_B,K)$ satisfies equation \eqref{LLN:limit_evolution},\eqref{l_t limit}, and \eqref{K_t for 1-gamma}, and $L_t(f)=0$. Since the solutions to equations \eqref{LLN:limit_evolution} and \eqref{K_t for 1-gamma} are unique, by Prokhorov's theorem, the processes $(\mathfrak{K}^N_t,L^N_t(f))$ converges in distribution to $(\mathfrak{K}_t,0)$.
\item When $\zeta = 2\varphi$, for any $t \in [0,T]$, $b_1, \ldots, b_p \in C_b(\R)$, and $0 \le s_1 < \cdots < s_p \le t$, we define $F_4(\mathfrak{K},L(f)): D_{E_2}([0,T]) \to \R_+$ as
\begin{equation}\label{F5}
\begin{aligned}
F_5(\mathfrak{K},L(f)) &= F_3(\mu, h, l_B, K) + \left\vert \left(L_t(f)-\int_0^t \int_{\CX \times \CY}  \alpha\paren{y-h_s(x')} l_s( C^f_{x'}(c,w)) \pi(dx',dy) ds\right.\right.\\
&\qquad \left.\left.+  \int_0^t \int_{\CX \times \CY} \alpha K_s(x')\ip{C^f_{x'}(c,w),\mu_0} \pi(dx',dy) ds\right) \times b_1(L_{s_1}(f)) \times \cdots \times b_p(L_{s_p}(f)) \right \vert,
\end{aligned}
\end{equation}
where $F_3(\mu, h, l_B, K)$ is as given in equation \eqref{F_3}. We first note that by equation \eqref{L^N_fluc}
\begin{equation*}
\begin{aligned}
&L^N_t(f) - \int_0^t \int_{\CX \times \CY}  \alpha\paren{y-h^N_s(x')} l^N_s( C_{x'}^{f}(c,w)) \pi(dx',dy) ds\\
&\quad +  \int_0^t \int_{\CX \times \CY}  \alpha K^N_s(x') \ip{ C_{x'}^{f}(c,w),\mu^N_0} \pi(dx',dy) ds\\
&= N^{\zeta-\frac{1}{2}}\ip{f,\sqrt{N}(\mu^N_0 - \mu_0)}+ N^{\zeta +\gamma -1}M^{1,N}_{f,t} + N^{\zeta +\gamma -1}M^{2,N}_{f,t}+ O\paren{N^{-1+\zeta}}\\
&\quad+ \frac{\alpha}{N^{\varphi}}\int_0^t \int_{\CX \times \CY}   K^N_s(x') \ip{ C_{x'}^{f}(c,w), \eta^N_0} \pi(dx',dy) ds
\end{aligned}
\end{equation*}

Then by Lemma \ref{CLT:lemma:martingale2}, and similar analysis as for equation \eqref{eq:K^N:temp}, we have
\begin{align*}
\E_{\pi^N} \left[F_5(\mathfrak{K},L(f))\right] &= E_{\pi^N} \left[F_3(\mu, h, l_B, K)\right] \\
&\quad + \E \left[ \left\vert \left(L^N_t(f)-\int_0^t \int_{\CX \times \CY}  \alpha\paren{y-h^N_s(x')} l^N_s( C_{x'}^{f}(c,w)) \pi(dx',dy) ds\right.\right.\right.\\
&\qquad \left.\left.\left.+  \int_0^t \int_{\CX \times \CY} \alpha K^N_s(x')\ip{C_{x'}^{f}(c,w),\mu_0} \pi(dx',dy) ds\right) \times \prod_{i=1}^{p} b_i(L^N_{s_i}(f))  \right \vert\right]\\
&\le C\paren{\frac{1}{N^{1-\gamma}}+ \frac{1}{N^{\frac{1}{2}-\varphi}}+\frac{1}{N^{1-\varphi}}+ \frac{1}{N^{\gamma-\varphi}}} \\
&\quad  + \frac{C}{N^{\frac{1}{2}-\zeta}} + \E\left[\abs{N^{\zeta+\gamma-1}M^{1,N}_{f,t}}^2\right]^{\frac{1}{2}} +\E\left[\abs{N^{\zeta+\gamma-1}M^{2,N}_{f,t}}^2\right]^{\frac{1}{2}}+ \frac{C}{N^{1-\zeta}}+\frac{C}{N^{\varphi}}\\
&\le C\paren{\frac{1}{N^{1-\gamma}}+ \frac{1}{N^{\frac{1}{2}-\zeta}}}.
\end{align*}
Therefore, $\lim_{N\to \infty} \E_{\pi^N}[F_5(\mathfrak{K},L(f))] = 0$. Since $F_5(\cdot)$ is continuous and uniformly bounded,
\[\lim_{N\to \infty} \E_{\pi^N}\left[F_5(\mathfrak{K},L(f))\right] = \E_{\pi}\left[F_5(\mathfrak{K},L(f))\right] = 0.\]

We have shown that any limit point $\pi$ of a convergence sequence must be a Dirac measure concentrated $(\mathfrak{K},L(f))\in D_{E_2}([0,T])$, where $\mathfrak{K}=(\mu,h,l_B,K)$ satisfies equation \eqref{LLN:limit_evolution},\eqref{l_t limit}, and \eqref{K_t for 1-gamma}, and $L_t(f)$ satisfies \eqref{limit_Lt}. Since the solutions to equations \eqref{LLN:limit_evolution} and \eqref{K_t for 1-gamma} are unique, by Prokhorov's theorem, the processes $(\mathfrak{K}^N_t,L^N_t(f))$ converges in distribution to $(\mathfrak{K}_t,L_t(f))$.
\end{custlist}
\end{proof}

Moving back to the analysis of $\Psi^N_t$, when $\sigma \in C^4_b(\R^{1+d})$, $\zeta \le \min\{\gamma -\frac{1}{2}, 2-2\gamma\}$, we first show compact containment of $\Psi^N_t$ in the next lemma.
\begin{lemma}\label{Psi_compact_containment}
There exit a constant $C<\infty$, such that
\[\sup_{N \in \mathbb{N}, 0\le t\le T} \E \left[\abs{\Psi^
N_t(x)}^2 \right] < C.\]
Thus, for any $\epsilon>0$, there exist a compact subset $U \subset \R^M$, such that
$\sup_{N\in \mathbb{N}, 0\le t\le T} \P\paren{\Psi^N_t \notin U} < \epsilon. $
\end{lemma}
\begin{proof}In the proof below, $C<\infty$ represents some positive constant, which may be different from line to line.  Since by equation \eqref{Eq:Gamma_2}, the term $N^{\zeta-\varphi}\Gamma^N_t$ involves the term $K_t(x)$, we first look at the bound for $K_t(x)$. By the Cauchy-Schwarz inequality, equations \eqref{LLN:limit_evolution}, \eqref{l_t limit} and the analysis in Lemma \ref{CLT:lemma:bound of ex_Kt}, for any $t \in [0,T]$, we have
\begin{equation*}
\abs{K_t(x)}^2 \le Ct^2 + Ct \int_0^t \int_{\CX \times \CY} \abs{K_s(x')}^2 \pi(dx',dy)ds \le CT^2 + \frac{CT}{M} \int_0^t  \sum_{x'\in\mathcal{X}}\abs{K_s(x')}^2 \pi(dx',dy)ds.
\end{equation*}

Summing over $x\in \mathcal{X}$ on both sides gives
\begin{equation*}
\sum_{x\in \mathcal{X}} \abs{K_t(x)}^2  \le CT^2M + CT \int_0^t  \sum_{x'\in\mathcal{X}}\abs{K_s(x')}^2 \pi(dx',dy)ds.
\end{equation*}

Then by applying Gr\"onwall's inequality, we have
\begin{equation*}
\sup_{0\le t\le T} \sum_{x\in \mathcal{X}} \abs{K_t(x)}^2  \le \sup_{0\le t\le T} CT^2M \exp\paren{CTt} < C,
\end{equation*}
which implies that $\sup_{0\le t\le T}\abs{K_t(x)}^2<C$ for any $x\in \mathcal{X}$. By equation \eqref{Psi^N_t} and the analysis in Section \ref{CLT:sec::remainder}, we have
\begin{equation*}
\begin{aligned}
\E \left[\abs{\Psi^N_t(x)}^2\right] &\le Ct^2 + Ct\int_0^t \int_{\CX \times \CY} \E \left[\abs{\Psi^N_t(x')}^2\right] \pi(dx',dy)ds + \frac{Ct^2}{N^{2(2\varphi-\zeta)}} \\
&\quad + C\E \left[\abs{N^{\zeta}M^N_t(x)}^2\right] + O\paren{N^{-2(1-\zeta)}} + O\paren{N^{-2(\gamma-\zeta)}}.
\end{aligned}
\end{equation*}

Summing over $x\in \mathcal{X}$ on both sides gives
\begin{equation*}
\begin{aligned}
\sum_{x\in \mathcal{X}}\E \left[\abs{\Psi^N_t(x)}^2\right] &\le CMT^2 + CT\int_0^t \sum_{x'\in\mathcal{X}} \E \left[\abs{\Psi^N_t(x')}^2\right] \pi(dx',dy)ds.
\end{aligned}
\end{equation*}

By Gr\"onwall's inequality, we get
\begin{equation*}
\sup_{0\le t\le T} \sum_{x\in \mathcal{X}} \E\left[\abs{\Psi^N_t(x)}^2\right]  \le \sup_{0\le t\le T} CT^2M \exp\paren{CTt} < C,
\end{equation*}
which implies that $\sup_{0\le t\le T}\E\left[\abs{\Psi^N_t(x)}^2\right] < C$ for any $x\in \mathcal{X}$. The result of the lemma then follows.
\end{proof}

The next lemma establishes the regularity of the process $\Psi^N_t$ in $D_{\R^M}([0,T])$. For the purpose of this lemma, we denote $q(z_1,z_2) = \min\{\norm{ z_1-z_2}_{l^1},1\}$ for $z_1,z_2 \in \R^M$. The proof of the lemma is similar to that for Lemma \ref{CLT:lemma:regularity}, which we omit here.
\begin{lemma}\label{Psi_regularity}
For any   $\delta \in (0,1)$, there is a constant $C<\infty$ such that for $0\le u \le \delta$, $0\le v\le \delta \wedge t$, and $t \in [0,T]$,
\[\E \left[q\paren{\Psi^N_{t+u},\Psi^N_{t}}q\paren{\Psi^N_{t},\Psi^N_{t-v}}\vert \CF^N_t\right]\le {C\delta}+\frac{C}{N^{1-\zeta}}.\]
\end{lemma}
Combining these with our analysis of $L^N_t(f)$, we can now identify the limit for $\Psi^N_t$. Note that if $\zeta = \gamma -\frac{1}{2}$,
\[\Psi^N_0(x)=N^{1-\gamma + \zeta}\ip{c\sigma(wx),\mu^N_0} = \ip{c\sigma(wx),\sqrt{N}\mu^N_0} \xrightarrow{d} \mathcal{G}(x),\]
where $\mathcal{G}(x)$ is the Gaussian random variable defined in \eqref{limit_gaussian}. And if $\zeta <\gamma-\frac{1}{2}$, $\Psi^N_0(x) \xrightarrow{d} 0$. We denote $\mathfrak{L}^N_t=(\mathfrak{K}^N_t,L^N_t(B_{x,x'}(c,w))$. In the next lemma, we prove the convergence of the processes $(\mathfrak{L}^N_t,\Psi^N_t)$ in distribution in the space $D_{E_3}([0,T])$, where $E_3 = \CM(\R^{1+d}) \times \R^M \times \R \times \R^M \times \R \times \R^M$.
\begin{lemma}When $\gamma \in \paren{\frac{3}{4},1}$, $\varphi = 1-\gamma$ and $\zeta\le \gamma -\frac{1}{2}$, the processes $(\mathfrak{L}^N_t,\Psi^N_t)$ in distribution in the space $D_{E_3}([0,T])$ to $(\mathfrak{L}_t,\Psi_t)$. In particular, $\mathfrak{L}_t = (\mu_t, h_t, l_t(B_{x,x'}(c,w)), K_t, L_t(B_{x,x'}(c,w))$ satisfies equations \eqref{LLN:limit_evolution}, \eqref{l_t limit}, and \eqref{K_t for 1-gamma}. When $\zeta <2\varphi$, $L_t(B_{x,x'}(c,w))=0$ and $\Psi_t$ satisfies equation \eqref{limit_Psi}. When $\zeta = 2\varphi$, $L_t(B_{x,x'}(c,w))$ satisfies equation \eqref{limit_Lt} and $\Psi_t$ satisfies equation \eqref{limit_Psi_2}.
\end{lemma}
\begin{proof}
By analysis in Lemma \ref{lemma:limit_L}, $\{\mathfrak{L}^N\}_{N\in\mathbb{N}}$ is relatively compact in $D_{E_2}([0,T])$, where $E_2= \CM(\R^{1+d}) \times \R^M \times \R \times \R^M \times \R$. By Lemmas \ref{Psi_compact_containment} and \ref{Psi_regularity}, $\{\Psi^N\}_{N\in\mathbb{N}}$ is relatively compact in $D_{\R^M}([0,T])$. These implies that the probability measures of the family of processes $\{\mathfrak{L}^N\}_{N\in\mathbb{N}}$ and the probability measures of the family of processes $\{\Psi^N\}_{N\in\mathbb{N}}$ are tight. Therefore, $\{\mathfrak{L}^N,\Psi^N\}_{N\in\mathbb{N}}$ is tight. Hence, $\{\mathfrak{L}^N,\Psi^N\}_{N\in\mathbb{N}}$ is relatively compact in $D_{E_3}([0,T])$.

Denote $\pi^N \in \CM(D_{E_3}([0,T])$ the probability measure corresponding to $(\mathfrak{L}^N,\Psi^N)$. 
We now show that any limit point $\pi$ of a convergence subsequence $\pi^{N_k}$ is a Dirac measure concentrated on $(\mathfrak{L},\Psi)\in D_{E_3}([0,T])$.

\begin{custlist}[Case]
\item When $\gamma \in \paren{ \frac{3}{4}, \frac{5}{6}}$ and $\zeta \le \gamma - \frac{1}{2} < 2\varphi$, or when $\gamma \in \left[\frac{5}{6}, 1\right)$ and $\zeta < 2\varphi \le \gamma - \frac{1}{2}$, for any $t \in [0,T]$, $d_1, \ldots, d_p \in C_b(\R^M)$, and $0 \le s_1 < \cdots < s_p \le t$, we define $F_6(\mathfrak{L},\Psi): D_{E_3}([0,T]) \to \R_+$ as
\begin{equation}
\begin{aligned}
F_6(\mathfrak{L},\Psi) &= F_4(\mathfrak{K},L(B_{x,x'}(c,w))) + \sum_{x\in\mathcal{X}} \left\vert \left(\Psi_t(x)-\Psi_0(x) -\int^t_0 \int_{\CX \times \CY} \alpha\paren{y-h_s(x')} L_s(B_{x,x'}(c,w)) \pi(dx',dy) ds\right.\right.\\
& \qquad \left.\left.+   \int^t_0 \int_{\CX \times \CY} \alpha \Psi_s(x') \ip{B_{x,x'}(c,w),\mu_0} \pi(dx',dy) ds \right)
\times d_1(\Psi_{s_1}) \times \cdots \times d_p(\Psi_{s_p}) \right \vert,
\end{aligned}
\end{equation}
where $F_4(\mathfrak{K},L(B_{x,x'}(c,w)))$ is as given in equation \eqref{F4}. Note that by equation \eqref{Psi^N_t},
\begin{equation*}
\begin{aligned}
&\Psi^N_t(x)-\Psi^N_0(x) -\int^t_0 \int_{\CX \times \CY} \alpha\paren{y-h^N_s(x')} L^N_s(B_{x,x'}(c,w)) \pi(dx',dy) ds\\
&\quad  +   \int^t_0 \int_{\CX \times \CY} \alpha \Psi^N_s(x') \ip{B_{x,x'}(c,w),\mu^N_0} \pi(dx',dy) ds \\
&= \frac{\alpha}{N^{\varphi}} \int_0^t \int_{\CX \times \CY} K_s^N(x') L^N_s(B_{x,x'}(c,w))
\pi(dx',dy) ds + \frac{\alpha}{N^{\varphi}} \int_0^t \int_{\CX \times \CY} \Psi_s^N(x') \ip{B_{x,x'}(c,w),\eta^N_0}
\pi(dx',dy) ds\\
&\quad + N^{\zeta-\varphi} \Gamma^N_{t}(x) + N^{\zeta}V_t^{N}(x) + N^{\zeta}M_t^{N}(x) + \frac{1}{N^{\gamma+1-\zeta}} \sum_{k=0}^{\floor{Nt}-1} G^N_{k}(x).
\end{aligned}
\end{equation*}

By the Cauchy-Schwartz inequality, Lemmas \ref{CLT:lemma:eta_compact_contatinment},  \ref{lemma:L^N_compact containment}, and \ref{Psi_compact_containment}, we have
\begin{equation*}
\begin{aligned}
&\frac{\alpha}{N^{\varphi}}\E \left[\abs{ \int_0^t \int_{\CX \times \CY} K_s^N(x') L^N_s(B_{x,x'}(c,w))
\pi(dx',dy) ds +  \int_0^t \int_{\CX \times \CY} \Psi_s^N(x') \ip{B_{x,x'}(c,w),\eta^N_0}
\pi(dx',dy) ds} \right] \le \frac{C}{N^{\varphi}},
\end{aligned}
\end{equation*}
and
\begin{equation*}
\begin{aligned}
\E\left[\abs{N^{\zeta-\varphi} \Gamma^N_t}\right]&\le \frac{C}{N^{\varphi}}\int_0^t \int_{\CX \times \CY} \E\left[\abs{\Psi_s^N(x')}^2\right]\pi(dx',dy) ds \int_0^t \int_{\CX \times \CY} \E\left[\abs{l^N_s(B_{x,x'}(c,w))}^2\right]\pi(dx',dy) ds \\
&\quad + \frac{C}{N^{2\varphi-\zeta}}\int_0^t \int_{\CX \times \CY}\E\left[\abs{ l^N_s(B_{x,x'}(c,w))}\right]  \pi(dx',dy) ds\\
&\le C \paren{\frac{1}{N^{\varphi}} + \frac{1}{N^{2\varphi-\zeta}}}.
\end{aligned}
\end{equation*}

Putting everything together, by equation \eqref{Psi^N_t}, Lemmas \ref{lemma:L^N_compact containment}, \ref{lemma:limit_L}, \ref{Psi_compact_containment}, and the analysis in Section \ref{sec::CLT}, we have
\begin{align*}
\E_{\pi^N}  \left[F_6(\mathfrak{L},\Psi)\right] &= \E_{\pi^N}  \left[F_4(\mathfrak{K},L(B_{x,x'}(c,w)))\right] \\
& \quad + \sum_{x\in\mathcal{X}} \E \left[ \left\vert \left(\Psi^N_t(x)-\Psi^N_0(x) -\int^t_0 \int_{\CX \times \CY} \alpha\paren{y-h^N_s(x')} L^N_s(B_{x,x'}(c,w)) \pi(dx',dy) ds\right.\right.\right.\\
&\qquad \left.\left.\left.+   \int^t_0 \int_{\CX \times \CY} \alpha \Psi^N_s(x') \ip{B_{x,x'}(c,w),\mu^N_0} \pi(dx',dy) ds \right)
\times \prod_{i=1}^p d_i(\Psi^N_{s_i})  \right \vert \right]\\
&\le C \paren{\frac{1}{N^{1-\gamma}} + \frac{1}{N^{2\varphi-\zeta}} } + C\paren{\frac{1}{N^{\varphi}} + \frac{1}{N^{2\varphi-\zeta}}} \\
&\quad + C \E \left[\abs{\paren{N^{\zeta}V_t^N + N^{\zeta}M_t^{N} + \frac{1}{N^{\gamma+1-\zeta}} \sum_{k=0}^{\floor{Nt}-1} G^N_k} \times \prod_{i=1}^p d_i(\Psi^N_{s_i})}\right]\\
&\le C \paren{\frac{1}{N^{1-\gamma}} + \frac{1}{N^{2\varphi-\zeta}} }.
\end{align*}
Therefore, $\lim_{N\to \infty} \E_{\pi^N}[F_6(\mathfrak{L},\Psi)] = 0$. Since $F_6(\cdot)$ is continuous and uniformly bounded,
\[\lim_{N\to \infty} \E_{\pi^N}\left[F_6(\mathfrak{L},\Psi)\right] = \E_{\pi}\left[F_6(\mathfrak{L},\Psi)\right] = 0.\]

We have shown that any limit point $\pi$ of a convergence sequence must be a Dirac measure concentrated $(\mathfrak{L},\Psi)\in D_{E_3}([0,T])$, where $\mathfrak{L} = (\mathfrak{K},L(B_{x,x'}(c,w)))=(\mu,h,l(B_{x,x'}(c,w)),K, L(B_{x,x'}(c,w)))$ satisfies equation \eqref{LLN:limit_evolution},\eqref{l_t limit}, and \eqref{K_t for 1-gamma}, $L_t(B_{x,x'}(c,w)))=0$, and $\Psi_t$ satisfies equation \eqref{limit_Psi}. Since the solutions to equations \eqref{LLN:limit_evolution}, \eqref{K_t for 1-gamma} and  \eqref{limit_Psi} are unique, by Prokhorov's theorem, the processes $(\mathfrak{L}^N_t,\Psi^N_t)$ converges in distribution to $(\mathfrak{L}_t,\Psi_t)$.

\item  When $\gamma \in \left[\frac{5}{6}, 1\right)$ and $\zeta = 2-2\gamma = 2\varphi$, for any $t \in [0,T]$, $d_1, \ldots, d_p \in C_b(\R^M)$, and $0 \le s_1 < \cdots < s_p \le t$, we define $F_7(\mathfrak{L},\Psi): D_{E_3}([0,T]) \to \R_+$ as
\begin{equation}
\begin{aligned}
F_7(\mathfrak{L},\Psi) &= F_5(\mathfrak{K},L(B_{x,x'}(c,w))) + \sum_{x\in\mathcal{X}} \left\vert \left(\Psi_t(x)-\Psi_0(x) -\int^t_0 \int_{\CX \times \CY} \alpha\paren{y-h_s(x')} L_s(B_{x,x'}(c,w)) \pi(dx',dy) ds\right.\right.\\
& \qquad \left.+   \int^t_0 \int_{\CX \times \CY} \alpha \Psi_s(x') A_{x,x'} \pi(dx',dy) ds +   \int^t_0 \int_{\CX \times \CY} \alpha K_s(x') l_s(B_{x,x'}(c,w)) \pi(dx',dy) ds\right)\\
&\qquad \left.\times d_1(\Psi_{s_1}) \times \cdots \times d_p(\Psi_{s_p}) \right \vert,
\end{aligned}
\end{equation}
where $F_5(\mathfrak{K},L(B_{x,x'}(c,w)))$ is as given in equation \eqref{F5}. Note that by equation \eqref{Psi^N_t},
\begin{equation*}
\begin{aligned}
&\Psi^N_t(x)-\Psi^N_0(x) -\int^t_0 \int_{\CX \times \CY} \alpha\paren{y-h^N_s(x')} L^N_s(B_{x,x'}(c,w)) \pi(dx',dy) ds \\
&\quad +  \int^t_0 \int_{\CX \times \CY} \alpha \Psi^N_s(x') \ip{B_{x,x'}(c,w),\mu^N_0} \pi(dx',dy) ds+   \int^t_0 \int_{\CX \times \CY} \alpha K^N_s(x') l^N_s(B_{x,x'}(c,w)) \pi(dx',dy) ds \\
&= \frac{\alpha}{N^{\varphi}} \int_0^t \int_{\CX \times \CY} K_s^N(x') L^N_s(B_{x,x'}(c,w))
\pi(dx',dy) ds + \frac{\alpha}{N^{\varphi}} \int_0^t \int_{\CX \times \CY} \Psi_s^N(x') \ip{B_{x,x'}(c,w),\eta^N_0}
\pi(dx',dy) ds\\
&\quad + N^{\zeta}V_t^N(x) + N^{\zeta}M_t^{N}(x) + \frac{1}{N^{\gamma+1-\zeta}} \sum_{k=0}^{\floor{Nt}-1} G^N_k(x).
\end{aligned}
\end{equation*}
By equation \eqref{Psi^N_t}, Lemmas \ref{lemma:L^N_compact containment}, \ref{lemma:limit_L}, \ref{Psi_compact_containment}, and analysis in Section \ref{sec::CLT}, we have
\begin{align*}
\E_{\pi^N}  \left[F_7(\mathfrak{L},\Psi)\right] &= \E_{\pi^N}  \left[F_5(\mathfrak{K},L(B_{x,x'}(c,w)))\right] \\
& \quad + \sum_{x\in\mathcal{X}} \E \left[ \left\vert \left(\Psi^N_t(x)-\Psi^N_0(x) -\int^t_0 \int_{\CX \times \CY} \alpha\paren{y-h^N_s(x')} L^N_s(B_{x,x'}(c,w)) \pi(dx',dy) ds\right.\right.\right.\\
&\qquad \qquad + \int^t_0 \int_{\CX \times \CY} \alpha \Psi^N_s(x') \ip{B_{x,x'},\mu^N_0} \pi(dx',dy) ds\\
&\qquad  \qquad \left.\left.\left. +\int^t_0 \int_{\CX \times \CY} \alpha K^N_s(x') l^N_s(B_{x,x'}(c,w)) \pi(dx',dy) ds\right)
\times \prod_{i=1}^p d_i(\Psi^N_{s_i})  \right \vert \right]\\
&\le C \paren{\frac{1}{N^{1-\gamma}}+ \frac{1}{N^{\frac{1}{2}-\zeta}}} + \frac{C}{N^{\varphi}} + C \E \left[\abs{N^{\zeta}V_t^N(x) + N^{\zeta}M_t^{N}(x) + \frac{1}{N^{\gamma+1-\zeta}} \sum_{k=0}^{\floor{Nt}-1} G^N_k(x)}\right]\\
&\le C \paren{\frac{1}{N^{1-\gamma}} + \frac{1}{N^{\gamma-\zeta}} }.
\end{align*}
Therefore, $\lim_{N\to \infty} \E_{\pi^N}[F_7(\mathfrak{L},\Psi)] = 0$. Since $F_7(\cdot)$ is continuous and uniformly bounded,
\[\lim_{N\to \infty} \E_{\pi^N}\left[F_7(\mathfrak{L},\Psi)\right] = \E_{\pi}\left[F_7(\mathfrak{L},\Psi)\right] = 0.\]

We have shown that any limit point $\pi$ of a convergence sequence must be a Dirac measure concentrated $(\mathfrak{L},\Psi)\in D_{E_3}([0,T])$, where $\mathfrak{L} = (\mathfrak{K},L(B_{x,x'}(c,w)))=(\mu,h,l(B_{x,x'}(c,w)),K, L(B_{x,x'}(c,w)))$ satisfies equation \eqref{LLN:limit_evolution}, \eqref{l_t limit}, \eqref{K_t for 1-gamma}, \eqref{limit_Lt}, and $\Psi_t$ satisfies equation \eqref{limit_Psi_2}. Since the solutions to equations \eqref{LLN:limit_evolution}, \eqref{K_t for 1-gamma} and  \eqref{limit_Psi_2} are unique, by Prokhorov's theorem, the processes $(\mathfrak{L}^N_t,\Psi^N_t)$ converges in distribution to $(\mathfrak{L}_t,\Psi_t)$.
\end{custlist}
\end{proof}

\subsection{Proof of Theorem \ref{R:BiasToZero2}}
Recall that $\tilde{h}_t = h_t - \hat{Y}$ and that $A, B_t \in \R^{M \times M}$ are matrices whose elements are $\alpha A_{x,x'}$ and $\alpha l_t(B_{x,x'}(c,w))$, respectively, for any $x,x' \in \mathcal{X}$. To show that $\Psi_t = \paren{\Psi(x^{(1)}), \ldots, \Psi(x^{(M)})} \to 0$ as $t \to \infty$, we consider the two cases of Theorem \ref{thm::Psi} separately.

\begin{custlist}[Case]
\item When $\gamma \in \paren{ \frac{3}{4}, \frac{5}{6}}$ and $\zeta \le \gamma - \frac{1}{2}$, or when $\gamma \in \left[\frac{5}{6}, 1\right)$ and $\zeta < 2-2\gamma \le \gamma - \frac{1}{2}$, by equation \eqref{limit_Psi}, we have
\begin{equation*}
\begin{aligned}
d \Psi_t &= -A \Psi_t dt, \text{ with } \Psi_0 = \begin{cases}
\mathcal{G}, &\text{if } \gamma \in\paren{ \frac{3}{4},\frac{5}{6}}, \zeta = \gamma - \frac{1}{2},\\
0, &\text{if } \gamma \in \paren{\frac{3}{4}, 1}, \zeta < \gamma - \frac{1}{2}, \end{cases}
\end{aligned}
\end{equation*}
where $\mathcal{G} \in \R^M$ is a Gaussian random variable with elements $\mathcal{G}(x)$ as given in equation \eqref{limit_gaussian} for $x \in \mathcal{X}$.
Since $A \in \R^{M \times M }$ is positive definite, we have $\Psi_t = e^{-tA} \Psi_0 \to 0$ exponentially fast as $t \to \infty$ for $\gamma\in(3/4,5/6)$.

\item  When $\gamma \in \left[\frac{5}{6}, 1\right)$ and $\zeta = 2-2\gamma$, by equation \eqref{limit_Psi_2}, we have
\begin{equation*}
\begin{aligned}
d \Psi_t &= -A \Psi_t - B_t K_t - C_t \tilde{h}_t dt, \text{ with }
\Psi_0 = \begin{cases}
\mathcal{G}, &\text{if } \gamma = \frac{5}{6},\\
0, &\text{if } \gamma \in \paren{\frac{5}{6}, 1}. \end{cases}
\end{aligned}
\end{equation*}
where $C_t \in \R^{M \times M}$ has elements $\alpha L_t(B_{x,x'}(c,w))$ for $x,x' \in \mathcal{X}$.  By similar calculation as for  Case 2 in Section \ref{sec::Kt_to_0}, we have the following integral equation for the solution:
\[\Psi_t = e^{-tA} \Psi_0 - \int_0^t e^{-(t-s)A} \paren{B_sK_s + C_s \tilde{h}_s} ds.\]

To find the bound for $\Vert \Psi \Vert$, we first show that $L_t(B_{x,x'}(c,w))$ is uniformly bounded. Since $B_{x,x'}(c,w) \in C^3_b(\R^{1+d})$ and for any $t>0$, $f \in C^3_b(\R^{1+d})$, by equation \eqref{limit_Lt}, Assumption  \ref{assumption}, and the proof for Theorem \ref{R:BiasToZero1}, we have
\begin{equation}\label{Lt_uni_bounded}
\begin{aligned}
\abs{L_t(f)} &\le \frac{C}{\lambda_0} \int_0^t \int_{\CX \times \CY} \abs{y-h_s(x')} \pi(dx',dy) ds + C \int_0^t \int_{\CX\times \CY} \abs{K_s(x')} \pi(dx',dy)ds \\
&\le \frac{C}{\lambda_0}\int_0^t \frac{1}{M} \sum_{x'\in \mathcal{X}} \abs{\tilde{h}_s(x')} ds + C \int_0^t  \frac{1}{M} \sum_{x'\in \mathcal{X}}\abs{K_s(x')} ds\\
&\le \frac{C}{\lambda_0} \int_0^t e^{-\lambda_0s} \norm{\tilde{h}_0(x')} ds + C \int_0^t  s e^{-\lambda_0s}\norm{\tilde{h}_0(x')} ds\\
&\le \frac{C}{\lambda_0} + C \left[ \frac{1}{\lambda_0^2} \paren{1-e^{-\lambda_0t}} - \frac{t}{\lambda_0} e^{-\lambda_0t} \right]\le C \paren{\frac{1}{\lambda_0} + \frac{1}{\lambda_0^2}},
\end{aligned}
\end{equation}
where the unimportant finite constant $C$ may change from term to term. Note that the last inequality holds because $te^{-\lambda_0t} \to 0$ as $t\to \infty$, which implies that the $te^{-\lambda_0t}$ is bounded in $t$.

By the proof of Theorem \ref{R:BiasToZero1} in Section \ref{sec::Kt_to_0} and equation \eqref{Lt_uni_bounded}, we have
\begin{equation*}
\begin{aligned}
\norm{\Psi_t} &\le C e^{-t\lambda_0} \norm{\Psi_0} + C \int_0^t e^{-(t-s)\lambda_0} \paren{\norm{B_sK_s} + \norm{C_s \tilde{h}_s}} ds\\
&\le C e^{-t\lambda_0} \norm{\Psi_0} + C \int_0^t e^{-(t-s)\lambda_0} \norm{K_s} ds + C \int_0^t e^{-(t-s)\lambda_0}\norm{\tilde{h}_s} ds\\
 &\le C e^{-t\lambda_0} \norm{\Psi_0} + Cte^{-\lambda_0t} \norm{\tilde{h}_0} + C \int_0^t e^{-(t-s)\lambda_0} e^{-\lambda_0s} s \norm{\tilde{h}_0} ds \\
&\le C e^{-t\lambda_0} \norm{\Psi_0} + C \paren{t+t^2}e^{-\lambda_0t}\norm{\tilde{h}_0}.
\end{aligned}
\end{equation*}

Since $\lim_{t \to \infty} (t+t^2) e^{-\lambda_0t} =0$, we have $\Vert \Psi_t \Vert \to 0$ as $t\to \infty$ exponentially fast. Hence the result of Theorem \ref{R:BiasToZero2} follows.
\end{custlist}

\section{Derivation of the asymptotic expansion of $h^{N,\gamma}_t$ for  $\gamma\in(1/2,1)$}\label{sec::higher order}

The goal of this section is to provide an inductive argument to derive the asymptotic expansion for $\ip{f,\mu^{N}_{t}}$ and $h^{N,\gamma}_{t}$ claimed in (\ref{measure_expansion}) and (\ref{network_expansion}) respectively.

Let $\nu\in\mathbb{N}$ and $\gamma\in\left(\frac{2\nu-1}{2\nu}, \frac{2\nu+1}{2\nu+2}\right]$. In regards to $h^{N,\gamma}_{t}$ for example, starting with the assumption that $Q^0_t = h_t$, the strategy is to use expression (\ref{LLN:evolution_h}), plug in the ansatz (\ref{measure_expansion}) and (\ref{network_expansion}) and then match coefficients in terms of powers of $N$ to derive an expression for $Q^{\nu}_t$. The outcome  of this process agrees with the results of Theorems \ref{CLT:theorem} and \ref{thm::Psi} for $\nu=1$ and $\nu=2$ respectively and shows what one expects for $\nu=\{3,4,\cdots\}$. For the asymptotic expansion for $\ip{f,\mu^N_t}$, we start with assuming $l^0_t(f) = \ip{f,\mu_0}$, and derive an expression for $l^{\nu}_t(f)$ by using equation \eqref{LLN:evolution_mu_integral} for different $\nu \in \mathbb{N}$ and the corresponding intervals for $\gamma$.

We note that the process presented in this section, does recover the rigorous results for $\nu=1$ and $\nu=2$ as obtained in Theorems \ref{CLT:theorem} and \ref{thm::Psi} and it consistently gives the expansion for $k\in\{3,4,5,\cdots\}$. The results of this section can be rigorously derived inductively as Theorem \ref{thm::Psi} was derived using Theorem \ref{CLT:theorem}. However, the formal approach used here is sufficient for our purposes and consequently, without loss of generality, the presentation will be less rigorous than before.

\subsection{General $\nu>2$ case}
To find an expression for $Q^{\nu}_t$ for any $\nu >2$, we assume that $Q^0_t(x) = h_t(x)$ and $l^0_t(f) = \ip{f,\mu_0}$, that the result holds for $\nu=1$ and for $\nu=2$ (proven in Sections \ref{sec::CLT} and \ref{sec::Psi} respectively) and that for $j=3,\ldots,\nu-1$, $Q^j_t$ and $l^j_t(f)$ satisfy the deterministic evolution equations (\ref{Eq:l_equation}) and (\ref{Eq:Qj_formula1}) respectively
\begin{equation*}
\begin{aligned}
Q^j_t(x) &= \alpha \int^t_0 \int_{\CX \times \CY} \paren{y-Q^0_s(x')} l^j_s(B_{x,x'}(c,w))\pi(dx',dy) ds\nonumber\\
&\qquad-\alpha\sum_{m=1}^j \int^t_0 \int_{\CX \times \CY} Q^m_s(x')l^{j-m}_s (B_{x,x'}(c,w)) \pi(dx',dy) ds,
\end{aligned}
\end{equation*}
and
\begin{equation*}
\begin{aligned}
l^j_t(f) &= \alpha \int_0^t \int_{\CX\times \CY}  \paren{y-Q^0_s(x')}l_s^{j-1}(C_{x'}^{f}(c,w)) \pi(dx',dy)ds\\
&\quad -\alpha \sum_{m=1}^{j-1}\int_0^t \int_{\CX\times \CY}  Q^{j-m}_s(x')l^{m-1}_s(C_{x'}^{f}(c,w)) \pi(dx',dy)ds.
\end{aligned}
\end{equation*}

The presentation below uses an inductive argument. We have already rigorously shown that the statement holds for $\nu=1$ and $\nu=2$. We will assume that it holds for $j=1,\ldots,\nu-1$ and demonstrate that it should also hold for $j=\nu$ for any $\nu\in\mathbb{N}$. Furthermore, we will also show that $Q^{\nu}_t \to 0$ as $t \to \infty$ for any fixed $\nu>2$ assuming that for $j=1,\ldots,\nu-1$, $Q^j_t \to 0$ as $t \to \infty$ and $l^j(f)$ is uniformly bounded for any $t \ge 0$ and $f\in C^{\infty}_b(\R^{1+d})$. In particular, we have
\begin{equation}\label{Qj_bound}
\begin{aligned}
\norm{Q^j_t} &\le C\paren{\sum_{m=1}^j t^m} e^{-\lambda_0t} \norm{\tilde{h}_0}.
\end{aligned}
\end{equation}
for an unimportant finite constant $C<\infty$. Note that Theorems \ref{R:BiasToZero1} and \ref{R:BiasToZero2} proved the convergence of $Q^{\nu}_t$ to zero for $\nu=1$ and $\nu=2$.
\begin{itemize}
\item When $\gamma \in \left(\frac{2\nu-1}{2\nu}, \frac{2\nu+1}{2\nu+2}\right)$,
plugging equations \eqref{network_expansion} and \eqref{measure_expansion} into the left hand side of equation \eqref{LLN:evolution_h} gives (the symbol $\approx$ is used below in place of the remainder term $V_t^{N}(x) + M_t^{N}(x) + \frac{1}{N^{\gamma+1}} \sum_{k=0}^{\floor{Nt}-1} G^N_{k} (x)$ in \eqref{LLN:evolution_h})
\begin{align*}
h_t^{N}(x) - h^N_{0}(x) &\approx \alpha \int_0^t  \int_{\CX \times \CY} \paren{y-\sum_{j=0}^{\nu-1} \frac{1}{N^{j(1-\gamma)}}Q^j_s(x') - \frac{1}{N^{\gamma-\frac{1}{2}}}Q^{\nu}_s(x')}  \\
&\qquad  \qquad \times \sum_{j=0}^{\nu-1} \frac{1}{N^{j(1-\gamma)}} l^j_s(B_{x,x'}(c,w)) \pi(dx',dy) ds\\
&=\alpha \int_0^t \int_{\CX \times \CY}  \left[\paren{y-Q^0_s(x')} -\sum_{j=1}^{\nu-1} \frac{1}{N^{j(1-\gamma)}}Q^j_s(x') - \frac{1}{N^{\gamma-\frac{1}{2}}}Q^{\nu}_s(x')\right] \\
&\qquad \qquad  \times \left[A_{x,x'} + \sum_{j=1}^{\nu-1} \frac{1}{N^{j(1-\gamma)}} l^j_s(B_{x,x'}(c,w)) \right] \pi(dx,dy)ds
\end{align*}
\begin{align*}
&=\alpha\int^t_0 \int_{\CX \times \CY} \paren{y-Q^0_s(x')} A_{x,x'} \pi(dx',dy) ds \\
&\quad+ \left\lbrace\alpha \sum_{j=1}^{\nu-1} \frac{1}{N^{j(1-\gamma)}} \int_0^t \int_{\CX \times \CY} \paren{y-Q^0_s(x')} l^j_s(B_{x,x'}(c,w)) \pi(dx',dy)ds\right.\\
&\quad \left.  -  \alpha \int_0^t \int_{\CX \times \CY} \paren{\sum_{j=1}^{\nu-1} \frac{1}{N^{j(1-\gamma)}} Q^j_s(x')} \paren{\sum_{i=0}^{\nu-1} \frac{1}{N^{i(1-\gamma)}} l^i_s(B_{x,x'}(c,w))}\pi(dx',dy)ds\right\rbrace\\
&\quad -\frac{\alpha}{N^{\gamma-\frac{1}{2}}}\int^t_0 \int_{\CX \times \CY} Q^{\nu}_s(x')A_{x,x'}\pi(dx',dy) ds \\
&\quad - \frac{\alpha}{N^{\gamma-\frac{1}{2}}} \sum_{j=1}^{\nu-1} \frac{1}{N^{j(1-\gamma)}} \int_0^t \int_{\CX \times \CY} Q^{\nu}_x(x') l^j_s(B_{x,x'}(c,w)) \pi(dx',dy)ds.\\
\end{align*}

The second term in the last equality can be re-arranged to
\begin{align}
&\alpha \sum_{j=1}^{\nu-1} \frac{1}{N^{j(1-\gamma)}} \int_0^t \int_{\CX \times \CY} \paren{y-Q^0_s(x')} l^j_s(B_{x,x'}(c,w)) \pi(dx',dy)ds\nonumber\\
&\quad  - \alpha \int_0^t \int_{\CX \times \CY} \paren{\sum_{j=1}^{\nu-1} \frac{1}{N^{j(1-\gamma)}} Q^j_s(x')} \paren{\sum_{i=0}^{\nu-1} \frac{1}{N^{i(1-\gamma)}} l^i_s(B_{x,x'}(c,w))}\pi(dx',dy)ds\nonumber\\
&= \alpha \sum_{j=1}^{\nu-1} \frac{1}{N^{j(1-\gamma)}} \int_0^t \int_{\CX \times \CY} \paren{y-Q^0_s(x')} l^j_s(B_{x,x'}(c,w)) \pi(dx',dy)ds\nonumber\\
&\quad  - \alpha \sum_{j=1}^{\nu-1}\sum_{i=0}^{\nu-1} \frac{1}{N^{(j+i)(1-\gamma)}} \int_0^t \int_{\CX \times \CY}  Q^j_s(x')l^i_s(B_{x,x'}(c,w))\pi(dx',dy)ds\nonumber\\
&= \alpha \sum_{j=1}^{\nu-1} \frac{1}{N^{j(1-\gamma)}} \int_0^t \int_{\CX \times \CY} \paren{y-Q^0_s(x')} l^j_s(B_{x,x'}(c,w)) \pi(dx',dy)ds\nonumber\\
&\quad  - \alpha \sum_{n=1}^{\nu-1}\sum_{m=1}^{n} \frac{1}{N^{n(1-\gamma)}} \int_0^t \int_{\CX \times \CY}  Q^m_s(x')l^{n-m}_s(B_{x,x'}(c,w))\pi(dx',dy)ds\nonumber\\
&\quad  - \alpha \sum_{n=\nu}^{2\nu-2} \frac{1}{N^{n(1-\gamma)}} \sum_{m=n-(\nu-1)}^{\nu-1} \int_0^t \int_{\CX \times \CY}  Q^m_s(x')l^{n-m}_s(B_{x,x'}(c,w))\pi(dx',dy)ds\nonumber\\
&=  \sum_{j=1}^{\nu-1} \frac{\alpha}{N^{j(1-\gamma)}} \int_0^t \int_{\CX \times \CY} \paren{y-Q^0_s(x')} l^j_s(B_{x,x'}(c,w)) \pi(dx',dy)ds\nonumber\\
&\quad  - \sum_{j=1}^{\nu-1}\frac{\alpha}{N^{j(1-\gamma)}}  \sum_{m=1}^{j}  \int_0^t \int_{\CX \times \CY}  Q^m_s(x')l^{j-m}_s(B_{x,x'}(c,w))\pi(dx',dy)ds + O\paren{N^{-\nu(1-\gamma)}}\nonumber\\
&= \sum_{j=1}^{\nu-1} \frac{1}{N^{j(1-\gamma)}} Q^j_t(x)+ O\paren{N^{-\nu(1-\gamma)}}.\label{rearrangment}
\end{align}

Therefore,
\begin{equation*}
h_t^{N}(x) - h^N_{0}(x) = Q^0_t(x) + \sum_{j=1}^{\nu-1} \frac{1}{N^{j(1-\gamma)}} Q^j_t(x) -\frac{\alpha}{N^{\gamma-\frac{1}{2}}}\int^t_0 \int_{\CX \times \CY} Q^{\nu}_s(x')A_{x,x'} \pi(dx',dy) ds + O(N^{-\Omega_{\nu}})
\end{equation*}
for some $\Omega_{\nu}  > \gamma - \frac{1}{2}$.
Adding $h^N_0(x) = \frac{1}{N^{\gamma-\frac{1}{2}}}(N^{\gamma-\frac{1}{2}}h^N_0(x))$ and subtracting $Q^0_t(x) + \sum_{j=1}^{\nu-1} \frac{1}{N^{j(1-\gamma)}} Q^j_t(x)$ to both sides, we have
\begin{equation*}
\begin{aligned}
\frac{1}{N^{\gamma-\frac{1}{2}}}Q^{\nu}_t(x) = \frac{1}{N^{\gamma-\frac{1}{2}}}(N^{\gamma-\frac{1}{2}}h^N_0(x))-\frac{\alpha}{N^{\gamma-\frac{1}{2}}}\int^t_0 \int_{\CX \times \CY} Q^{\nu}_s(x')A_{x,x'} \pi(dx',dy) ds.
\end{aligned}
\end{equation*}
Since $N^{\gamma-\frac{1}{2}}h^N_0(x)$ converges in distribution to the Gaussian random variable $\mathcal{G}(x)$ as defined in \eqref{limit_gaussian}, we have an expression for $Q^{\nu}_t$
\[
Q^{\nu}_t(x) = \mathcal{G}(x)-\alpha\int^t_0 \int_{\CX \times \CY} Q^{\nu}_s(x')A_{x,x'} \pi(dx',dy) ds,
\]
which coincides with (\ref{Eq:Qk_formula1}). This also implies that $Q^{\nu}_t$ (the vector whose elements are $Q^{\nu}_t(x)$ for $x\in\mathcal{X}$) satisfies the equation
\begin{equation*}
\begin{aligned}
d Q^{\nu}_t &= -A Q^{\nu}_t dt, \text{ and } Q^{\nu}_0 =\mathcal{G},
\end{aligned}
\end{equation*}
where $\mathcal{G} \in \R^M$ is a Gaussian random variable with elements $\mathcal{G}(x)$ as given in equation \eqref{limit_gaussian} for $x \in \mathcal{X}$.
Since $A \in \R^{M \times M }$ is positive definite, we have $Q^{\nu}_t = e^{-tA} Q^{\nu}_0$ and $\norm{Q^{\nu}_t} \to 0$ as $t \to \infty$ exponentially fast.
\item When $\gamma \ge \frac{2\nu+1}{2\nu+2}$, we first derive an expression for $l^{\nu}_t(f)$ by plugging \eqref{network_expansion} and \eqref{measure_expansion} into equation \eqref{LLN:evolution_mu_integral} (as before, for notational convenience, we use the symbol $\approx$ to account for the remainder terms in \eqref{LLN:evolution_mu_integral} that go to zero),
\begin{equation*}
\begin{aligned}
\ip{f,\mu^N_t} &\approx\ip{f,\mu^N_0} +\frac{\alpha}{N^{1-\gamma}} \int_0^t \int_{\CX \times \CY} \paren{y-Q^0_s(x') - \sum_{j=1}^{\nu} \frac{1}{N^{j(1-\gamma)}} Q^j_s(x') - O(N^{-(k+1)(1-\gamma)})}\\
& \times \left[\sum_{j=0}^{\nu} \frac{1}{N^{j(1-\gamma)}} l^j_s(C_{x'}^{f}(c,w)) + O(N^{-(k+1)(1-\gamma)}) \right] \pi(dx',dy)ds
\\
&= \ip{f,\mu^N_0} + \sum_{j=1}^{\nu-1} \frac{1}{N^{j(1-\gamma)}} l^j_t(f)  +\frac{\alpha}{N^{\nu(1-\gamma)}} \int_0^t \int_{\CX \times \CY} \paren{y-Q^0_s(x')}l^{\nu-1}_s(C_{x'}^{f}(c,w)) \pi(dx',dy)ds \\
&\quad -\frac{\alpha}{N^{\nu(1-\gamma)}}\sum_{j=1}^{\nu-1} \int_0^t \int_{\CX \times \CY} Q^{\nu-j}_s(x')l^{j-1}_s(C_{x'}^{f}(c,w)) \pi(dx',dy)ds +O(N^{-(\nu+1)(1-\gamma)}).
\end{aligned}
\end{equation*}
Note that the second term in the last equation above can be obtained by manipulating the middle terms following a similar idea as in equation \eqref{rearrangment}. Rearranging terms in the equation above yields
\begin{equation*}
\begin{aligned}
l^{\nu}_t(f)
&= N^{\nu(1-\gamma)}\paren{\ip{f,\mu^N_0}-\ip{f,\mu_0}} \\
&\quad + \alpha \int_0^t \int_{\CX \times \CY} \paren{y-Q^0_s(x')}l^{\nu-1}_s(C_{x'}^{f}(c,w)) \pi(dx',dy)ds \\
&\quad -\alpha\sum_{j=1}^{\nu-1} \int_0^t \int_{\CX \times \CY} Q^{\nu-j}(x')l^{j-1}_s(C_{x'}^{f}(c,w)) \pi(dx',dy)ds +O(N^{-(1-\gamma)}).
\end{aligned}
\end{equation*}
Since $N^{\nu(1-\gamma)}\paren{\ip{f,\mu^N_0}-\ip{f,\mu_0}}$ converges to zero in distribution when $\gamma \ge \frac{2\nu+1}{2\nu+2}$, we can get the following evolution equation for $l^{\nu}_t(f)$,
\begin{align*}
l^{\nu}_t(f)
&= \alpha \int_0^t \int_{\CX \times \CY} \paren{y-Q^0_s(x')}l^{\nu-1}_s(C_{x'}^{f}(c,w)) \pi(dx',dy)ds \\
&\quad -\alpha\sum_{j=1}^{\nu-1} \int_0^t \int_{\CX \times \CY} Q^{\nu-j}_s(x')l^{j-1}_s(C_{x'}^{f}(c,w)) \pi(dx',dy)ds,
\end{align*}
which concludes the inductive step for $l^{\nu}_t(f)$.

Next, we derive $Q^{\nu}_t$ by plugging equations \eqref{network_expansion} and \eqref{measure_expansion} into the left hand side of equation \eqref{LLN:evolution_h} (using the symbol $\approx$ to ignore the same remainder terms as before):
\begin{equation*}
\begin{aligned}
h_t^{N}(x) - h^N_{0}(x) &\approx \alpha \int_0^t  \int_{\CX \times \CY} \paren{y-\sum_{j=0}^{\nu} \frac{1}{N^{j(1-\gamma)}}Q^j_s(x') - O(N^{-(k+1)(1-\gamma)})} \\
&\qquad \times \paren{ \sum_{j=0}^{\nu} \frac{1}{N^{j(1-\gamma)}} l^j_s(B_{x,x'}(c,w))+ O(N^{-(\nu+1)(1-\gamma)}) }\pi(dx',dy) ds\\
&=Q^0_t(x)+ \sum_{j=1}^{\nu-1} \frac{\alpha}{N^{j(1-\gamma)}}Q^j_t(x)\\
&\quad +\frac{\alpha}{N^{\nu(1-\gamma)}} \int^t_0 \int_{\CX \times \CY} \paren{y-Q^0_s(x')} l^{\nu}_s(B_{x,x'}(c,w))\pi(dx',dy) ds\\
&\quad -\frac{\alpha}{N^{\nu(1-\gamma)}}\sum_{j=1}^{\nu} \int^t_0 \int_{\CX \times \CY} Q^j_s(x')l^{\nu-j}_s (B_{x,x'}(c,w)) \pi(dx',dy) ds + O(N^{-\Omega_{\nu 1}}),\\
\end{aligned}
\end{equation*}
where $\Omega_{\nu 1}>(\nu+1)(1-\gamma )$.
Following the same idea as earlier, when $\gamma = \frac{2\nu+1}{2\nu+2}$, we note that $\nu(1-\gamma) = \gamma -\frac{1}{2}=\frac{\nu}{2\nu+2}$, we can obtain an expression for $Q^{\nu}_t$ (which coincides with (\ref{Eq:Qj_formula1})):
\begin{equation*}
\begin{aligned}
Q^{\nu}_t(x) &= \mathcal{G}(x) + \alpha \int^t_0 \int_{\CX \times \CY} \paren{y-Q^0_s(x')} l^{\nu}_s(B_{x,x'}(c,w))\pi(dx',dy) ds\\
&\quad -\alpha\sum_{j=1}^{\nu} \int^t_0 \int_{\CX \times \CY} Q^j_s(x')l^{\nu-j}_s (B_{x,x'}(c,w)) \pi(dx',dy) ds,
\end{aligned}
\end{equation*}
where $\mathcal{G}(x)$ is the Gaussian random variable as in \eqref{limit_gaussian}. When $\gamma >\frac{2\nu+1}{2\nu+2}$, $Q^{\nu}_t$ is driven by the deterministic equation
\begin{equation*}
\begin{aligned}
Q^{\nu}_t(x) &= \alpha \int^t_0 \int_{\CX \times \CY} \paren{y-Q^0_s(x')} l^{\nu}_s(B_{x,x'}(c,w))\pi(dx',dy) ds\\
&\quad -\alpha\sum_{j=1}^{\nu} \int^t_0 \int_{\CX \times \CY} Q^j_s(x')l^{\nu-j}_s (B_{x,x'}(c,w)) \pi(dx',dy) ds.
\end{aligned}
\end{equation*}

This concludes the inductive step for the derivation of $Q^{\nu}_t(x)$. Let us now show that, for fixed $\nu\in\mathbb{N}$, $Q^{\nu}_t(x)$ decay to zero exponentially fast as $t\rightarrow\infty$. For $\gamma \ge \frac{2\nu+1}{2\nu+2}$, we can write that $Q^{\nu}_t$  satisfies the system of equations
\begin{equation*}
\begin{aligned}
d Q^{\nu}_t &= -D^{\nu}_t\tilde{h}_t - \sum_{j=1}^{\nu} D^{\nu-j}_{t}Q^j_t, \text{ with }
Q^{\nu}_0 = \begin{cases} \mathcal{G}, &\text{if } \gamma = \frac{2\nu+1}{2\nu+2}, \\
0, &\text{if } \gamma > \frac{2\nu+1}{2\nu+2},
\end{cases}
\end{aligned}
\end{equation*}
where for $j = 0, \ldots ,\nu$, $D^j_t \in \R^{M\times M}$ whose elements are $\alpha l^j_t(B_{x,x'}(c,w))$ with $x,x' \in \mathcal{X}$. Note that $D^0_t = A$ is positive definite, and $D^1_t = B_t$, $D^2_t = C_t$ as defined in the proofs for Theorems \ref{R:BiasToZero1} and \ref{R:BiasToZero2}. By similar analysis as in Section \ref{sec::LLN}, one can show that solution for this system is
\begin{equation}\label{sol_Qk}
Q^{\nu}_t = e^{-tA}Q^{\nu}_0 - \int_0^t e^{-(t-s)A} \paren{D^{\nu}_s \tilde{h}_s + \sum_{j=1}^{\nu-1}D^{\nu-j}_s Q^j_s} ds.
\end{equation}

To show that $Q^{\nu}_t \to 0$ as $t \to \infty$, for $\gamma \ge \frac{2\nu+1}{2\nu+2}$, we first show that $l^{\nu}_t(f)$ is uniformly bounded. Since $l^j(f)$ is uniformly bounded for $j=1,\ldots, \nu-1$, by equation \eqref{Qj_bound}, we have
\begin{equation*}
\begin{aligned}
\abs{l^{\nu}_t(f)}
&\le C \int_0^t \int_{\CX \times \CY} \abs{y-Q^0_s(x')} \pi(dx',dy)ds  + C \sum_{j=1}^{\nu-1}\int_0^t \int_{\CX \times \CY} \abs{Q^{\nu-j}_s(x')} \pi(dx',dy)ds\\
& \le C \int_0^t \frac{1}{M} \sum_{x'\in \mathcal{X}} \abs{\tilde{h}_s(x')} ds + C \sum_{j=1}^{\nu-1} \int_0^t \frac{1}{M} \sum_{x'\in \mathcal{X}}\abs{Q^{\nu-j}_s(x')} ds\\
& \le C  \int_0^t e^{-\lambda_0 s}\norm{\tilde{h}_0} ds + C \sum_{j=1}^{\nu-1} \int_0^t \paren{\sum_{n=1}^{\nu-j}s^n} e^{-\lambda_0s}\norm{\tilde{h}_0} ds\\
&\le \frac{C}{\lambda_0} + C \sum_{j=1}^{\nu-1}\sum_{n=1}^{\nu-j} \int_0^t s^n e^{-\lambda_0s} ds.
\end{aligned}
\end{equation*}

Since for each $n = 1, \ldots, \nu-j$, as $t\to \infty$, we have
\[\lim_{t\to \infty} \int_0^t s^n e^{-\lambda_0s} ds = \int_0^{\infty} s^n  e^{-\lambda_0s} ds = \frac{\Gamma(n+1)}{\lambda_0^{n+1}} = \frac{n!}{\lambda_0^{n+1}} \le C_n, \]
where $\Gamma(\cdot)$ is the gamma function and $C_n$ is some finite number depending on $n$, this implies that $l^{\nu}_t(f)$ is uniformly bounded.

We shall proceed inductively. Since $Q^j_t \to 0$ as $t \to \infty$, $l^j_t(f)$ is uniformly bounded for $j=1,\ldots, \nu-1$, and $l^{\nu}_t(f)$ is uniformly bounded, by equations \eqref{Qj_bound} and \eqref{sol_Qk}, we have
\begin{equation*}
\begin{aligned}
\norm{Q^{\nu}_t} &\le  C e^{-\lambda_0 t}\norm{Q^{\nu}_0} + C \int_0^t e^{-\lambda_0(t-s)} \paren{\norm{D^{\nu}_s \tilde{h}_s} + \sum_{j=1}^{\nu-1}\norm{D^{\nu-j}_s Q^j_s}} ds\\
&\le C e^{-\lambda_0 t}\norm{Q^{\nu}_0} + C \int_0^t e^{-\lambda_0(t-s)} \norm{\tilde{h}_s} ds + C \sum_{j=1}^{\nu-1} \int_0^t e^{-\lambda_0(t-s)} \norm{ Q^j_s} ds\\
&\le  C e^{-\lambda_0 t}\norm{Q^{\nu}_0} +  C t e^{-\lambda_0 t}\norm{\tilde{h}_0} + C e^{-\lambda_0t} \norm{\tilde{h}_0} \sum_{j=1}^{\nu-1} \sum_{m=1}^j \int_0^t s^m ds\\
&\le C e^{-\lambda_0 t}\norm{Q^{\nu}_0} +  C t e^{-\lambda_0 t}\norm{\tilde{h}_0} + C  e^{-\lambda_0t} \norm{\tilde{h}_0} \frac{\nu-1}{2} \sum_{j=1}^{\nu-1} (\nu-j)  t^{j+1}\\
&\le  C e^{-\lambda_0 t}\norm{Q^{\nu}_0} +   C  \paren{\sum_{j=1}^{\nu}   t^{j}} e^{-\lambda_0 t} \norm{\tilde{h}_0}.
\end{aligned}
\end{equation*}

Therefore, $\norm{Q^{\nu}_t} \to 0$ as $t\to \infty$ exponentially fast, which concludes the inductive step and thus the derivation of the result. 
\end{itemize}

\bibliographystyle{abbrv}

\end{document}